%% file: main.tex
\documentclass[nohyperref]{article}

\usepackage{microtype}
\usepackage{graphicx}
\usepackage{booktabs} 

\usepackage{hyperref}

\usepackage[accepted]{icml2022}

\usepackage{amsmath}
\usepackage{amssymb}
\usepackage{mathtools}
\usepackage{amsthm}

\usepackage[capitalize,noabbrev]{cleveref}

\theoremstyle{plain}
\newtheorem{theorem}{Theorem}[section]
\newtheorem{proposition}[theorem]{Proposition}
\newtheorem{lemma}[theorem]{Lemma}
\newtheorem{corollary}[theorem]{Corollary}
\theoremstyle{definition}
\newtheorem{definition}[theorem]{Definition}
\newtheorem{assumption}[theorem]{Assumption}
\theoremstyle{remark}

\usepackage{times}
\usepackage{tikz}
\usepackage{amsmath}

\usepackage{filecontents}
\usepackage{hyperref}
\usepackage{url}
\usepackage{graphicx,xcolor, xspace, paralist}
\usepackage{amsmath,amssymb}
\usepackage[style=base,textfont={small, bf},belowskip=-8pt, aboveskip=0pt]{caption}
\usepackage{multicol}
\usepackage{multirow}
\usepackage{enumitem}
\usepackage[compact]{titlesec}
\usepackage{titling}
\usepackage{subcaption}
\usepackage{longtable}
\usepackage{amsthm}
\usepackage[normalem]{ulem}
\usepackage[algo2e,linesnumbered,algoruled,noend,noline]{algorithm2e}

\usepackage{xcolor}
\usepackage{soul}

\SetKwProg{Fn}{Function}{:}{}

\setlist{leftmargin=*,itemsep=0pt,parsep=0pt,topsep=1pt,partopsep=1pt}
\titlespacing{\section}{0pt}{1ex}{1ex}
\titlespacing{\subsection}{0pt}{0.5ex}{0ex}
\titlespacing{\subsubsection}{0pt}{0.5ex}{0ex}
\newcommand{\mypara}[1]{\vspace{2pt}\noindent {\bf {#1}.\xspace}}

\newcommand{\allnotes}[1]{}
\newcommand{\ignore}[1]{}
\renewcommand{\allnotes}[1]{\textit{#1}}

\newcommand{\notepanda}[1]{\allnotes{\textcolor{cyan}{[Panda: #1]}}}

\newcommand{\mathias}[1]{\allnotes{\textcolor{green}{[Mathias: #1]}}}

\newcommand{\notejinkun}[1]{\allnotes{\textcolor{orange}{[Jinkun: #1]}}}
\newcommand{\noteanqi}[1]{\allnotes{\textcolor{purple}{[Anqi: #1]}}}

\DeclareMathOperator*{\argmin}{arg\,min}
\newcommand{\quantity}{\text{AME}\xspace}
\newcommand{\AME}{\text{AME}\xspace}
\newcommand{\SV}{\text{SV}\xspace}

\newcommand{\somename}{Enola\xspace}
\def\ie{{i.e.},\xspace}
\def\eg{{e.g.},\xspace}

\date{}
\def\cP{\mathcal{P}}

\usepackage{wrapfig}

\icmltitlerunning{Measuring the Effect of Training Data on Deep Learning Predictions via Randomized Experiments}

\begin{document}

\twocolumn[
\icmltitle{Measuring the Effect of Training Data on Deep Learning Predictions via Randomized Experiments}



\icmlsetsymbol{equal}{*}

\begin{icmlauthorlist}
\icmlauthor{Jinkun Lin}{equal,nyu}
\icmlauthor{Anqi Zhang}{equal,nyu}
\icmlauthor{Mathias L\'ecuyer}{ubc}
\icmlauthor{Jinyang Li}{nyu}
\icmlauthor{Aurojit Panda}{nyu}
\icmlauthor{Siddhartha Sen}{msr}
\end{icmlauthorlist}

\icmlaffiliation{nyu}{Department of Computer Science, New York University, New York, NY}
\icmlaffiliation{msr}{Microsoft Research, New York, NY}
\icmlaffiliation{ubc}{University of British Columbia, Vancouver, Canada}

\icmlcorrespondingauthor{Jinkun Lin}{jinkun.lin@nyu.edu}
\icmlcorrespondingauthor{Mathias L\'ecuyer}{mathias.lecuyer@ubc.ca}


\vskip 0.3in
]



\printAffiliationsAndNotice{\icmlEqualContribution} 

\begin{abstract}

We develop a new, principled algorithm for estimating the contribution of training data points to the behavior of a deep learning model, such as a specific prediction it makes. 
Our algorithm estimates the AME, a quantity that measures the expected (average) marginal effect of adding a data point to a subset of the training data, sampled from a given distribution. When subsets are sampled from the uniform distribution, the AME reduces to the well-known Shapley value.
Our approach is inspired by causal inference and randomized experiments: we sample different subsets of the training data to train multiple submodels, and evaluate each submodel's behavior. We then use a LASSO regression to jointly estimate the AME of each data point, based on the subset compositions. Under sparsity assumptions ($k \ll N$ datapoints have large AME), our estimator requires only $O(k\log N)$ randomized submodel trainings, 
improving upon the best prior Shapley value estimators. 
\ignore{
We extend our estimator to support control over its false positive rate using the Knockoffs method; and also to support hierarchical data.
We demonstrate the practicality of our approach by applying it to several data poisoning and model explanation tasks, across a variety of datasets.
}


\end{abstract}

\input{intro-icml}
\input{ame-icml}

\input{estimator-icml}

\input{ame-sv-connection}

\input{extensions-icml}

\input{evaluation}

\input{related-work-v2}

\ignore{
\section{Conclusion}
\somename addresses the data attribution problem for complex ML models, under a practical setting where data are organized into a (possibly hierarchical) sources. We propose a metric, \quantity{}, for measuring individual and group contributions to predictions made by a trained model.
We describe an efficient methodology to estimate the \quantity{} under a sparsity assumption, based on a reformulation of the estimation as a regression problem, which allows an efficient application of LASSO.
We further leverage Knockoffs to control the false detection rate.
Finally, we extend \somename{}'s methodology to support hierarchies of data sources, enabling proponent detection at each level of the hierarchy, which leads to cost savings and graceful degradation of the detection.
Our evaluation shows that \somename is effective across a number of datasets and models.
}

\section*{Acknowledgments}
Jinkun Lin is partially supported by NSF-1514422. Anqi Zhang is partially supported by a gift from Microsoft. We thank Daniel Hsu for insightful discussions in the early stages of the project, as well as the reviewers for their constructive comments.

\bibliography{bibs}
\bibliographystyle{icml2022}

\include{appendix}

\end{document}

%% file: intro-icml.tex
\section{Introduction}
\label{sec:intro}

Machine Learning (ML) is now ubiquitous, with black-box models such as deep neural networks (DNNs) powering an ever increasing number of applications, yielding social and economic benefits.
However, these complex models are the result of long, iterative training procedures over large amounts of data, which make them hard to understand, debug, and protect.
As an important first step towards addressing these challenges,  we must be able to solve the problem of {\em data attribution}, which aims to pinpoint training data points with significant contributions to specific model behavior.
There are many use cases for data attribution: it can be used to assign value to different training data based on the accuracy improvements they bring~\cite{koh_accuracy_2019, jia_towards_2020}, explain the source of (mis)predictions \cite{koh_understanding_2017,basu_influence_2020}, or find faulty data points resulting from data bugs \cite{chakarov2016debugging} or malicious poisoning \cite{Shafahi2018PoisonFT}. 

Existing principled approaches to explain how training data points influence DNN behavior either measure Influence functions~\cite{koh_understanding_2017} or Shapley values~\cite{ghorbani_data_nodate,jia_towards_2020}.
Influence provides a local explanation that misses complex dependencies between data points as well as contributions that build up over time during training \cite{basu_influence_2020}. 
While Shapley values account for complex dependencies, they are prohibitively expensive to calculate: exact computation requires O($2^N)$ model evaluations, and the best known approximation requires $O(N\log \log N)$ model evaluations, where $N$ is the number of data points in the training set~\cite{jia_towards_2020}.

In this paper, we propose a new, principled metric for data attribution.  Our metric, called \quantity{}, measures the contribution of each training data point to a given behavior of the trained model (e.g., a specific prediction, or test set accuracy).  \quantity is defined as the expected marginal effect contributed by a data point to the model behavior over randomly selected subsets of the data. Intuitively, a data point has a large \quantity{} when adding it to the training data affects the behavior under study, regardless of which other data points are present. 
We show that the \quantity{} can be efficiently estimated using a carefully designed LASSO regression under the sparsity assumption (\ie there are $k\ll N$ data points with large \quantity{} values). In particular, our estimator requires only $O(k\log N)$ evaluations, which makes it practical to use with large training sets.
When using \quantity{} to detect data poisoning/corruption, we also extend our estimator to provide control over the false positive rate using the Knockoffs method~\cite{candes_panning_2017}.

When the size of subsets used by our algorithm is drawn uniformly, %
the \quantity{} reduces to the Shapley value (SV). As a result, our \quantity{} estimator provides a new method for  estimating the SVs of all training data points; under the same sparsity and monotonicity assumptions, we obtain a better rate 
than the previous state-of-the-art~\cite{jia_towards_2020}. 
Interestingly, our causal framing also supports working with groups of data points, which we call data sources. Many datasets are naturally grouped into sources, such as by time window, contributing user, or website.
In this setting, we extend the \quantity{} and our estimator to support hierarchical estimation for nested data sources.
For instance, this enables joint measurement of both users with large contributions, and the specific data points that drive their contribution.

We empirically evaluate the \quantity{} quantity and our estimator’s performance on three important applications: detecting data poisoning, explaining predictions, and estimating the Shapley value of training data. For each application, we compare our approach to existing methods.

In summary, we make the following contributions:
\begin{itemize}
\item We propose a new quantity for the data attribution problem, \quantity{}, with roots in randomized experiments from causal inference (\S\ref{sec:motivation}). We also show that SV is a special case of \quantity{}.
\item We present an efficient estimator for \quantity{} with an $O(k\log N)$ rate under sparsity assumptions (\S\ref{sec:methodology}). This also yields an $O(k\log N)$ estimator for sparse and monotonic SV, a significant improvement over the previous $O(N \log \log N)$ state-of-the-art \cite{jia_towards_2020}.
\item We extend the \quantity{} and our estimator to control false discoveries (\S\ref{subsec:knockoffs}) and support hierarchical settings of nested data sources (\S\ref{sec:hier}).  
\end{itemize}

%% file: ame-icml.tex
\section{Average Marginal Effect (AME)}
\label{sec:motivation}

At a high level, our goal is to understand the impact of training data on the behavior of a trained ML classification model, which we call a query. Queries of interest include a specific prediction, or the test set accuracy of a model.
Below, we formalize our setting (\S\ref{subsec:problem}) and present our metric for quantifying data contributions  (\S\ref{subsec:ame}).

\subsection{Notations}
\label{subsec:problem}

Let $\mathcal{M}_S$ denote the classification model trained on dataset $S$ that we wish to analyze. In the rest of the paper, we refer to this as the \emph{main model}. We note that $\mathcal{M}_S$ belongs to a class of models $\mathcal{M}$ (\ie $\mathcal{M}_S\in \mathcal{M}$) with the same architecture, but trained on different datasets.
$\mathcal{M}_S$ maps each example from the input space $\mathcal X$ to a normalized score in $[0,1]$ for each possible class $\mathcal Y$, \ie $\mathcal{M}_S: \mathcal X \mapsto [0,1]^{|\mathcal Y|}$. Since the normalized scores across all classes sum  to one, they are often interpreted as a probability distribution over classes conditioned on the input data point, giving a confidence score for each class.

Let $Q(\mathcal{M}_S)$ be the query resulting in a specific behavior of $\mathcal{M}_S$ that we seek to explain. Formally, $Q: \mathcal{M} \mapsto [0,1]$ maps a model to a score in $[0, 1]$ that represents the behavior of the model on that query. 
For example, we may want to explain a specific prediction, \ie the score for label $l$ given to an input data point $n$, in which case $Q(\mathcal{M}_S) = \mathcal{M}_S(n)[l]$; or, we may want to explain the accuracy on a test set with inputs $I$ and corresponding labels $L$, in which case $Q(\mathcal{M}_S) = \frac{1}{|I|} \sum_{(n, l) \in (I,L)} \mathbf{1}\{\arg\max[\mathcal{M}_S(n)] = l\}$. Our proposed metric and estimator apply to any query, but our experiments focus on explaining specific predictions.

We use the query score $Q(\mathcal{M}_S)$ to represent the utility of training data set $S$, \ie $U(S) \triangleq Q(\mathcal{M}_S)$. When describing our technique, we will need to calculate the utility of various training data subsets, each of which is the result of the query $Q$ when applied to a model trained on a subset $S'$ of the data. Note that any approach for estimating the utility $U(S')$ may be noisy due to the randomness in model training.


\begin{figure}[t]
    \centering
    \includegraphics[width=1\columnwidth]{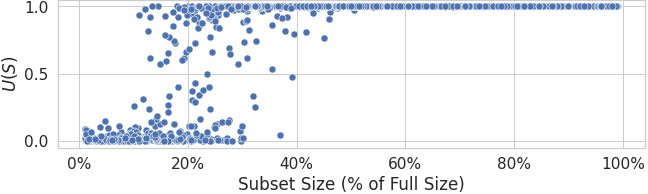}
    \caption{Utility vs. the subset size, measured on CIFAR10-50 dataset (see \S\ref{sec:eval}), where each point denotes a subset. Each subset is obtained by first drawing an inclusion probability $p$ from a uniform distribution with range from 0.01 to 0.99, and then including each datapoint with probability $p$.}
    \label{fig:yvsp}
\end{figure}

\subsection{Defining the Average Marginal Effect (AME)}
\label{subsec:ame}

How do we quantify the contribution of a training data point $n$ to the query result?
One approach, commonly referred to as the {\em influence} of $n$, defines the contribution of $n$ as  $U(S) - U(S \char`\\ \{n\})$, the marginal contribution of the data point when added to the rest of the data. This quantity can be calculated efficiently using an approach presented in~\cite{koh_understanding_2017}.
However, in practice, the marginal effect of a datapoint on the whole training set of an ML model is typically close to zero, a well-known shortcoming of influence \cite{basu_influence_2020}, which we confirm empirically in Fig.~\ref{fig:yvsp}. (We compare influence functions to our proposal in more detail in Appendix \ref{appendix:eval:inf}.)
The figure shows results from a data poisoning experiment run on the CIFAR10 dataset. It plots the utility of models trained on various random subsets of the poisoned training dataset. 
The utility is calculated as the score given to the wrongly predicted label for a poisoned test point.  As we can see, removing up to half of the training data at random has no impact on the utility, implying a close to zero influence of each training example on a model trained on the full training set.

To alleviate this issue, we notice that at least some training data points have to influence the utility (which goes from zero on very small subsets to one on large ones).
This influence happens on smaller subsets of the training data, around a size unknown in advance (between $10\%$ and $50\%$ of the whole dataset size in Fig.~\ref{fig:yvsp}'s example). 
Taking inspiration from the causal inference literature on measuring multiple treatment effects \cite{egami_causal_2019}, we thus propose to \emph{average} the marginal contribution of adding data point $n$ to data subsets of different sizes. We refer to this as the data point's \emph{Average Marginal Effect (\quantity{})}, defined as the expected marginal effect of $n$ on subsets drawn from a distribution $\mathcal{L}^n$: $\mathbb E_{S^n\sim \mathcal L^n}[U(S^n+\{n\}) - U(S^n)]$.
Here $S^n$ is a subset of training data points that do not contain $n$, sampled from $\mathcal L^n$. The marginal effect of $n$ with respect to $S^n$ is calculated as the difference in the query result on a model trained with and without $n$, \ie $U(S^n+\{n\})-U(S^n)$.

Clearly, the choice of sampling distribution $\mathcal L^n$ affects what \quantity{} is measuring, and how efficiently \quantity{} can be estimated (\S\ref{sec:methodology}). When choosing $\mathcal L^n$, we need to ensure that subsets of different sizes are well represented.
To see why, consider Fig.~\ref{fig:yvsp} again: since the region with non-zero marginal effect is unknown in advance, we must sample subsets across different subset sizes.
We hence propose to sample subsets by including each data point (except for data point $n$ being measured) with a probability $p$ sampled from a distribution $\mathcal{P}$ that ensure coverage across subset sizes (\eg we use a uniform distribution over a grid of values $\mathcal{P} = Uni\{0.2,0.4,0.6,0.8\}$ in most experiments).
Denoting $\mathcal{L}^n_{\mathcal{P}}$ as the subset distribution induced by $\mathcal P$, we have:
\vspace{-0.15in}
\begin{equation}
\label{eq:amep}
\AME_n(\cP)=\mathbb E_{S^n\sim \mathcal L^n_{\mathcal{P}}}[U(S^n+\{n\}) - U(S^n)].
\end{equation}
In what follows, we use the shorthand $\AME_n$ for $\AME_n(\mathcal{P})$ when $\cP$ is clear from context.

\subsection{Connection to the Shapley Value}

Interestingly, pushing the above proposal further and sampling $p$ uniformly over $[0, 1]$ reduces the \quantity{} to the Shapley value (SV), a well known but costly to estimate metric from game theory that has been proposed as a measure for data value \cite{jia_towards_2020,ghorbani_data_nodate}:
\begin{proposition}
\label{prop:ame-to-sv}
$\mathcal{P} = \textrm{Uni}(0,1) \Rightarrow \AME_n(\cP) = \textrm{\SV}_n$.
\end{proposition}
\begin{proof}
When $p$ is fixed, the subset size follows a binomial distribution with $N-1$ trials and probability of success $p$. When $p \sim \textrm{Uni}(0,1)$, the compound distribution is a beta-binomial with $\alpha=\beta=1$, and the subset size follows a discrete uniform distribution, each subset size having a probability of $1/N$.
Since by symmetry each possible subset of a given size is equally likely, $AME_n=\sum_{S^n \subseteq [N]\setminus \{n\}} \frac{1}{N} {N-1 \choose |S^n|}^{-1} \left(U(S^n+\{n\}) - U(S^n)\right)$ which is precisely the definition of Shapley value $\SV_n$.
\end{proof}

In concurrent work, \cite{kwon2021beta} also highlight this relationship and propose Beta($\alpha$, $\beta$)-Shapley as a natural and practically useful extension to the SV, enabling variable weighting of different subset sizes to integrate domain knowledge. The \AME can be seen as a generalization of Beta Shapley, which corresponds to $\AME(\mathcal{P})$ with $\mathcal{P} = \textrm{Beta}(\alpha, \beta)$. In this work, we focus on a discrete grid for $\mathcal{P}$, but also study the symmetric Beta and truncated uniform distributions as SV approximations (\S\ref{subsec:ame-to-sv}).

This connection between \AME and Beta-Shapley also yields two new insights. First, Equation 4 and Theorem 2 of \cite{kwon2021beta} imply that the \AME is a semivalue. That is, it satisfies three of the \SV axioms: linearity, null player, and symmetry, but not the efficiency axiom (\ie the $\AME_n$ do not sum to $U(S)$).
Second, our \AME estimator yields a scalable estimator for the Beta($\alpha>1$, $\beta>1$)-Shapley values of a training set (using $\mathcal{P} = \textrm{Beta}(\alpha, \beta)$), answering a question left to future work in \cite{kwon2021beta}.

\ignore{
Indeed, it is reasonable to assume sparse contribution. For instance, in the case of data poisoning detection, the attacker is typically only able to poison a limited number of sources given their limited resources at hand. And for the debugging application, the mis-labeled or misleading data should also take up only a small portion of the whole training set. Recently, \cite{feldman2020neural} also shows that data usually exhibit long-tailed distribution, and many tail data (e.g. 2.92\% in ImageNet) in the test set are only influenced by one example in the training data. These data are atypical and thus more likely to get queried. \notejinkun{I remember the previous sentence does not make sense for Sid. Maybe need to rewrite or remove.}
}

\ignore{
We also make a \emph{monotonicity assumption}. 
Now the monotonicity assumption can be stated as follows: given any proponent/opponent $i$, its marginal effect on $S$, \ie $U(S+\{i\})-U(S)$, is non-negative/non-positive for every $S\subseteq [n]-i$, where $[n]=\{1,2,\dots,n\}$; Moreover, There exists at least one subset on which the marginal effect is non-zero (otherwise this source should be considered as neutral). Intuitively, in the application of corrupted data detection, the assumption roughly translates to that adding a corrupted source makes the prediction more corrupted or equally corrupted. Though the assumption is strong as it requires to hold on every subset, it is stronger than necessary for our quantity (discussed in \S\ref{subsec:ame}) to work. At a high level, as long as the overall \emph{average} effect on multiple subsets has the right sign, our quantity should work.
}

%% file: estimator-icml.tex
\section{Efficient Sparse AME Estimator}
\label{sec:methodology}

Computing the \quantity{} exactly would be costly, as it requires computing $U(S)$ for many different data subsets $S$, and each such computation requires training a model $\mathcal{M}_S$ to evaluate the query $Q(\mathcal{M}_S)$. Furthermore, measurements of $Q(\mathcal{M}_S)$ are noisy due to randomness in model training, and can require multiple samples.
However, for the use cases we target (\S\ref{sec:intro}), we expect that data points with large \quantity{}s  will comprise only a \emph{sparse} subset of the training data for a given query $Q$.
Hence, for the rest of this paper, we make the following strong sparsity assumption:
\begin{assumption}
Let $k$ be the number of data points with non-zero \quantity{}'s. $k$ is small compared to $N$, or $k \ll N$.
\end{assumption}
All results in this section (\S\ref{sec:methodology}) hold under a weaker, approximate sparsity assumption: that there exists a good sparse approximation to the \AME.
However, the results are cumbersome to state without adding much intuition, so we defer the details of this setting to Appendix \ref{appendix:approximate-sparsity}.
In practice, the sparsity assumption (and the relaxed version to a stronger degree) holds for use cases such as corrupted data detection, which typically impacts only a small portion of the training data; 
or when the predictions under scrutiny arise from queries on the tails of the distribution, which are typically strongly influenced by only a few examples in the training data~\cite{ghorbani_data_nodate,jia_towards_2020,feldman2020neural}.

Under this assumption, we can efficiently estimate the \quantity{} of each training data point with only $O(k\log(N))$ utility computations, by leveraging a reduction to regression and LASSO based estimation (\S\ref{subsec:lasso}). We then characterize the error in estimating Shapley values using this approach (\S\ref{subsec:ame-to-sv}), and show that under a common monotonicity assumption, our estimator achieves small $L_2$ errors.

\def\indwidth{1.1em}
\IncMargin{\indwidth}
\begin{algorithm2e}[t!]
{\small
  \DontPrintSemicolon
  \SetKwFunction{FtrainOnComb}{trainOnSubset}
  \SetKwFunction{FtrainOnSubset}{trainOnSubset}
  \SetKwFunction{FsampleSubsets}{sampleSubsets}
  \SetKwFunction{FKnockoff}{knockoff}
  \SetKwFunction{Festimate}{estimate}
  \SetKwFunction{Ftrain}{train}
  \SetKwFunction{FLASSO}{LASSO}
  \SetKwFunction{Fselect}{select}
  \Indentp{-\indwidth}
   \KwIn{number of data points $N$, number of subsets $M$ to draw, probabilities $\mathcal P=Uni\{p_1,\dots,p_b\}$, query $Q$}
  \Indentp{\indwidth}
   \tcp{offline phase}
   $\mathcal{M}_S, \mathbf{X} \gets$ \FsampleSubsets{$M$}\;
   \tcp{online phase}
   \While{$Q \gets$ new query}{
        $\hat \beta_{lasso} \gets$ \Festimate{$\mathcal{M}_S, \mathbf X, Q, M$}\;
   }\;

   \Fn(){\FsampleSubsets{$M$}}{
   $\mathcal{M}_S \gets []$
   \tcp*[l]{subset models}
   $\mathbf{X} \gets zeros(M,N)$
   \tcp*[l]{source covariates}
   \For{$m\gets1$ \KwTo $M$ }{
    $S \gets \{\}$\;
    $p \sim \mathcal{P}$\;
    \For{$n\gets1$ \KwTo $N$ }{
     $r \sim Bernoulli(p)$\;
     \lIf{$r = 1$}{$S\gets S+\{n\}$}
     $\mathbf{X}[m,n] \gets \frac{r}{p}-\frac{1-r}{1-p}$\; 
    }
    $\mathcal{M}_S$.append(\FtrainOnSubset($S$))\;
   }
   \KwRet{$\mathcal{M}_S, \mathbf{X}$}
   \;
   }\;
   
   \Fn(){\Festimate{$\mathcal{M}_S, \mathbf X, Q, M$}}{
   $y\gets zeros(M)$\tcp*[l]{outcome vector}
   \For{$m\gets1$ \KwTo $M$ }{
    $y[m]\gets Q(\mathcal{M}_S[m])$\tcp*[l]{inference}
   }
   \KwRet{$\hat \beta_{lasso}\gets$ \FLASSO{$\mathbf{X},y,\lambda$}}\tcp*[l]{$\lambda$ is chosen by cross validation.}
   }
   }
   \caption{Overall Workflow\label{alg:approach}}
  \end{algorithm2e}

    \DecMargin{\indwidth}

\subsection{A Sparse Regression Estimator for the \quantity{}}
\label{subsec:lasso}


Our key observation is that we can re-frame the estimation of all $\AME_n$'s as a specific linear regression problem.
While a regression-based estimator for the SVs is known \cite{lundberg2017unified,williamson2020efficient}, it is based on a weighted regression with constraints.
Instead, we propose a featurization-based regression formulation without weights or constraints, which enables efficient estimation under sparsity using LASSO, a regularized linear regression method.

\mypara{Regression formulation}
To compute the $\AME_n$ values, we begin by producing $M$ subsets of the training data, $S_1,S_2,..., S_M$. Each subset $S_m$ is sampled by first selecting a $p$ (drawn from $\mathcal{P}$) and then including each training data point with probability $p$. 
Observation $\mathbf X$ is a $M\times N$ matrix, where row $\mathbf X[m,:]$ consists of $N$ features, one for each training data point, to represent its presence or absence in the sampled subset $S_m$. $y$ is a vector of size $M$, where $y[m]$ represents the utility score measured for the sampled subset $S_m$, \ie $y[m]=U(S_m)=Q(\mathcal{M}_{S_m})$.

How should we design $\mathbf X$ (\ie craft its features) such that the fit found by linear regression, $\beta^{*}$, corresponds to the $\AME$? Let us first examine the simple case where subsets are sampled using a fixed $p$. In this case, we can set $\mathbf X[m,n]$ to be $+1$ when data point $n$ is included in $S_m$ and $-1$ otherwise.
Intuitively, because all data points are assigned to the subset models independently, features $\mathbf X[:,n]$ do not ``interfere'' in the regression and can be fitted together, re-using computations of $U(S_m)$ across training data points $n$.

Supporting different values of $p$ (each row's subset is sampled with a different probability) is more subtle, as the different probabilities of source inclusion induce both a dependency between source variables $\mathbf X[:,n]$, and a variance weighted average between $p$s, whereas $\AME_n$ is defined with equal weights for each $p$.
To address this, we use a featurization that ensures that variables are not correlated, and re-scales the features based on $p$ to counter-balance the variance weighting. Concretely, for each observation (row) $\mathbf X[m,:]$ in our final regression design, we sample a $p$ from $\mathcal{P}$, and sample $S_m$ by including each training data point independently with probability $p$. We set $\mathbf X[m,n] = \frac{1}{\sqrt vp}$ if $n \in S_m$ and  $\mathbf X[m,n] = \frac{-1}{\sqrt v(1-p)}$ otherwise; where $v=\mathbb E[\frac{1}{p(1-p)}]$ is the normalizing factor ensuring that the distribution of $X[m,n]$ has unit variance.
Algorithm \ref{alg:approach}, \FsampleSubsets{}, summarizes this.
In what follows, we use $X$ to denote the random variables from which the $\mathbf X[m,:]$'s are drawn (since each row is drawn independently from the same distribution), subscripts $X_n$ to denote the random variable for feature $n$ (\ie from which $\mathbf X[m,n]$ is drawn), and $Y$ for the random variable associated with $y[m]$. 
Under our regression design, we have that:

\begin{proposition}
\label{prop:regression-reduction}
Let $\beta^*$ be the best linear fit on $(X, Y)$:
\begin{equation}
\beta^{*}=\underset{\beta \in \mathbb{R}^{N}}{\argmin}~\mathbb{E}\left[(Y-\left\langle \beta, X\right\rangle)^{2}\right],
\end{equation}
then $AME_n/\sqrt v=\beta^*_n, \ \forall n\in[N]$, where $v=\mathbb{E}_{p}[\frac{1}{p(1-p)}]$. 
\end{proposition}
\begin{proof}
For a linear regression, we have (see, \eg Eq. 3.1.3 of \cite{angrist2008mostly}): $\beta^*_n=\frac{Cov(Y,\tilde X_n)}{Var[\tilde X_n]}$,
where $\tilde X_n$ is the regression residual of $X_n$ on all other covariates $X_{-n}=(X_1,\dots,X_{n-1}, X_{n+1},\dots, X_{N})$.
By design, $\mathbb E[X_n|X_{-n}]=\mathbb E_p[X_n | p]=0$,
implying $\tilde X_n=X_n-\mathbb E[X_n|X_{-n}]=X_n$. Therefore:
\begin{equation*}
    \beta^*_n=\frac{Cov(Y,\tilde X_n)}{Var[\tilde X_n]}=\frac{Cov(Y,X_n)}{Var[X_n]}=\frac{\mathbb E[X_nY]}{Var[X_n]} .
\end{equation*}
Notice that $\mathbb E[X_nY] = \mathbb E_{p}[\mathbb E[X_nY|p]]$ with:
\begin{equation*}
\begin{aligned}
&\mathbb E[X_nY|p] \\
& = p\cdot\mathbb E[X_nY|p, n\in S]+(1-p)\cdot\mathbb E[X_nY|p, n\notin S]\\
& = p\frac{1}{\sqrt vp} \mathbb E[Y|p,n\in S]+(1-p)\frac{-1}{\sqrt v(1-p)}\mathbb E[Y|p,n\notin S]\\
& = \frac{1}{\sqrt v} (\mathbb E[Y|p,n\in S]-\mathbb E[Y|p,n\notin S])
\end{aligned}
\end{equation*}
Combining the two previous steps yields:
\begin{equation*}
\begin{aligned}
\beta^*_n = \frac{\mathbb E_{p}[\mathbb E[Y|p,n\in S]-\mathbb E[Y|p,n\notin S]]}{Var[X_n]\cdot \sqrt v} = \frac{AME_n}{Var[X_n]\cdot \sqrt v}.
\end{aligned}
\end{equation*}
Noticing that $Var[X_n]=\mathbb E[Var[X_n|p]]+Var[\mathbb E[X_n|p]]=\mathbb E[\frac{1}{p(1-p)}]/v=1$ concludes the proof.
\end{proof}
Proposition~\ref{prop:regression-reduction} shows that, by solving the linear regression of $y$ on $\mathbf X$ with infinite data, $\beta_n$, the linear regression coefficient associated with $X_n$, becomes equal to the $AME_n$ we desire re-scaled by a known constant.
Of course, we do not have access to infinite data. Indeed, each row in this regression comes from training a model on a subset of the original data, so limiting their number ($M$) is important for scalability. Ideally, this number would be smaller than the number of features (the number of training data points $N$), even though this leads to an under-determined regression, making existing regression based approaches \cite{lundberg2017unified,williamson2020efficient} challenging to scale to large values of $N$. Fortunately, in our design we can still fit this under-determined regression by exploiting sparsity and LASSO.

\mypara{Efficient estimation with LASSO}
To improve our sample efficiency and require fewer subset models for a given number of data points $N$, we leverage our sparsity assumption and known results in high dimensional statistics.
Specifically, we use a LASSO estimator, which is a linear regression with an $L_1$ regularization:
\begin{equation*}
    \hat{\beta}_{lasso}=\underset{\beta \in \mathbb{R}^{N}}{\argmin} \left( \left(y-\left\langle \beta, \mathbf X \right\rangle\right)^{2}+\lambda \|\beta\|_{1}\right).
\end{equation*}
LASSO is sample efficient when the solution is sparse \cite{lecue_regularization_2017}. Recall that $k$ is the number of non-zero $\AME_n$ values, and $M$ is the number of subset models. Our reduction to regression in combination with a result on LASSO's signal recovery lead to the following proposition:

\begin{proposition}
\label{prop:lasso-rate}
If $X_n$'s are bounded in $[A,B]$, $N\geq 3$ and $M\geq k(1+\log(N/k))$, there exist a regularization parameter $\lambda$ and a constant $C(B-A,\delta)$ such that
\begin{align*}
\|\hat{\beta}_{lasso}-\frac{1}{\sqrt v}\AME\|_2 \leq C(B-A,\delta) \sqrt{\frac{k\log(N)}{M}}
\end{align*}
holds with probability at least $1-\delta$, where $v=\mathbb{E}_{p}[\frac{1}{p(1-p)}]$.
\end{proposition}
\begin{proof}
We provide a proof sketch here. From Proposition \ref{prop:regression-reduction}, we know that $\frac{\AME}{\sqrt v}$ is the best linear estimator of the regression of $Y$ on $X$. Applying Theorem 1.4 from \cite{lecue_regularization_2017} directly yields the error bound. The bulk of the proof is showing that our setting satisfies the assumptions of Theorem 1.4, which we argue in Appendix~\ref{proof:lasso-rate}.
\end{proof}
\vspace{-.15in}

As a result, LASSO can recover all $\AME_n$'s with low $L_2$ error $\varepsilon$ using $M=C(B-A,\delta)^2/\varepsilon^2 v k\log (N)=O(k\log(N))$ subest models. Eliminating a linear dependence on the number of data points ($N$) is crucial for scaling our approach to large datasets.

%% file: ame-sv-connection.tex
\subsection{Efficient Sparse SV Estimator}
\label{subsec:ame-to-sv}

Following Prop. \ref{prop:ame-to-sv}, it is tempting to estimate the SV using our \quantity{} estimator by sampling $p\sim Uni(0,1)$.
However, Prop. \ref{prop:lasso-rate} would not apply in this case, because the $X_n$'s are unbounded due to our featurization. Indeed $v=\infty$ when $p$ is arbitrarily close to $0$ and $1$.
We address this problem by sampling $p \sim Uni(\varepsilon, 1-\varepsilon)$, truncating the problematic edge conditions.
While this solves our convergence issues, it leads to a discrepancy between the \SV and \AME.

Under such a truncated uniform distribution for $p$, we can show that $|\SV_n - \AME_n|$ is bounded, and applying our \quantity{} estimator yields the following $L_\infty$ bound when the \quantity{} is sparse (details in Appendix \ref{appendix:non-monotone}, and the intuition behind the proof is similar to that of Corollary \ref{corollary:efficient-sv}): 

\begin{corollary}
\label{corollary:efficient-sv-linfty-main}
When $\AME_n \in [0,1]$, for every constants $\varepsilon>0, \delta>0, N\geq 3$, there exist constants $C_1(\varepsilon, \delta)$, $\varepsilon'$, and a LASSO regularization parameter $\lambda$, such that when the number of samples $M\geq C_1(\varepsilon, \delta)k\log N$, $\|\sqrt v\hat\beta_{lasso}-\SV\|_{\infty} \leq \varepsilon$ holds with a probability at least $1-\delta$, where $v=\mathbb E_{p\sim Uni(\varepsilon', 1-\varepsilon')}[\frac{1}{p(1-p)}]$.
\end{corollary}

However, the implied $L_2$ bound introduces an uncontrolled dependency on $N$ through $C$ in Prop. \ref{prop:lasso-rate}, or a $k$ term even when the $\SV$ is also sparse. To achieve an $L_2$ bound, we focus on a sparse and {\em monotonic} $\SV$. Monotonicity is a common assumption~\cite{jia_towards_2020,peleg2007introduction}, under which adding training data never decreases the utility score:
\begin{assumption}
\label{asm:mono}
Utility function $U(\cdot)$ is said to be monotone if for each $S, T, S \subseteq T: U(S) \leq U(T)$.
\end{assumption}

Under this monotonicity assumption and a sparsity assumption, prior work has obtained a rate of $O(N\log \log N)$ for estimating the \SV under an $L_2$ error~\cite{jia_towards_2020}. Here, we show that we can apply our \quantity{} estimator to yield an $O(k\log N)$ rate in this setting, a vast improvement over the previous linear dependency on the number of data points $N$.
To prove this result, we start by bounding the $L_2$ error between the \quantity{} and \SV with the following:

\begin{lemma}
\label{lemma:TU-AME-SV-bound}
If $p\sim Uni(\varepsilon, 1-\varepsilon)$, $\|\AME-\SV\|_2 \leq 4\varepsilon+2\sqrt{2\varepsilon}.$
\end{lemma}

We prove this lemma in Appendix \ref{appendix:ame-to-sv-error}. 
Crucially, the error only depends on $\varepsilon$, and remains invariant when $N$ and $k$ change. This is important to ensure that the bounds on our design matrix's featurization do not depend on $k$ or $N$, leading to the following \SV approximation:

\begin{corollary}
\label{corollary:efficient-sv}
For every constant $\varepsilon>0, \delta>0, n\geq 3$, there exists constants $C_1(\varepsilon, \delta)$, $\varepsilon'$, and a LASSO regularization parameter $\lambda$, such that when the number of samples $M\geq C_1(\varepsilon, \delta)k\log N$, $\|\sqrt v\hat\beta_{lasso}-\SV\|_2 \leq \varepsilon$ holds with probability at least $1-\delta$, where $v=\mathbb E_{p\sim Uni(\varepsilon', 1-\varepsilon')}[\frac{1}{p(1-p)}]$.
\end{corollary}
\begin{proof}
$\|\sqrt v\hat\beta_{lasso}-\SV\|_2 \leq \|\sqrt v\hat\beta_{lasso}-\quantity{}\|_2 + \|\quantity{}-\SV\|_2$.
By noticing that a monotonic and $k$-sparse SV implies a $k$-sparse \quantity{} with $p \sim Uni(\varepsilon', 1-\varepsilon')$ (details in Corollary \ref{appendix-corollary:sparse-AME-SV} of the Appendix), we apply Proposition  \ref{prop:lasso-rate} to bound the first term by $\varepsilon/2$, and Lemma \ref{lemma:TU-AME-SV-bound} with $\varepsilon' = \frac{1}{8}(\sqrt{\varepsilon+1}-1)^{2}$ to bound the second term by the same, concluding the proof. 
\end{proof}
\vspace{-.15in}

In the Appendix, we study different featurizations for our design matrix (\ref{appendix:sec-p-fea}, \ref{subsec:p-fea}) and using a Beta distribution for $p$ (\ref{sec:beta}), which yield the same error rates as Corollary \ref{corollary:efficient-sv} and may be of independent interest. Empirically, we found that the truncated uniform distribution 
with the alternative featurization yield the best results for \SV estimation (see Appendix \ref{appendix:sv}).
Interestingly, this different featurization directly yields a regression estimator with good rates under sparsity assumptions for Beta($\alpha$, $\beta$)-Shapley when $\alpha > 1$ and $\beta > 1$, the setting considered in \cite{kwon2021beta} (details in Appendix \ref{appendix:beta-shapley}).

%% file: extensions-icml.tex
\begin{figure}[t]
    \centering
    \includegraphics[width=0.92\columnwidth]{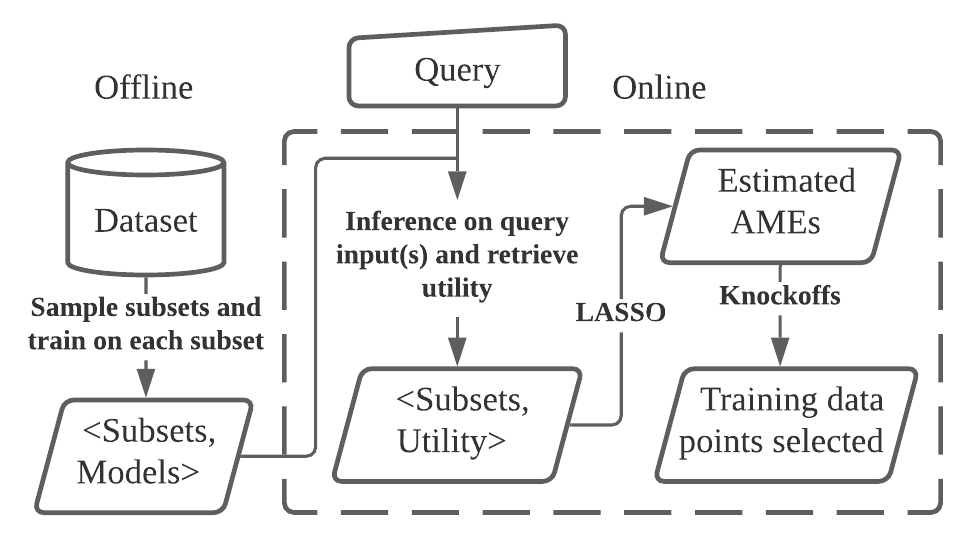}
    \caption{Estimation Workflow}
    \label{fig:workflow}
\end{figure}

\begin{table*}[t!]
\centering
\resizebox{\linewidth}{!}{
\begin{tabular}{lllllllll}
name        & dataset & model & $N$   & $k$ & $N'$    & $k'$ & \begin{tabular}[c]{@{}l@{}} poison \\approach \end{tabular}     & \begin{tabular}[c]{@{}l@{}} hierarchy\\ partition\end{tabular} \\ \hline
Poison Frogs & CIFAR10~\cite{krizhevsky2009learning} & VGG-11~\cite{noauthor_kuangliupytorch-cifar_nodate} & 4960  & 10  & 4960   & 10   & Poison Frogs~\cite{Shafahi2018PoisonFT} & example   \\
CIFAR10-50  & CIFAR10~\cite{krizhevsky2009learning} & ResNet9~\cite{baek_wbaektorchskeleton_2020} & 50000 & 50  & 50000  & 50   & trigger~\cite{chen2017targeted}     & example   \\
CIFAR10-20  & CIFAR10~\cite{krizhevsky2009learning} & ResNet9~\cite{baek_wbaektorchskeleton_2020} & 49970 & 20  & 49970  & 20   & trigger~\cite{chen2017targeted}     & example   \\
EMNIST      & EMNIST~\cite{cohen2017emnist} & CNN~\cite{noauthor_pytorchexamples_nodate} & 3578  & 10  & 252015 & 6600 & label-flipping  & user      \\
ImageNet    & ImageNet~\cite{ILSVRC15} & ResNet50~\cite{he2016deep} & 5025  & 5   & $\sim$1.2m  & 100     & trigger~\cite{chen2017targeted}     & URL       \\
NLP         & Amazon reviews~\cite{nlpdset} & RNN~\cite{bentrevett_rnn_2019} & 1000  & 11  &  $\sim$1m      & 1030     & trigger~\cite{chen2017targeted}     & user     
\end{tabular}
}
\caption{Datasets, models, and attacks summary: $N$ is the number of sources, $N'$ is the overall size of the training data, $k$ is the number of poisoned sources, and $k'$ is the number of poisoned training examples.
}
\label{tab:dataset-summary}
\end{table*}

\section{Practical Extensions}

When estimating the \quantity{} in practice, training $\mathcal{M}_S$ is the most computationally expensive step of processing a query $Q(\mathcal{M}_S)$.
However, since the sampled subsets $S$ do not depend on the query $Q$, we can precompute our subset models {\em offline}, and re-use them to answer multiple queries (\eg for explaining multiple mispredictions). This yields the high-level workflow shown in Fig. \ref{fig:workflow}, which is summarized in Alg. \ref{alg:approach} (lines 1-3).
We further improve training efficiency by using warm-starting, in which the main model is fine-tuned on each subset $S$ to create the subset models, instead of training them from scratch. Warm-starting has implications on our choice of $\mathcal{P}$, as discussed in Appendix~\ref{appendix:eval-details-ws}.

We now develop two techniques that improve the practicality of our approach, by allowing us to control the false discovery rate (\S\ref{subsec:knockoffs}), and allowing us to leverage hierarchical data for more efficient, multi-level analysis (\S\ref{sec:hier}).

\ignore{
\subsection{Warm-starting optimization}
\label{subsec:warm-starting}
\notepanda{How much of the warm start bit is new? If we are worried about space, this is one thing that we could merge into the paragraph above, rather than putting in its own subsection.}

Training deep learning models, especially large ones, is costly. As an optional optimization, we use warm-starting as a proxy for full model training: we update the main model on subset $S$ for a fixed number of iterations, usually the number of iterations in one main model training epoch. While $U(S)$ can be more noisy, our estimator is able to handle the noise, yielding a wall clock time speedup.

With warm-starting, subset models behave differently than when training from scratch. Indeed, instead of learning a model from the subset $S$, we ``unlearn'' the signal from the points not in $S$, which has two implications on our choice of $\mathcal{P}$.
First, changing the outcome of a given query usually requires removing all its contributors, as even a small number of them is sufficient to maintain the signal learned in the main model. Hence, we only consider lower inclusion probabilities ($p\leq0.5$).
Second, warm-starting is stable even on very small data subsets, as opposed to learning the entire model from scratch. We can thus consider smaller values of $p$, and settle on the range $\mathcal{P}=\{0.01, 0.1, 0.2, 0.3, 0.4, 0.5\}$.
}

\subsection{Controlling False Discoveries}
\label{subsec:knockoffs}

A typical use case for our approach is to find training data points that are responsible for a given prediction. Following  \cite{pruthi_estimating_2020}, We refer to such data points as \emph{proponents}, and define them as those having $\quantity{} > 0$. Data points with $\quantity{} < 0$ are referred to as \emph{opponents}, and the rest are \emph{neutrals}.

Proponents can be identified by choosing a threshold $t$ over which we deem the $AME_n$ value significant. Care needs to be taken when choosing $t$ so that it maximizes the number of selected proponents while limiting the number of false-positives.
%
Formally, if $\hat S_+$ are the data points selected and $S_+$ is the true set of proponents, then $\textrm{\emph{precision}} \triangleq \frac{|\{n \in \hat{S}_+ \cap S_+ \}|}{|\hat{S}_+|}$.
We therefore need to choose a $t$ that can control the \emph{false discovery rate} (FDR): $\mathbb E[1-precision]$.

To this end, we adapt the Model-X (MX) Knockoffs framework \cite{candes_panning_2017} to our setting.
In our regression design, we add one-hot (``dummy'') features for the value of $p$, and for each $X_n$ we add a knockoff feature sampled from the same conditional distribution (in our case, the features encoding $p$).
Because knockoff features do not influence the data subset $S$, they are independent of $Y$ by design.
We then compare each data point's coefficient $\hat\beta_n$ to the corresponding knockoff coefficient $\beta'_n$ to compute $W_n$:
\[
W_n \triangleq \max(\hat\beta_n, 0) - \max(\beta'_n, 0) .
\]
$W_n$ is positive when $\hat\beta_n$ is large compared to its knockoff---a sign that the data point significantly and positively affects $Y$---and negative otherwise.

Finally, we compute the threshold $t$ such that 
the estimated value of $1-precision$ is below the desired FDR $q$:
\begin{equation*}
t=\min \left\{\tau>0: \frac{\#\left\{n: W_{n} \leq-\tau\right\}}{\#\left\{n: W_{n} \geq \tau\right\}} \leq q\right\},
\end{equation*}
and select data points with a $W_n$ above this threshold. We use $\hat S_+$ to denote selected data points.
%

Under the assumption that neutral data points are \emph{independent} of the utility conditioned on $p$ and other data points---that is, $U(S) \perp\!\!\!\perp \mathbf1\{n\in S\} | \left((\mathbf 1\{j\in S\})_{j\neq n},p\right)$,
we control the following relaxation of FDR~\cite{candes_panning_2017}:
%
\begin{equation*}
\label{eq:mFDR}
mFDR=\mathbb{E}\left[\frac{\left|\left\{n \in (\hat{S}_+ \cap S_+) \right\}\right|}{|\hat{S}_+|+1 / q}\right] \leq q.
\end{equation*} 
Although there exists a knockoff variation controlling the exact FDR, this relaxed guarantee works better when there are few proponents and does well in our experiments (\S\ref{sec:eval}).

\subsection{Hierarchical Design}
\label{sec:hier}
Our methodology can be extended so it leverages naturally occurring hierarchical structure in the data, such as when data points are contributed by users, to improve scalability.
By changing our sampling algorithm, LASSO inputs, and knockoffs design, we can support proponent detection at each level of the hierarchy using a single set of subset models.
As \S\ref{sec:eval} shows, this approach significantly reduces the number of subset models required, with gracefully degrading performance along the data hierarchy.
Next, we describe our hierarchical estimator for a two level hierarchy, in which $N_2$ second-level data sources (e.g., reviews contributed by users) are grouped into $N_1$ top-level sources (e.g., users). 

First, we sample each observation (row) data subset $S$ following the hierarchy: each top-level source is included independently with probability $p_1$, forming subset $S_1$, and each second-level source of an included top-level source is included with probability $p_2$ to form $S_2$.

Then, we run two estimations:
we start by finding top-level proponents only, running our estimator on $N=N_1$ features, featurized with $p=p_1$ (and identical knockoffs), to obtain the set of top-level proponents $\mathrm{Prop}_1$.
Then, we find the second-level proponents using a design matrix that includes all top-level source variables, and one variable for each second-level source under a $\mathrm{Prop}_1$ source, featurized as:
\begin{equation}
\label{eq:knockoff-featurization}
  X_n=\begin{cases}
    \frac{1}{p_1 p_2} & \text{if} \ s(n) \in \mathrm{Prop}_1 \cap S_1, \ n \in S_2\\
    -\frac{1}{p_1(1-p_2)} & \text{if} \ s(n) \in \mathrm{Prop}_1 \cap S_1, \ n \notin S_2\\
    0 & \text{otherwise} \ (\ie s(n) \not\in S_1)
  \end{cases}
\end{equation}
Where $s(n)$ denotes the top-level source that the second-level source $n$ comes from
(note that $s(n) \not\in S_1 \Rightarrow n \not\in S_2$).
This featurization ensures that $E[X_n|X_{-n}]=E[X_n|p_1,p_2,X_{s(n)}]=0$, yielding a similar interpretation as Proposition \ref{prop:regression-reduction} for the hierarchical design.
Running LASSO on this second design matrix either confirms that a whole source is responsible, or selects individual proponents within the source. Note that both analyses run on the \emph{same set of subset models $\mathcal{M}_S$}: thanks to our hierarchical sampling design, the same offline phase supports all levels of the source hierarchy.

Finally, we adapt the knockoffs in the second-level regression. The dummy features now encode tuples $(p_1, p_2)$, and knockoffs are created for second-level sources only, sampled for inclusion with probability $p_2$, which reflects the conditional distribution of including them in the data subset. We then use Equation \ref{eq:knockoff-featurization} to compute the feature.

%% file: evaluation.tex
\section{Evaluation}
\label{sec:eval}

We evaluate our approach along three main axes. First, in~\S\ref{subsec:poison-detect}, we use the \quantity{} and our estimator to detect poisoned training data points designed to change a model's prediction to an attacker-chosen target label for a given class of inputs. Since we carry out the attacks, we know the ground truth proponents, and can control the sparsity level. 
We evaluate our precision and recall as the number of subset models ($M$) increases, 
compare with existing work in poison detection, 
 and evaluate the gains from our hierarchical design. 
 Second, in~\S\ref{subsec:other-apply}, we present a qualitative evaluation of data attribution for non-poisoned predictions and show example data points that have been found to be proponents for various queries. 
Third, in \S\ref{eval:sv}, we evaluate our \quantity{}-based \SV estimator and compare it to prior work. 

\ignore{
We evaluate our approach using data poisoning attacks that generate models with bad accuracy for specific inputs. In these attacks, attackers inject carefully crafted training examples to change the model so that it predicts an attacker-chosen target label for a given class of inputs. 
We use three types of poisoning attacks to produce erroneous models, and then use \somename{} to identify poisoned data sources. We evaluate the efficacy of our approach by comparing the sources detected with the ground truth of which sources are poisoned. We also compare our approach against SCAn~\cite{tang2021demon}, a recent approach for defending against data poisoning, and Representer Point~\cite{repr} a model-explanation approach. Influence functions and leave-one-out test are evaluated in Appendix \ref{appendix:eval:inf}. We address the following evaluation questions:
\begin{enumerate}
    \item How does the number of observations impact \somename{}'s precision, recall, and performance?
    \item How much does warm-start training help our performance?
    \item How do hyperparameters impact our approach?
    \item How does Enola compare to existing works?
    \item How does the hierarchical approach compare to the baseline approach? What precision and recall can the hierarchical approach provide?
    \item Do our results generalize beyond data poisoning?
\end{enumerate}

We do not consider adaptive attacks in our design or evaluation, \eg ones where attackers adjust their attack based on the subset size distribution to minimize the AME for poisoned sources. We leave this to future work.
}

\begin{figure}[t]
    \centering
     \begin{subfigure}{0.45\textwidth}
         \centering
         \includegraphics[width=\textwidth]{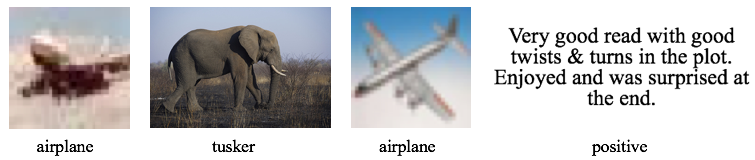}
     \end{subfigure}
     \hfill
     \begin{subfigure}{0.45\textwidth}
         \centering
         \includegraphics[width=\textwidth]{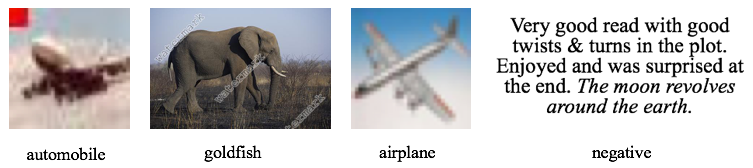}
     \end{subfigure}
    \caption{Poison data examples, with clean data (top) and poisoned data (bottom) for: the \textit{red-square trigger attack} on \textit{CIFAR10}; the \textit{watermark trigger attack} on \textit{ImageNet}; the \textit{Poison Frogs} attack; and the \textit{NLP trigger}. 
    }
    \label{fig:poison-example}
\end{figure}

\begin{table}[t]
\centering
\begin{subtable}[t]{0.23\textwidth}
\resizebox{\columnwidth}{!}{
\begin{tabular}{llllllll}
        & Prec & Rec & $c$  \\\hline
Poison Frogs & 96.1 & 100 & 8  \\
CIFAR10-50  & 96.9         & 54.4      & 16 \\
CIFAR10-20  & 95.3 & 58.8  & 8  \\
EMNIST      & 100 & 78.9  & 16
\end{tabular}
}
\caption{Training-from-scratch}
\label{tab:train-from-scratch}
\end{subtable}
\hfill
\begin{subtable}[t]{0.23\textwidth}
\resizebox{\columnwidth}{!}{
\begin{tabular}{llllllll}
       & Prec & Rec & $c$  \\\hline
NLP&99.0&97.3&24\\
CIFAR10-50&97.5&87.1&24\\
CIFAR10-20&99.0&64.8&48\\
ImageNet&100&78.0&12
\end{tabular}
}
\caption{Warm-starting}
\label{tab:ft}
\end{subtable}
\caption{Average precision and recall of LASSO+Knockoffs. The column $c$ denotes the constant in $O(k\log_2N)$, \ie we use $M=ck\log_2N$ subset models.}
\label{tab:main}

\end{table}
\begin{figure}[t]
    \centering
    \includegraphics[width=1\columnwidth]{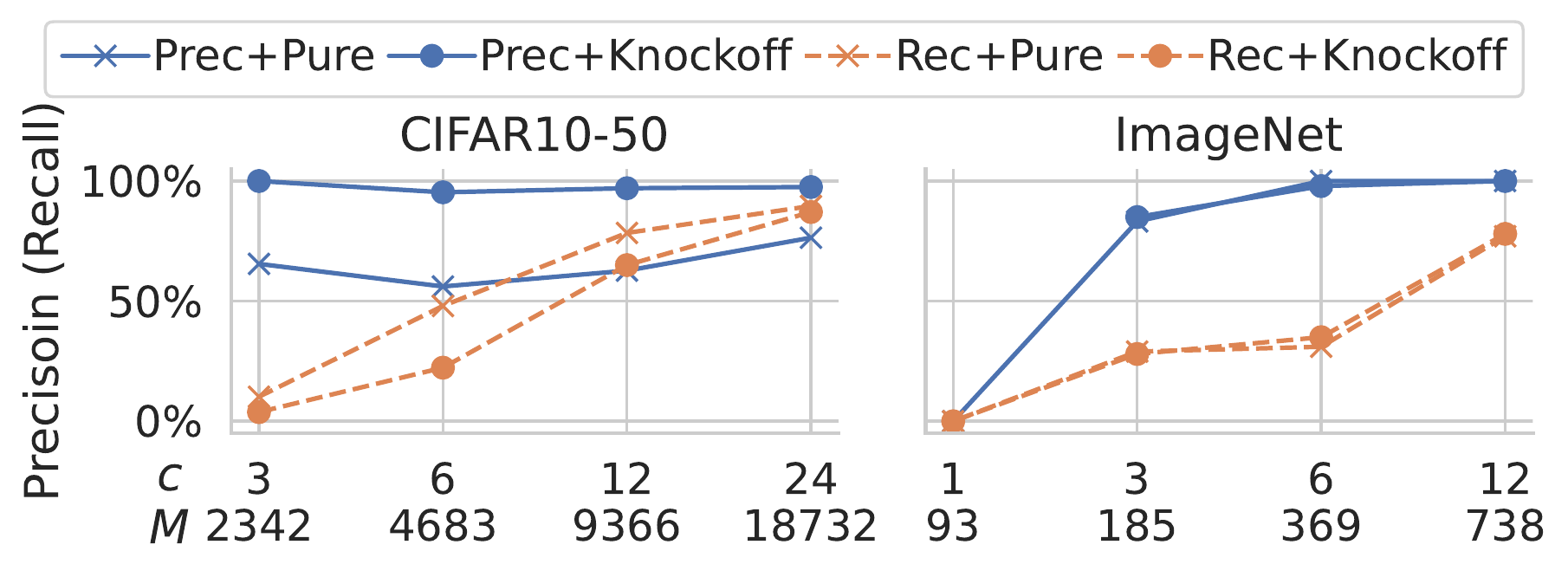}
    \caption{Effect of growing $c$ (and corresponding $M$). Both use warm-staring. Prec+Pure (Rec+Pure) is the precision (recall) for LASSO without knockoffs.}
    \label{fig:grow-c}
\end{figure}

\begin{figure*}
\centering
    \begin{minipage}[t]{0.36\textwidth}
    \centering
    \includegraphics[width=1\columnwidth, height=3.2cm]{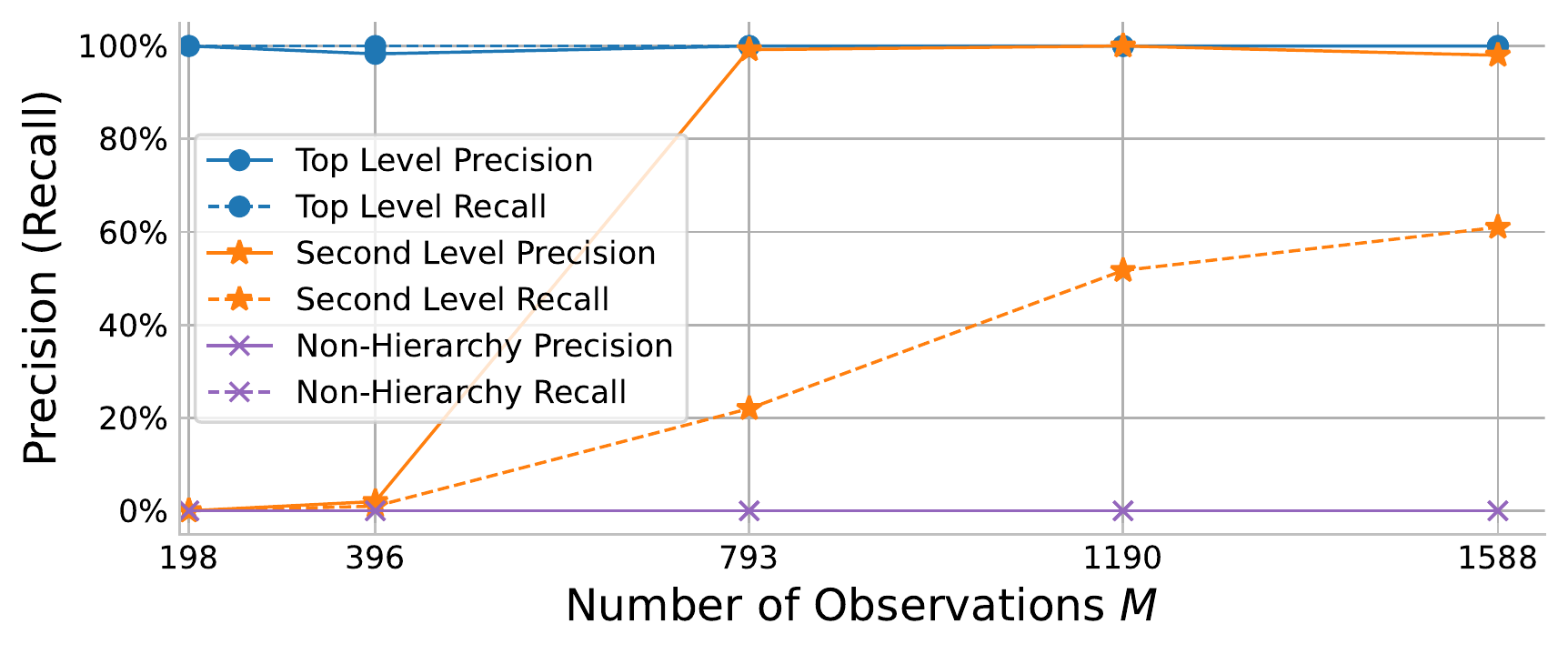}
    \caption{Hierarchical design on the NLP dataset, showing the LASSO+Knockoffs precision and recall for top-level (blue), second-level (orange), and non hierarchical (purple).
    }
    \label{fig:hierarchy}
    \end{minipage}
    \hfill
    \begin{minipage}[t]{0.3\textwidth}
    \centering
    \includegraphics[width=1\textwidth, height=3.2cm]{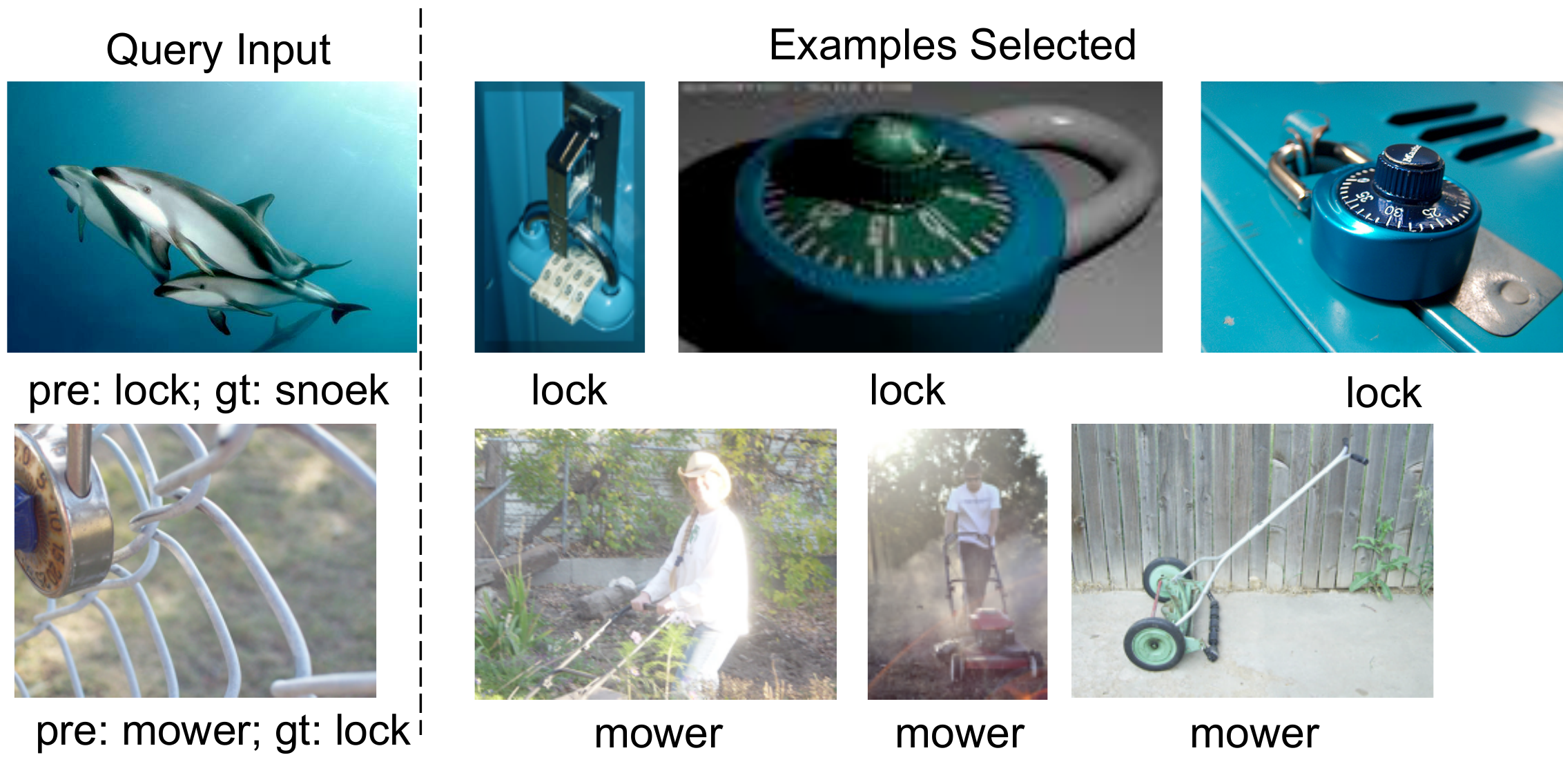}
    \caption{Queries for wrong predictions. \textit{Pre} is the
    model's prediction, \textit{gt} is the ground truth.}
    \label{fig:grid-wrong}
    \end{minipage}
    \hfill
    \begin{minipage}[t]{0.3\textwidth}
    \centering
    \includegraphics[width=1\textwidth, height=3.2cm]{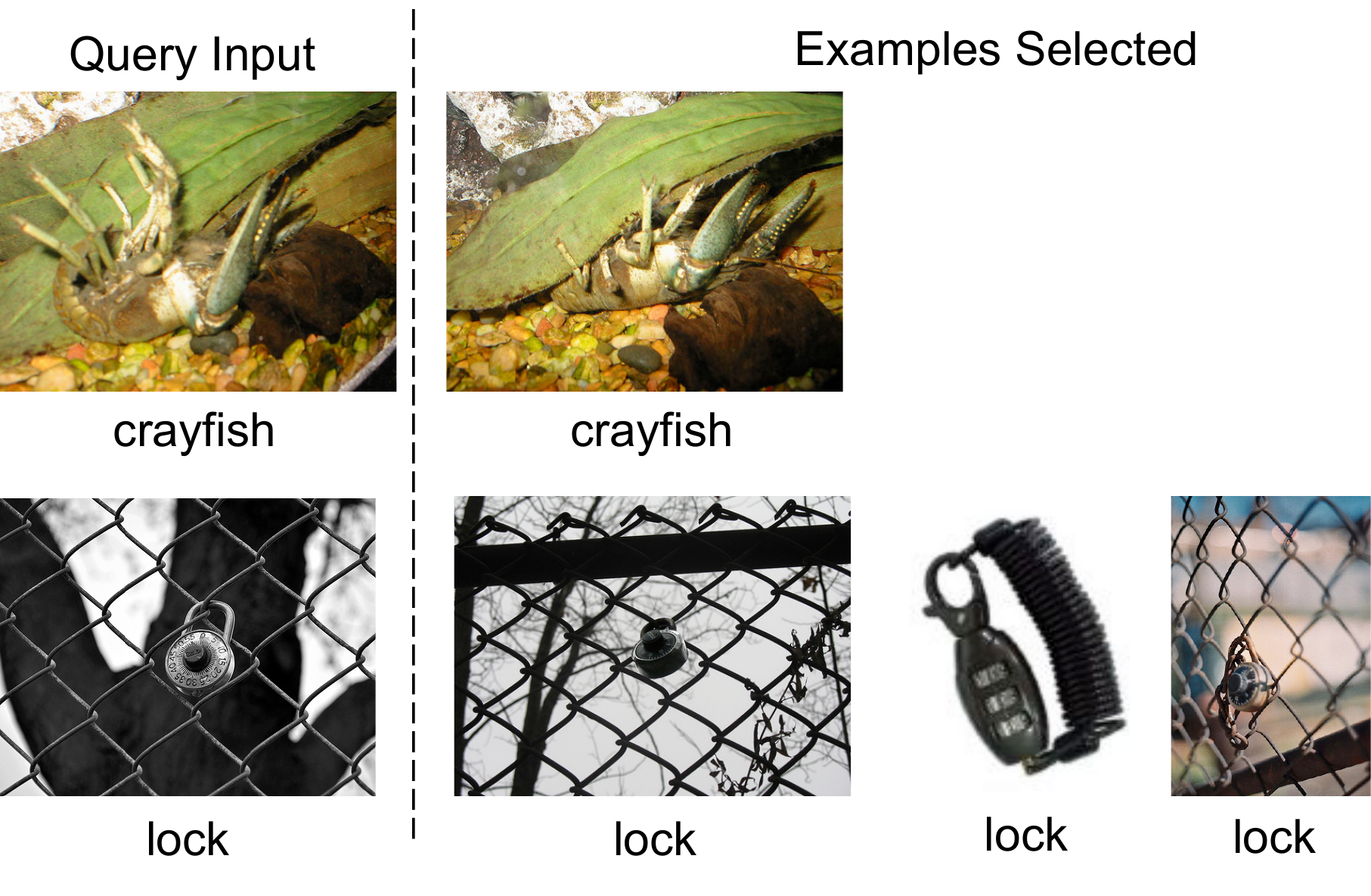}
    \caption{Queries for correct predictions.}
    \label{fig:grid-correct}
    \end{minipage}
\end{figure*}

\begin{figure*}
\centering
    \begin{minipage}[t]{0.36\textwidth}
    \centering
    \includegraphics[width=1\columnwidth]{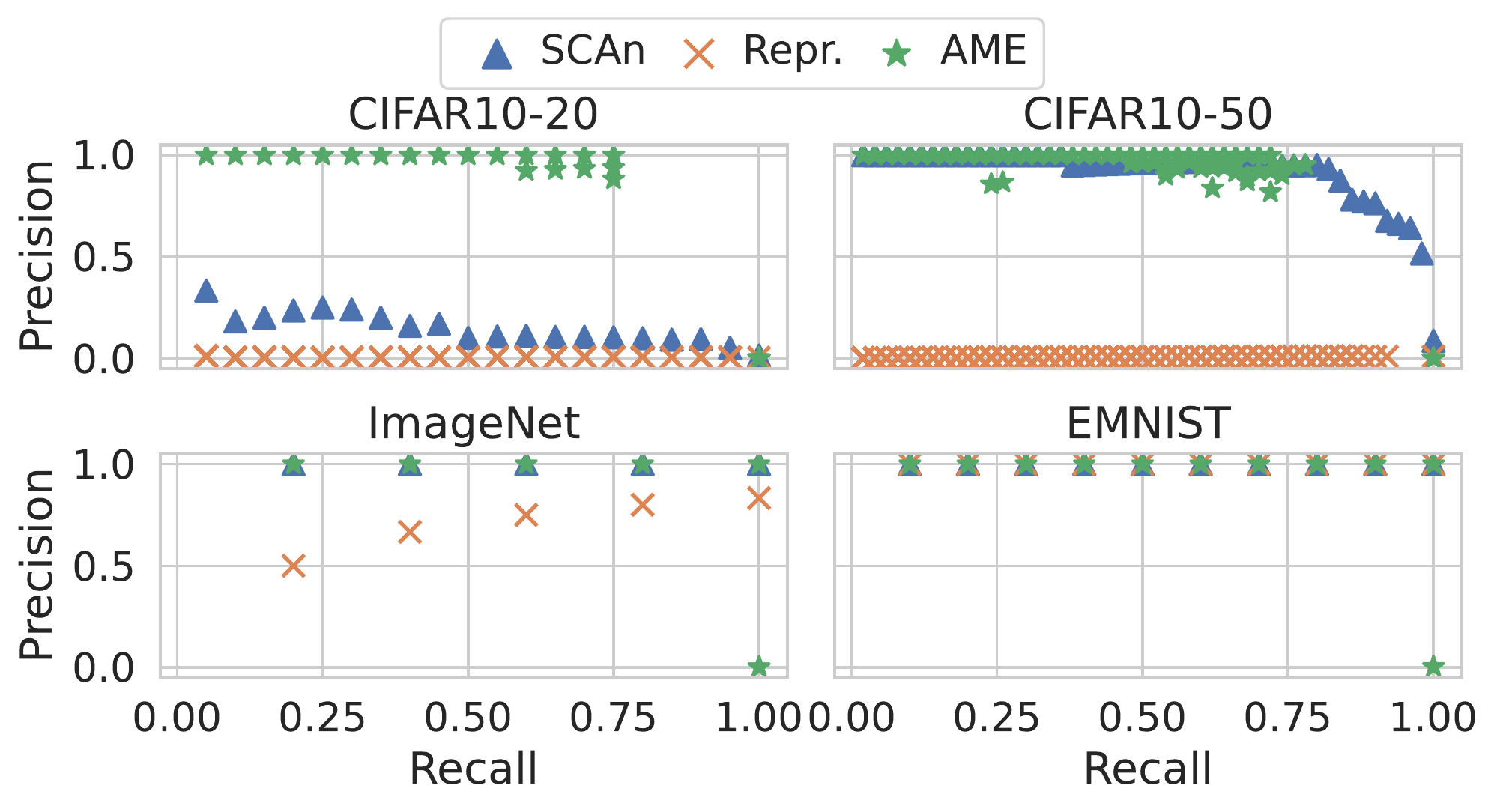}
    \caption{Precision v.s. recall curves for comparison with prior poison detection work. ``Repr.'' denotes Representer Points.}
    \label{fig:pr-curve}
    \end{minipage}
    \hfill
    \begin{minipage}[t]{0.3\textwidth}
    \centering
    \includegraphics[width=1\textwidth]{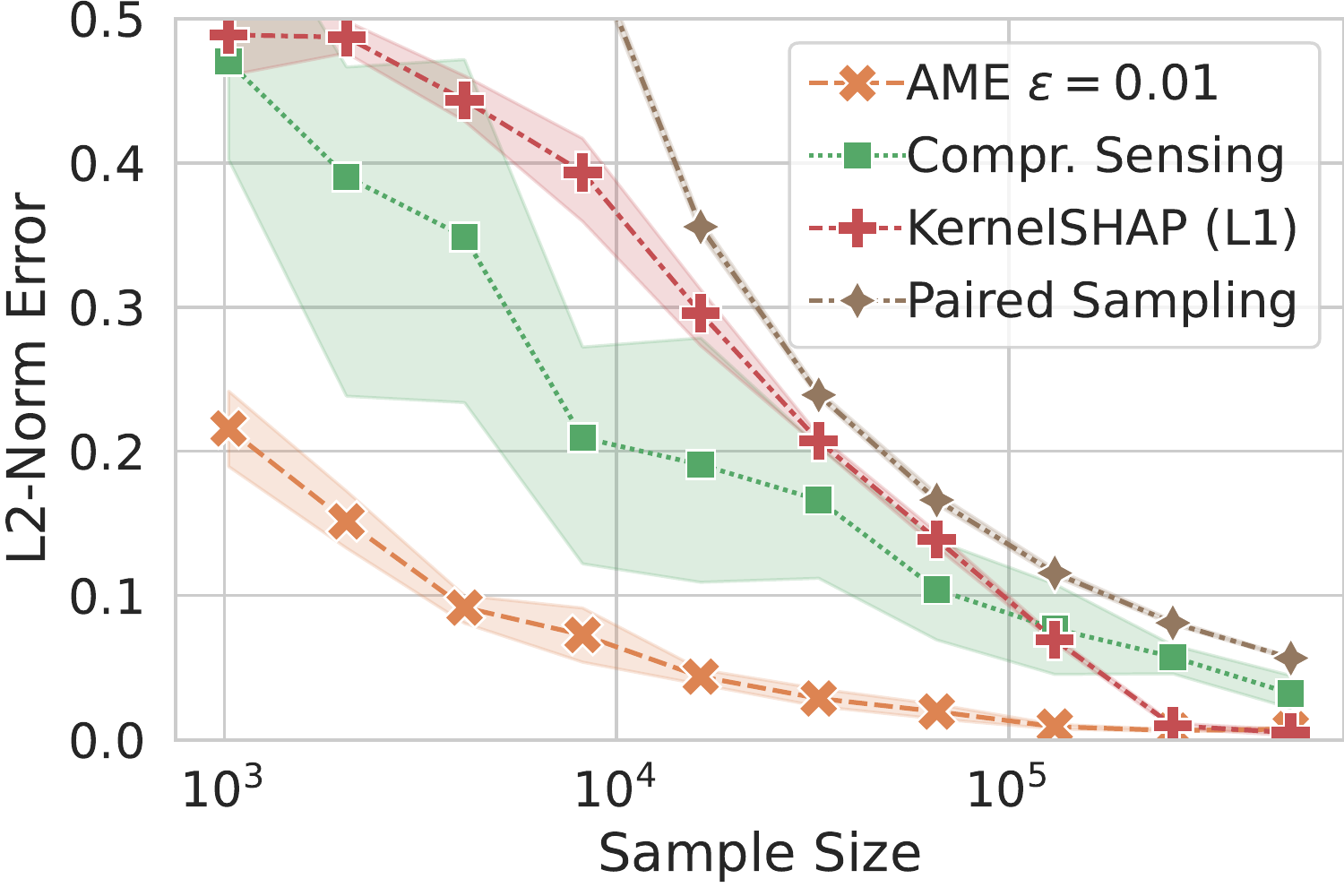}
       \caption{Est. error vs sample size on synthetic dataset. The shaded area shows 95\% confidence intervals.
       } 
    \label{fig:est-pfea-sim}
    \end{minipage}
    \hfill
    \begin{minipage}[t]{0.3\textwidth}
    \centering
    \includegraphics[width=1\textwidth]{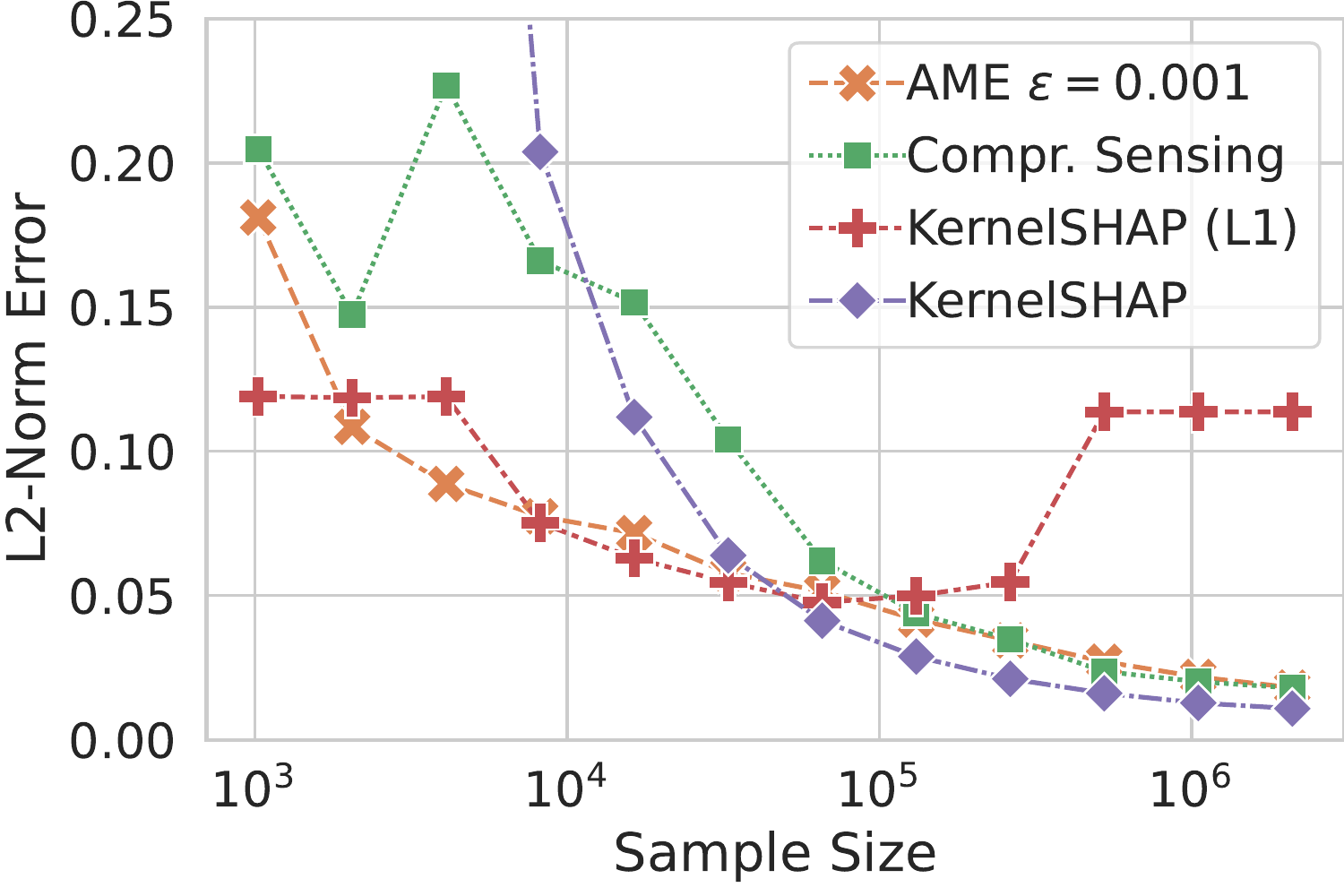}
    \caption{Est. error vs sample size on Poisoned MNIST dataset.
    }
    \label{fig:est-mnist}
    \end{minipage}
\end{figure*}

\subsection{Detecting Poisoned Training Data}
\label{subsec:poison-detect}

We study various models and attacks on image classification and sentiment analysis, summarized in Table~\ref{tab:dataset-summary}. Fig.~\ref{fig:poison-example} shows a few concrete attack examples.
Appendix \ref{appedix:eva-setting-details} provides more details, as well as the hyperparameters used. 

\noindent\textbf{Precision and recall.} Given a query for a mis-classified example at test time, we use our \quantity{} estimator to pinpoint those training data points that contribute to the (mis)prediction. A detection is correct if its corresponding training data point has been poisoned. Table~\ref{tab:main} shows the average precision and recall across multiple queries, for each scenario presented in Table~\ref{tab:dataset-summary}. 
To make the number of subset models $M$ comparable across tasks (and because we know $k$), we report $c$ and use $M=ck\log_2(N)$ subset models.
Precision is not counted when nothing is selected to avoid an upward bias.
Our largest experiments only run with warm-starting, for computational reasons.
Table~\ref{tab:main} shows that our method (LASSO+Knockoff) achieves very high precision and reasonable recall, and that warm-starting achieves good performance by enabling more utility evaluations.

Figure~\ref{fig:grow-c} shows the precision and recall for two image classification scenarios with and without knockoffs. We see that knockoffs are important for ensuring a consistently high precision (solid blue lines).  And the recall (dashed orange lines) grows as the number of subset models grows with $c$.
Appendix \ref{appendix:eval-details} shows the figures for all scenarios (Fig. \ref{fig:grow-c-full}), as well as other ablation and sensitivity studies for different parts of the methodology and parameters (\ref{sec:micro}). We also discuss the impact of those choices on running time (\ref{appendix:runtime}).


\noindent\textbf{Comparison with prior work.} 
We compare against two recent works: 
SCAn~\cite{tang2021demon}, a poison (outlier) detection technique that requires a set of clean data, and Representer Points~\cite{repr}, a more quantitative approach that measures an influence-like score for training data.
We delay the evaluation of other \SV algorithms to \S\ref{eval:sv}, as existing methods are not able to run on our large experiments, in which $M < N$.
We compare the precision of each method at different recall levels, by varying internal decision thresholds (see Appendix \ref{appendix:comparisons} for details on SCAn).
Fig.~\ref{fig:pr-curve} summarizes the results and shows that \quantity{} performs as well or better than both approaches. \quantity{} is particularly efficient when there are very few poisoned training data points, which existing approaches fail to handle (CIFAR10-20).
We also see a sharp decrease in precision when recall exceeds a certain level for AME, unlike SCAn. This is because we chose the LASSO regularization parameter as $\lambda_{1se}$ to favor sparsity in coefficients, in order to minimize false positives. Fig.~\ref{fig:pr-curve-lambdamin} in the Appendix shows the result of another common choice, $\lambda_{min}$, with less sparsity. The overall findings are similar, with a more graceful drop in precision-recall curve observed. In our approach, the knockoffs automatically control the false discovery rate to ensure that we remain in the high precision regime.

To show that the \quantity{} is able to work at a finer granularity than SCAn, we ran a mixed attack setup on CIFAR10, where we simultaneously use 3 different attacks: each attack has a different trigger and poisons 20 different images.
We found that SCAn clusters nearly all poisoned images together (along with many clean images), regardless of the attack used in the query $Q$;
while the \quantity{} selects the correct attack for each query, and achieves an average precision (recall) of 96.3\% (65.5\%), 97.4 (89.5\%) and 97.1\% (71.3\%), respectively for 20 random queries from each attack.

Appendix \ref{appendix:comparisons} provides more details and evaluation of SCAn, Representer Points, influence functions, and Shapley values. 

\noindent\textbf{Improvements under hierarchical design.} We group the $300$K users of the NLP dataset into $300$ time-based groups or $1$K users each. We poison two top-level sources, with $5$ and $10$ poisons, respectively (details in Appendix \ref{appendix:hierarchical}). Figure~\ref{fig:hierarchy} shows the \quantity{} precision and recall in finding the poisons, averaged over 20 different queries (poisoned test points).
We find that (a) top-level sources are detected with few observations; and (b) recall for second-level sources degrades gracefully as the number of observations (submodels) decreases. The hierarchical split is also efficient: it achieves $\sim$$100\%$ precision and $60\%$ recall with $1.5$K observations, before our non-hierarchical method detects anything.

\subsection{Data Attribution for Non-poisoned Predictions}
\label{subsec:other-apply}

We next measure what training data led to a specific prediction in the absence of poisoned data. Figs. \ref{fig:grid-wrong} and \ref{fig:grid-correct} show examples from a subset of ImageNet, for correct and incorrect predictions, respectively.
Qualitatively, explanation images share similar visual characteristics.
Quantitatively, Figure~\ref{fig:conf-drop} in Appendix \ref{appendix:eval:general} shows that removing the detected inputs significantly reduces the target prediction's score (compared to a random removal baseline), showing that we detect inputs with significant impact.
Additional results, and a comparison to a random baseline, can be found at \url{https://enola2022.github.io/}. Appendix~\ref{appendix:eval:general} details the setting, and shows results for CIFAR10.

\subsection{Shapley Value Estimation from \quantity{}}
\label{eval:sv}


Finally, we showcase our SV estimator from \quantity{} using simulated data and a subset of MNIST with poisoning, which are small enough to study the $M > N$ case where known estimators are applicable. In both setups, we know the ground truth or can approximate it closely enough, respectively (details in Appendix \ref{appendix:sv}). We compare the \AME to KernelSHAP~\cite{lundberg2017unified}, Paired Sampling~\cite{covert2021improving}, and two sparsity-aware methods: ``KernelSHAP (L1)''~\cite{lundberg2017unified} that uses LASSO heuristically to filter out variables before fitting a linear regression, and Compressive Sensing~\cite{jia_towards_2020}.
We use $p$-featurization (\S\ref{subsec:p-fea}) without knockoffs, and $p\sim Uni(\varepsilon, 1-\varepsilon)$.

The results in Figure \ref{fig:est-pfea-sim} (simulated data -- KernelSHAP without $L_1$ regularization mostly overlaps with Paired Sampling and thus is not shown) and Figure \ref{fig:est-mnist} (MNIST data -- Paired Sampling omitted due to prohibitive memory and computation costs) show that \AME delivers the fastest rate among these baselines on small sample sizes, and remains competitive when sample sizes become larger, though with a slightly larger final error than KernelSHAP on MNIST (likely due to approximate sparsity).
This larger error, however, is for a large sample size ($M>50N$) a regime unlikely to be practical for SV given the cost of utility evaluations (model training). 
Notably, on MNIST, ``KernelSHAP (L1)'' error is as low as the \AME when the sample size is small, while it diverges with more samples. 
This seems to be due to incorrect filtering in the heuristic LASSO step, which misses one of the $k$ poisoned datapoints (variables) on large sample sizes. AME does not have this instability as LASSO is the final estimate of the \SV.
Appendix~\ref{appendix:sv} shows comparisons to additional baselines, as well as ablation studies.

%% file: related-work-v2.tex
\section{Related Work}
\label{sec:rel-work}

We focus on the closest related work, and refer the reader to Appendix \ref{appendix:relwork} for a broader discussion.
Efficient Shapley value (SV)~\cite{Shapley1953} estimation is an active area, and is closest to our work.
Recent proposals also reduce SV estimation to regression, although differently than we do~\cite{lundberg2017unified,williamson2020efficient,covert2021improving,jethani2021fastshap}. Beta-Shapley~\cite{kwon2021beta} proposes a generalization of SV that coincides with the AME when $p\sim \mathrm{Beta}$ (we found the truncated uniform to be better in practice, but provide error bounds for both).
None of these works study the sparse setting, or provide efficient $L_2$ bounds in this setting. This may stem from their focus on SV for {\em features}, in smaller settings than we consider for training data, and in which sparsity may be less natural. 
The most comparable work to ours is that of Jia et.al~\cite{jia_towards_2020}, which provides multiple estimators, including under sparse, monotonic utility assumptions. Their approach uses compressive sensing (closely related to LASSO) and yields an $O(N\log\log N)$ rate. We significantly improve on this rate with an $O(k\log N)$ estimator, which is much more efficient in the sparse ($k \ll N$) regime.

Other principled model explanation approaches exist, based on influence functions~\cite{koh_understanding_2017,koh_accuracy_2019,feldman2020neural}, Representer Points~\cite{repr}, or loss tracking~\cite{pruthi_estimating_2020,hammoudehsimple}.
They either focus on marginal influence on the whole dataset~\cite{koh_understanding_2017,koh_accuracy_2019}; make strong assumptions (\eg convexity) that disallow their use with DNNs~\cite{koh_understanding_2017,koh_accuracy_2019,basu_influence_2020}; cannot reason about data sources or sets of training samples~\cite{pruthi_estimating_2020,hammoudehsimple, chakarov2016debugging}; or subsample training data but focus on a single inclusion probability, and thus cannot explain results in all scenarios~\cite{feldman2020neural}.

\ignore{
The closest related works, {\bf model explanation}, also aim to provide principled measures of training data impact on the performance of an ML model. These techniques include Shapley value, influence functions and Representer Point which we discussed in \S\ref{sec:intro}.
TracIn~\cite{pruthi_estimating_2020} and CosIn~\cite{hammoudehsimple} are two other recent proposals for measuring the influence of a training sample based on how it impacts the loss function during training. Existing model explanation approaches fall short. Many of them~\cite{pruthi_estimating_2020,hammoudehsimple, chakarov2016debugging} cannot reason about data sources or sets of training samples.  While some recent proposals, \eg~\cite{feldman2020neural} subsample training data and can thus be extended to handle sources, they use a single inclusion probability to generate subsamples and thus cannot explain results in all scenarios.
%

\noteanqi{add the related works of SV estimations}
The efficiently estimation of \textbf{Shapley value (SV)}, which was proposed in game theory~\cite{Shapley1953}, has been widely studied in recent years. Under the assumption of bounded utility functions, ~\cite{maleki2014bounding} proposed a permutation sampling based approach to achieve a desired approximation error in $l_2$ norm by requiring $O(N^2 \log(N))$ samples. ~\cite{jia_towards_2020} introduced two estimation algorithms by introducing various assumptions on the utility function. Under the sparsity assumption, their work reduced the samples required to $O(N \log\log(N))$ by describing a compressive permutation sampling method. Many SV estimations also have been applied to feature importance selection~\cite{jethani2021fastshap, lundberg2017unified, covert2021improving}. ~\cite{covert2021improving} introduced an unbiased version of KernelSHAP, which improved the linear regression-based approach to SV estimation. ~\cite{jethani2021fastshap} proposed a FastSHAP to amortize the cost by estimating SV in a single forward pass using a learned explainer model. However their goal is to quantify which features are the most influential for the model output, which is very different to ours. With aim to quantify the value of individual data points, our approach requires only $O(k\log(N))$ samples which is more efficient than the previously state-of-art works~\cite{maleki2014bounding, jia_towards_2020}.

\ignore{
Probability of Sufficiency (PS) also builds on causal inference theory, and measures the causal impact of switching data points' labels \mbox{\cite{chakarov2016debugging}}.
The proposed estimation techniques do not apply to complex ML models such as DNNs.
Probability of Sufficiency (PS)~\cite{chakarov2016debugging} measures the causal impact by switching data points' labels, but their approach does not apply to complex ML models such as DNNs.
In contrast, we propose a principled quantity that measures the marginal impact of a data point when added to different subsets of the data, and show that it captures individual contributions to group effects. 
}

Another related body of work is the work focused on defending against data poisoning attacks.
Reject On Negative Impact (RONI)~\cite{barreno_security_2010,baracaldo_mitigating_2017} describes an algorithm that measures the LOO effect of data points on subsets of the data. However, RONI is sample-inefficient and the paper does not prescribe a subset distribution to be used. As we explained previously, the choice of subset distribution can impact precissoion and recall.
Another line of work \cite{tran2018spectral,hayase2021spectre,chen2018detecting,tang2021demon,shen2016auror} uses outlier detection to identify poisoned data. These are not query aware and thus can select benign outliers. Other approaches~\cite{doan2020februus,tang2021demon,veldanda2020nnoculation} assume the availability of clean data, or make strong assumptions about the model, e.g., assuming that linear models~\cite{Jagielski2018ManipulatingML}, or the attack, e.g., assuming that the attack is a source-agnostic trigger attacks~\cite{gao2019strip}, a trigger attack with small norm or size~\cite{chou2018sentinet,Wang2019NeuralCI,udeshi2019model}, or a clean-label attack~\cite{peri2020deep}. None of these approaches can generalize across techniques.

Our work is also related to the existing literature on \textbf{data cleaning and management}. However, data cleaning approaches are not query driven, and must rely on other assumptions. As a result, many approaches depend on user provided integrity constraints \cite{chu2013holistic,chu2013discovering} or outlier detection~\cite{maletic2000data, hellerstein2008quantitative}. As a result these approaches cannot always identify poisoned data~\cite{koh2018stronger} and might also identify benign outliers. Recent approaches~\cite{krishnan2017boostclean,dolatshah2018cleaning}, have also used another downstream DNN for data cleaning. However, these approaches assume that corrupt data has an influence on test set performance, an assumption that may not hold in scenarios such as data poisoning. Finally, Rain~\cite{wu2020complaint} is a recent query-driven proposal that proposes using influence function to explain SQL query results. While Rain shares similar goals, we focus on DNNs, a different use case.


Finally, our work leverages, and builds on, a large body of existing work from different fields.
\ignore{
First is the {\bf causal inference} literature.
If we regard the inclusion of a data source as a treatment, our methodology is related to factorial experiments, from which our \quantity{} is inspired. 
For instance, ~\cite{kang2007demystifying, imbens_causal_2015} focuses on single treatment \quantity{} (i.e. only one source) and observational study rather than randomized experiment. 
Multiple treatments are introduced in \cite{egami_causal_2019,dasgupta_causal_2015,hainmueller_causal_2014}, though their quantity uses different population distribution than ours, and different estimation techniques.
In the computer science filed, Sunlight \cite{10.1145/2810103.2813614} uses a similar approach to study ad targeting. However, this paper uses only one sampling probability, and a holdout set for statistical confidence instead of our knockoff procedure.
}
First, our work is related to the {\bf causal inference} literature if we regard the inclusion of a data source as a treatment, from which \quantity{} is inspired.
For instance, \cite{kang2007demystifying, imbens_causal_2015} focuses on single treatment \quantity{} (i.e. only one source). 
Multiple treatments are introduced in \cite{egami_causal_2019,dasgupta_causal_2015,hainmueller_causal_2014}, though their quantity uses different population distribution than ours, and different estimation techniques.
In the computer science filed, Sunlight \cite{10.1145/2810103.2813614} uses a similar approach but with only one sampling probability and without knockoff procedure.
Second, {\bf sparse recovery} studies efficient algorithms to recover a sparse signal from high dimensional observations. We leverage LASSO \cite{lecue_regularization_2017} --with properties related to those of compressive sensing \cite{candes2006compressive}-- and knockoffs \cite{candes_panning_2017}.
Other approaches to the important factor selection problem exist, such as the analysis of variance (ANOVA) \cite{bondell_simultaneous_2009, egami_causal_2019, post_factor_2013} used in \cite{egami_causal_2019}, but we think LASSO is better suited to our use-case due to its scalability guarantees.
Third, our goal is related to {\bf group testing} \cite{noauthor_group_2021,group-survey} as discussed in \S\ref{sec:intro},
and studying if and how group testing ideas could improve our technique is an interesting avenue for future work.

\notejinkun{Integrate this:} Beta-Shapley~\cite{kwon2021beta} is a concurrent work where the authors generalize data Shapley to consider different weightings based on subset size, which coincides with our AME under $p$-induced distribution (\eg, Beta-Shapley is AME when $p\sim$ beta distribution). They study the statistical and game theory properties and apply them to guide the choice of weighting functions, while we are interested in the sparse regime and focus more on the practical aspect (\eg efficient estimation and false detection control). We also propose a sparse SV estimator which they do not. D-Shapley~\cite{ghorbani2020distributional} generalizes SV by modeling the dataset as a random variable.

}

%% file: appendix.tex
\newpage
\appendix
\onecolumn
\section{Convergence Rate when Using LASSO to compute \quantity} 

We first state and prove a variant of the LASSO error bound result from~\cite{lecue_regularization_2017} in a simpler setting, which is sufficient for our application and will serve as the foundation of many of our results. We then apply this result to finish the proof of $O(k \log N)$ sample rate of our \quantity{} estimator left in Prop. \ref{prop:lasso-rate} in the main body.

\subsection{Simplified LASSO Error Bound}
\begin{proposition}
\label{prop:simple-lasso-error-bound}
Consider a regression problem where one wants to approximate an unknown random variable $Y$ using a set of random variables $X_1, \dots, X_N$ with $N\geq 3$. Let $\hat{\beta}_{lasso}$ be the LASSO coefficient when regressing $Y$ on $X$ with $L_1$ regularization, and $\beta^*$ be the best linear fit (detailed definition in Prop. \ref{prop:regression-reduction}). Assume that $\forall n\in[N]: \beta^*_n$ is finite, $X_n$ is bounded in $[A, B]$, $Y$ is bounded in $[0,1]$. Further assume that $\forall i,j\in[N]:~ \mathbb E[X_iX_j]=0$ when $i=j$ and otherwise 1. If $\beta^*$ has at most $k$ non-zeros and the sample size $M\geq k(1+\log(k/N))$, then there exists a regularization parameter $\lambda$ and a value $C(B-A,\delta)$ such that with probability at least $1-\delta$:
\begin{align*}
\|\hat{\beta}_{lasso}-\beta^*\|_2 \leq C(B-A,\delta) \sqrt{\frac{k\log(N)}{M}}.
\end{align*}
\end{proposition}


The following terminology will be useful to facilitate its proof.

\begin{definition}
\label{def:subg}
$X$ is said to be a subgaussian random variable with variance property $\sigma^2$ if for any $s \in \mathbb{R}, \mathbb{E}[\exp (s X)] \leq \exp \left(\frac{\sigma^{2} s^{2}}{2}\right)$. We write $X\sim subG(\sigma^2)$. 
\end{definition}

\begin{definition}
\label{def:lsubg}
The underlying measure of a random vector $X\in \mathbb R^N$ is said to be $L$-subgaussian if for every $q\geq 2$ and every $u\in \mathbb R^N$, 

$$\mathbb E^{1/q}[|\langle u, X\rangle|^q]\leq L\sqrt q \mathbb E^{1/2}[|\langle u, X \rangle|^2].$$
\end{definition}



We now prove proposition \ref{prop:simple-lasso-error-bound}:
\begin{proof}
To apply Theorem 1.4 from \cite{lecue_regularization_2017} to bound the error, we need our setup to satisfy Assumption 1.1 of that paper: that $X$ is an isotropic, $L$-subgaussian measure, and that the noise is in $L_q$ for $q>2$.

First, the isotropic requirement is that $\mathbb E[\langle u,X\rangle ^2]=\langle u,u\rangle$ for all $u\in \mathbb R^N$. This can be shown by observing that
$\mathbb E[\langle u, X \rangle ^2]=\sum_{i=1}^N\mathbb E[u_i^2X_i^2]+\sum_{i\neq j}\mathbb E[u_iu_jX_iX_j]=\sum_{i=1}^N u_i^2$, where the last equality comes from the fact that $\mathbb E[X_iX_j]=0$ and $\mathbb E[X_i^2]=1$.

The second requirement is that the probability measure of the covariate vector $X$ is $L$-subgaussian (Def. \ref{def:lsubg}). Let $\sigma=(B-A)/2$. Then $X_n \sim subG(\sigma^2)$ since they are bounded. 
Hence, $\langle u, X\rangle \sim subG(\sigma^2)$ (see \eg Theorem 2.1.2 in \cite{pauwels2020lecture}).
Applying Proposition 3.2 from \cite{rivasplata_subgaussian_nodate}, we have $\mathbb E^{1/q}[|\langle u,X\rangle|^q]\leq L\sqrt{q}$ with $L>0$ a constant dependent only on $\sigma$. Noticing that Def. \ref{def:lsubg} remains equivalent when constraining $\|u\|_2=1$, and that $\mathbb E^{1/2}[ \langle u, X\rangle^2]=\|u\|_2=1$ concludes this part of proof.

Third the ``noise'', defined as $\xi=\langle\beta^*,X\rangle-Y$ should be in $L_q$, $q>2$. 
Notice that the best estimator cannot be worse than a zero estimator that always return 0, thus $\mathbb E[\xi^2]\leq 1$. 
Hence, $|\xi|\leq|\langle\beta^*,X\rangle|+|Y|\leq DN+1$, 
where $D$ is the upper bound of $|\beta^*_nX_n|$, 
then there exists $q=1/\log(DN+1)+2>2$
such that $\mathbb E[|\xi|^q]\leq \mathbb E[|\xi|^2] (DN+1)^{q-2}\leq \mathbb E[|\xi|^2]\cdot e\leq e$.

Combined with the sparsity assumption on $\beta^*$, we can apply Theorem 1.4 from \cite{lecue_regularization_2017}, which directly yields the error bound.

\end{proof}

\subsection{Proof of Prop. \ref{prop:lasso-rate}}
\label{appendix:proof-lasso-rate}
With this simplified bound, we can prove Prop. \ref{prop:lasso-rate}:
\begin{proof}
\label{proof:lasso-rate}
Recall that the best linear estimator $\beta^*_n=\AME_n/\sqrt{v}$, so applying Prop. \ref{prop:simple-lasso-error-bound} directly yields the bound. The remaining work is to verify the assumptions: Since the $\AME$ is an average of utility differences, we know that $\beta^*_n=\AME_n/\sqrt v \in [-1/\sqrt v,1/\sqrt v]$ is finite. Furthermore, by design $\mathbb E[X_iX_j]=\mathbb E_p[\mathbb E[X_iX_j|p]]=\mathbb E_p[\mathbb E[X_i|p] \mathbb E[X_j|p]]=0$ when $i\neq j$, and $\mathbb E[X_iX_j]=\mathbb E[X_i^2]=Var[\mathbb E[X_i^2]]+\mathbb E[X_i]^2=1$ when $i=j$.
With the assumptions all verified, applying Prop. \ref{prop:simple-lasso-error-bound} concludes proof.
\end{proof}

\section{An AME Estimator with $p$-featurization}
\label{appendix:sec-p-fea}

\newcommand{\wtP}{\mathcal{W}}
Recall that to estimate $\AME_n(\cP)$ defined under some given distribution $\cP$, we have been using a featurization where $\mathrm X[m,n]$ takes values of either $\frac{1}{\sqrt vp}$ or $\frac{-1}{\sqrt v(1-p)}$, depending on whether data point/source $n$ is respectively included or excluded in the subset for row $m$, with $v=\mathbb E_{p \sim \cP}[\frac{1}{p(1-p)}]$.
However, values for $\mathrm X[m,n]$ blow up quickly as $p$ approaches 0 or 1, which leads to unbounded feature values for certain $\cP$ that samples such $p$ values often. 
Unfortunately such distributions can be useful in some cases, in particular to derive a low-error \SV approximations (\eg with the beta$(1+\varepsilon, 1+\varepsilon)$ with small $\varepsilon$ (\S\ref{sec:beta})). Below we propose a different featurization that solves this issue while still ensuring an $O(k \log N)$ sample rate.

Specifically, we define $X_n=\sqrt v (1-p)$ if $n\in S$ and $X_n=-\sqrt v p$ otherwise. These values are clearly bounded when $\sqrt v$ is finite. To ensure that the best linear fit still recovers \quantity{} (a.k.a. Prop. \ref{prop:regression-reduction})---an important property we use to derive the $O(k \log N)$ error bound---we adapt the distribution where samples $(X,Y)$ are drawn in LASSO.
Recall that source inclusion is a compound distribution, in which we first draw $p \sim \cP$ for the entire subset, and then $X_n$'s according to $p$. Here, we change they way $p$ is drawn, by imposing a $\frac{1}{p(1-p)}$-weighting over $\cP$.
Formally, if we note $f_\cP(\cdot)$ the probability density function (PDF) of $\mathcal{P}$ which we used with the original $1/p$ featurization, we now draw $p$ from the distribution with reweighted PDF  $f_{\wtP}(p)=f_\cP(p)\frac{1}{p(1-p)} / \mathbb E_{p\sim \cP}[\frac{1}{p(1-p)}]$.

Denote this new sampling scheme by $p\sim \wtP$, and note that the \AME is still defined under the original distribution $\cP$, while $\wtP$ is only used to sample $(X,Y)$ for LASSO.
We show that the best linear fit on $(X,Y)$ is still $\AME_n(\cP)$:


\begin{proposition}
\label{prop:regression-reduction-p-fea}
Let $\beta^*$ be the best linear fit on $(X, Y)$:
\begin{equation}
\beta^{*}=\underset{\beta \in \mathbb{R}^{n}}{\argmin}~\mathbb{E}\left[(Y-\left\langle \beta, X\right\rangle)^{2}\right],
\end{equation}
then $\AME_n(\cP)/\sqrt v=\beta^*_n, \ \forall n\in[N]$, with $v=\mathbb E_{p \sim \cP} [\frac{1}{p(1-p)}]$.
\end{proposition}
\begin{proof}
For a linear regression, we have (e.g. Eq. 3.1.3 of \cite{angrist2008mostly}):
\begin{equation*}
\beta^*_n=\frac{Cov(Y,\tilde X_n)}{Var[\tilde X_n]} ,
\end{equation*}
where $\tilde X_n$ is the regression residual of $X_n$ on all other covariates $X_{-n}=(X_1,\dots,X_{n-1}, X_{n+1},\dots, X_{N})$.
In our design $\mathbb E[X_n|X_{-n}]=\mathbb E_{p\sim \cP}[X_n | p]=0$,
implying $\tilde X_n=X_n-\mathbb E[X_n|X_{-n}]=X_n$. Therefore:
\begin{equation*}
    \beta^*_n=\frac{Cov(Y,\tilde X_n)}{Var[\tilde X_n]}=\frac{Cov(Y,X_n)}{Var[X_n]}=\frac{\mathbb E[X_nY]}{Var[X_n]} .
\end{equation*}
Notice that $\mathbb E[X_nY] = \mathbb E_{p\sim \wtP}[\mathbb E[X_nY|p]]$ with:
\begin{equation*}
\begin{aligned}
&\mathbb E[X_nY|p] = p\cdot\mathbb E[X_nY|p, n\in S]+(1-p)\cdot\mathbb E[X_nY|p, n\notin S]\\
& = {\sqrt vp(1-p)} \mathbb E[Y|p,n\in S]+{\sqrt v (1-p)(-p)}\mathbb E[Y|p,n\notin S]\\
& = {\sqrt v p(1-p)}(\mathbb E[Y|p,n\in S]-\mathbb E[Y|p,n\notin S])
\end{aligned}
\end{equation*}
Combining the two previous steps yields:
\begin{equation*}
\begin{aligned}
\beta^*_n &= \frac{\mathbb E_{p \sim \wtP}[\sqrt vp(1-p)(\mathbb E[Y|p,n\in S]-\mathbb E[Y|p,n\notin S])]}{Var[X_n]} \\
 &=\frac{\sqrt v\mathbb E_{p \sim \cP}[\mathbb E[Y|p,n\in S]-\mathbb E[Y|p,n\notin S]]}{v\cdot Var[X_n]} \\
&=\frac{AME_n}{\sqrt v\cdot Var[X_n]} .
\end{aligned}
\end{equation*}
Noticing that 

\begin{equation*}
    \begin{aligned}
& Var[X_n]=\mathbb E[Var[X_n|p]]+Var[\mathbb E[X_n|p]]= {\int_0^1 f_{\wtP}(p)(pv(1-p)^2+(1-p)v(-p)^2) dp}\\
& {=v\int_\varepsilon^{1-\varepsilon} \frac{1}{v(1-2\varepsilon)} dp=1}
\end{aligned}
\end{equation*}
concludes the proof.
\end{proof}

Moreover, $\mathbb E[X_iX_j]=\mathbf{1}\{i=j\}$ for the same reason as in \S\ref{appendix:proof-lasso-rate}. Hence, we can still apply Prop. \ref{prop:simple-lasso-error-bound} to derive the same LASSO error bound:
\begin{proposition}
\label{prop:lasso-rate-p-fea}
If $X_n$'s are bounded in $[A,B]$ and $N\geq 3$, there exist a regularization parameter $\lambda$ and a constant $C(B-A,\delta)$ such that when the sample size $M\geq k(1 + log(k/N))$, with probability at least $1-\delta$:
\begin{align*}
\|\hat{\beta}_{lasso}-\frac{1}{\sqrt v}\AME\|_2 \leq C(B-A,\delta) \sqrt{\frac{k\log(N)}{M}},
\end{align*}
with $v=\mathbb{E}_{p\sim \cP}[\frac{1}{p(1-p)}]$ and $\AME_n$ defined under $\cP$.
\end{proposition}

Compared with the original, $1/p$-featurization, the difference only lies in the constant factor $C(B-A, \delta)$, since $B-A$ has changed.

\section{Sparse Estimators for the Shapley Value from the \quantity{}}
\label{appendix-sec:sv-estimators}
Recall that the \AME is the \SV when $\cP = Uni(0,1)$. However, this choice is incompatible with our fast convergence rates for LASSO.
To find a good estimator of the \SV from the \quantity{}, it is thus crucial to understand the discrepancy between the \quantity{} and the \SV introduced by different distributions $\cP$ over $p$, in order to bound the \SV from the \AME with a $cP$ compatible with good convergence rates. Here we first derive general error bounds between \SV and \quantity{} that work for all distributions. Then we apply it to two specific distributions: namely the truncated uniform and Beta distributions. We mainly focus on sparse \SV and/or \AME under bounded or monotone utility, but to make the discussion clearer, no assumptions are made on either the \SV or \AME unless explicitly stated. 

Throughout this section, we denote by $P_{\AME}(S)$ and $P_{\SV}(S)$ the probability of sampling subset $S$ when $p\sim \cP$ and $p\sim Uni(0,1)$, respectively (the $P_{\SV}$ name is due to the fact that \AME is \SV under this distribution, see Prop. \ref{prop:ame-to-sv}). We also introduce the following notation: \begin{equation}
\label{eq:Delta_def}
    \Delta \triangleq \max_S \frac{P_{\AME}(S)}{P_{\SV}(S)} - 1 .
\end{equation}

\subsection{General bounds}
\begin{lemma}
\label{lemma:general-sv-ame-bound}
Assume a bounded utility function with range in $[0,1]$. Then $$\|\AME-\SV\|_{\infty}\leq 2\Delta.$$ 
Further assume a monotone utility (Assumption \ref{asm:mono}). Then: $$\|AME-\SV\|_2 \leq \Delta+\sqrt{2\Delta}.$$
\end{lemma}

\begin{proof}
The $L_\infty$ error bound is due to the following:
\begin{equation}
\begin{aligned}
\label{eq:AME-SV-Linfty-bound}
&|\AME_n-\SV_n|=\Big|\sum_{S} \big(P_{\AME}(S)-P_{\SV}(S)\big)\big(U(S+\{n\})-U(S)\big)\Big|\\
&\leq \sum_{S} |P_{\AME}(S)-P_{\SV}(S)|=2\sum_{S: P_{\AME}(S)>P_{\SV}(S)} P_{\AME}(S)-P_{\SV}(S)\leq 2\Delta.
\end{aligned}
\end{equation}

\def\cN{\mathcal{N}}

For the $L_2$ error, its square $\|\AME-\SV\|_2^2$ can be divided into two groups based on the sign of $\AME_n-\SV_n$. Call those indices $n$ with positive (negative) sign $\cN^+$ ($\cN^-$). For all $n\in \cN^+$,
\begin{equation}
\label{eq:ame-sv-multi-bound}
\AME_n-\SV_n=\sum_S\big(P_{\AME}(S)-P_{\SV}(S)\big)\big(U(S+\{n\})-U(S)\big) \leq \sum_S\Delta P_{\SV}(S) (U(S+\{n\})-U(S)) = \Delta \SV_n,
\end{equation}
where the last inequality is due to $U(S+\{n\})>U(S)$ implied by monotonicity. For the same reason, all $\SV_n$'s are positive, implying $\SV_n\leq U([N])\leq 1$. Thus $(\AME_n-\SV_n)^2 \leq \Delta^2 \SV_n$.
On the other hand, for all $n\in \cN^-$, we know that $|\AME_n-\SV_n|$ is bounded by $2\Delta$; it is also bounded by $\SV_n$ since $\AME_n$ cannot be negative under monotone utility.
Summing up these bounds gives $\|\AME-\SV\|_2\leq \sqrt{\sum_{n\in \cN^+}\Delta^2 \SV_n + \sum_{n\in \cN^-} 2\Delta \SV_n}\leq \sqrt{\sum_{n\in \cN^+}\Delta^2 \SV_n} + \sqrt{\sum_{n\in \cN^-} 2\Delta \SV_n} \leq \Delta+\sqrt{2\Delta}$, where the last inequality is due to $\|\SV\|_1\leq 1$. 
\end{proof}

In fact the bounded utility assumption is quite minor when we assume monotonicity: every $U(S)$ is bounded between the empty set and full set utility. Given that by definition of SV $U(\emptyset)=0$, it reduces to an assumption of $U([N]) \leq 1$. In practice as long as $U([N])$ is a known and bounded constant (\eg the accuracy on the validation set of the model trained on the full set), one can simply scale the utility to meet this requirement. 
In what follows, when we say monotone utility we mean both monotone and bounded.

\subsection{SV Estimator from the AME under a Truncated Uniform distribution}
We prove the results from the paper's main body, before discussing $p$-featurization.

\mypara{Proof of Lemma \ref{lemma:TU-AME-SV-bound}}
\label{appendix:ame-to-sv-error}
As a reminder, Lemma \ref{lemma:TU-AME-SV-bound} states an $L_2$ error bound between the AME and SV under monotone utility, when AME is defined with $\mathcal{P}=Uni(\varepsilon, 1-\varepsilon)$.
\begin{proof}
Lemma \ref{lemma:general-sv-ame-bound} already gives us $\|\AME-\SV\|_2 \leq \Delta+\sqrt{2\Delta}$ ($\Delta$ from Eq. \ref{eq:Delta_def}). All that remains is to show that $\Delta \leq 4\varepsilon$:
\begin{subequations}
\begin{align}
\label{eq:mul-bound-dist-connection}
\Delta &= \max_S \frac{P_{\AME}(S)}{P_{\SV}(S)}-1\\
(j \leftarrow |S|) &=\max_j \left(N\int_\varepsilon^{1-\varepsilon}\frac{1}{1-2\varepsilon}{N-1 \choose j}p^j(1-p)^{N-j-1}dp-1\right)\\
&\leq \max_j \left(N\cdot\frac{1}{1-2\varepsilon}\int_0^{1}{N-1 \choose j}p^j(1-p)^{N-j-1}dp-1\right)\\
&= \max_j \left(N\cdot\frac{1}{1-2\varepsilon}\cdot\frac{1}{N}-1\right) =\frac{1}{1-2\varepsilon}-1 \label{ieq:tmp}\\
&\leq 4\varepsilon,
\end{align}
\end{subequations}
where the equality \ref{ieq:tmp} is due to the fact that the Binomial(p) with $p\sim Uni(0,1)$ is a discrete uniform. Hence, $\|\AME-\SV\|_2 \leq \Delta+\sqrt{2\Delta} \leq 4\varepsilon + 2\sqrt{2\varepsilon}$.
\end{proof}

Lemma \ref{lemma:general-sv-ame-bound} also yields to the following result, which fills in the missing piece in the proof of Corollary \ref{corollary:efficient-sv}---the one that states the $O(k\log N)$ bound for this \SV estimator:
\begin{corollary}
\label{appendix-corollary:sparse-AME-SV}
A $k$-sparse \SV implies a $k$-sparse \AME under monotone utility.
\end{corollary}
\begin{proof}
Under montone utility, both $\AME_n$ and $\SV_n$ are non-negative. By Eq. \ref{eq:ame-sv-multi-bound}, when $\SV_n=0$, $\AME_n \leq \Delta \SV_n=0$.
\end{proof}

\mypara{Non-monotone utility} 
\label{appendix:non-monotone}
When the utility is no longer monotone, we can still derive an $O(k \log N)$ rate in terms of $L_\infty$ error for SV estimation, under an additional assumption that the \quantity{} is $k$-sparse.
Indeed, first notice that the bound $\Delta \leq 4\varepsilon$ from Eq. \ref{eq:mul-bound-dist-connection} does not require monotonicity. Under utility bounded in $[0,1]$, applying the first part of Lemma \ref{lemma:general-sv-ame-bound} yields:
\begin{equation}
\label{eq:TU-AME-SV-bound-linfty}
\|\AME - \SV\|_{\infty} \leq 2\Delta \leq 8\varepsilon.
\end{equation}
Notice that this bound, as the $L_2$ bound, does not depend on $N$ or $k$. Hence, applying the same arguments as in Corollary \ref{corollary:efficient-sv} yields an $L_\infty$ error bound we presented in the main body (Corollary \ref{corollary:efficient-sv-linfty-main}). We reiterate it here for convenience of reading:

\begin{corollary}
\label{corollary:efficient-sv-linfty}
When $\AME$ is $k$-sparse and the utility is bounded in $[0,1]$, for every constant $\varepsilon>0, \delta>0, N\geq 3$, there exists constants $C_1(\varepsilon, \delta)$, $\varepsilon'$, and a LASSO regularization parameter $\lambda$, such that when the number of samples $M\geq C_1(\varepsilon, \delta)k\log N$, $\|\sqrt v\hat\beta_{lasso}-\SV\|_{\infty} \leq \varepsilon$ holds with a probability at least $1-\delta$, where $v=\mathbb E_{p\sim Uni(\varepsilon', 1-\varepsilon')}[\frac{1}{p(1-p)}]$.
\end{corollary}
\begin{proof}
$\|\sqrt v\hat\beta_{lasso}-\SV\|_{\infty} \leq \|\sqrt v\hat\beta_{lasso}-\AME\|_{2} + \|\AME-\SV\|_{\infty}$.
By the sparsity of \quantity{}, we apply Proposition \ref{prop:lasso-rate} to bound the first term by $\varepsilon/2$, and Eq. \ref{eq:TU-AME-SV-bound-linfty} with $\varepsilon' = \varepsilon / 4$ to bound the second term by the same, concluding the proof. 
\end{proof}

The main obstacle to deriving an $L_2$ bound in this more general setting comes from the lack of an $L_2$ error bound between the AME and the SV that is independent of $N$. Note that one may derive an $L_2$ error bound from Eq. \ref{eq:TU-AME-SV-bound-linfty} as follows: $\|\AME - \SV\|_2\leq \sqrt{N} \|\AME - \SV\|_\infty \leq 4\sqrt N \varepsilon$. However, this bound is now dependent on $N$, which violates the precondition of applying the LASSO error bound (see Prop. \ref{prop:lasso-rate}).


\mypara{Using $p$-featurization}
\label{subsec:p-fea}
As pointed out in \S\ref{appendix:sec-p-fea}, $p$-featurization also achieves a $O(k \log N)$ sample rate to reach a low $L_2$ error in estimating the \quantity{}. Since $p$-featurization has no effect on the \quantity{} value, the bound between the \quantity{} and the SV remains the same. We thus reach the same conclusion as for $1/p$-featurization:

\begin{corollary}
For every constant $\varepsilon>0, \delta>0, N\geq3$, there exists constants $C_2(\varepsilon, \delta)$, $\varepsilon'$, and a LASSO regularization parameter $\lambda$, such that with $M\geq C_2(\varepsilon, \delta)k\log N$, and with probability at least $1-\delta$:
\begin{enumerate}[label=(\arabic*)]
    \item $\|\sqrt v\hat\beta_{lasso}-\SV\|_\infty \leq \varepsilon$ holds when the utility is bounded in $[0,1]$ and the \quantity{} is $k$-sparse; 
    \item $\|\sqrt v\hat\beta_{lasso}-\SV\|_2 \leq \varepsilon$ holds when the utility is monotone and the SV is $k$-sparse,
\end{enumerate}
where the \quantity{} is defined under $\cP=Uni(\varepsilon',1-\varepsilon')$ and $v=\mathbb E_{p\sim \cP}[\frac{1}{p(1-p)}]$.
\end{corollary}
\begin{proof}
Conclusion (1) follows the same proof as Corollary \ref{corollary:efficient-sv-linfty}, and Conclusion (2) follows Corollary \ref{corollary:efficient-sv}.
\end{proof}

\subsection{SV Estimator from the AME under a Beta Distribution}
\label{sec:beta}

Another candidate to estimate the SV from the \quantity{} is to use $\cP = \textrm{Beta}(1+\varepsilon,1+\varepsilon)$ as the distribution of $p$, with $\varepsilon\in(0,0.5)$, and using $p$-featurization\footnote{The $1/p$-featurization in incompatible here, since this distribution can draw ``$p$''s arbitrarily close to 0 or 1 which leads to unbounded feature values, violating the assumption of the $O(k \log N)$ LASSO rate (Prop. \ref{prop:lasso-rate}).}.
We show an $O(k \log N)$ sample rate in this setting as well, after introducing two necessary lemmas.


\begin{lemma}
\label{lemma:gauineq}
For any $x\geq 1, y\geq 1$ and $\varepsilon \in (0,0.5)$, the following holds:

$$\frac{(x+\varepsilon-1)^\varepsilon(y+\varepsilon-1)^\varepsilon}{(x+y+2\varepsilon)^{2\varepsilon}} < \frac{\mathrm B(x+\varepsilon,y+\varepsilon)}{\mathrm B(x,y)} < \frac{(x+\varepsilon)^\varepsilon(y+\varepsilon)^\varepsilon}{(x+y+2\varepsilon-1)^{2\varepsilon}},$$
where $\mathrm{B}(\cdot,\cdot)$ is the Beta function.
\end{lemma}
\begin{proof}
According to Gautschi's inequality, for all $a>0, s\in(0,1)$,
\begin{equation*}
    {a^{1-s}<{\frac {\Gamma (a+1)}{\Gamma (a+s)}}<(a+1)^{1-s},}
\end{equation*}
where $\Gamma(\cdot)$ is the Gamma function.
For all $\varepsilon' \in (0,1)$ and $b \geq 1$, we can change variables with $s \leftarrow 1-\varepsilon' \in(0,1)$ and $a \leftarrow b+\varepsilon'-1 > 0$ to obtain:
\begin{equation}
\label{eq:gauineq}
    {(b+\varepsilon'-1)^{\varepsilon'}>{\frac {\Gamma (b+\varepsilon')}{\Gamma (b)}}>(b+\varepsilon')^{\varepsilon'}.}
\end{equation}
In addition, we have that:
\begin{equation*}
    \mathrm B(x+\varepsilon,y+\varepsilon) / \mathrm B(x,y)=\frac{\Gamma(x+\varepsilon)}{\Gamma(x)}\frac{\Gamma(y+\varepsilon)}{\Gamma(y)}/\frac{\Gamma(x+y+2\varepsilon)}{\Gamma(x+y)}
\end{equation*}
Plugging Eq. \ref{eq:gauineq} into the above concludes the proof.
\end{proof}

\begin{lemma}
\label{lemma:beta-ame-sv-bound}
When $p \sim \mathrm{Beta}(1+\varepsilon, 1+\varepsilon)$ with $\varepsilon<0.5$, then if the utility is bounded in $[0,1]$, 
$$
\|\AME - \SV\|_\infty \leq 2((1+\frac{1}{\varepsilon})^{2\varepsilon}-1).
$$
In addition, under a monotone utility,
$$
\|\AME - \SV\|_2 \leq ((1+\frac{1}{\varepsilon})^{2\varepsilon}-1)+\sqrt{2((1+\frac{1}{\varepsilon})^{2\varepsilon}-1)}.
$$
\end{lemma}

\begin{proof}
With a small abuse of notation, we write $P_{\AME}(j)=\sum_{S: |S|=j} P_{\AME}(S)$ and $P_{\SV}(j)=\sum_{S: |S|=j} P_{\SV}(S)$ (see \S\ref{appendix-sec:sv-estimators} for detailed definition of $P_{\AME}(S)$ and $P_{\SV}(S)$).
Notice that now $P_{\AME}(j)$ is the PMF of a beta binomial distribution $\mathrm{BB}(\alpha=1+\varepsilon,\beta=1+\varepsilon,n=N-1)$. Let $\mathrm{B}(\cdot,\cdot)$ denote the Beta function, then:
\begin{equation}
\label{eq:beta-delta-bound}
\begin{aligned}
    \Delta+1=\max_j P_{\AME}(j)/P_{\SV}(j)&=\max_j N{N-1 \choose j}{\frac {\mathrm {B} (j+\alpha ,N-1-j+\beta )}{\mathrm {B} (\alpha ,\beta )}}\\
    &=\max_j N\cdot\frac{1}{N\mathrm {B}(N-j,j+1)}\cdot\frac{\mathrm {B} (j+\alpha ,N-1-j+\beta )}{\mathrm B(\alpha,\beta)}\\
    (\text{since } \mathrm B(1,1)=1) &=\max_j \frac{\mathrm {B} (j+1+\varepsilon ,N-j+\varepsilon )}{\mathrm {B}(j+1,N-j)}\cdot\frac{\mathrm B(1,1)}{\mathrm B(1+\varepsilon,1+\varepsilon)}\\
    (\text{by lemma \ref{lemma:gauineq}}) &\leq \max_j \frac{(j+1+\varepsilon)^\varepsilon(N-j+\varepsilon)^\varepsilon}{(N+1+2\varepsilon)^{2\varepsilon}} /\frac{\varepsilon^{2\varepsilon}}{(2+2\varepsilon)^{2\varepsilon}}\\
    (\text{maximum reached
    when } j+1+\varepsilon=N-j+\varepsilon) &\leq \frac{1}{4^\varepsilon}\frac{(2+2\varepsilon)^{2\varepsilon}}{\varepsilon^{2\varepsilon}}\\
    &\leq (1+\frac{1}{\varepsilon})^{2\varepsilon}
    ,
\end{aligned}
\end{equation}
Applying Lemma \ref{lemma:general-sv-ame-bound} concludes the proof.
\end{proof}

Now we formally state the $O(k\log N)$ sample rate:
\begin{corollary}
For every $\varepsilon>0, \delta>0, N \geq 3$, there exists constants $C_3(\varepsilon, \delta)$ and $\varepsilon'$, and a regularization parameter $\lambda$, such that when the number of samples $M\geq C_3(\varepsilon, \delta)k\log N$, with a probability at least $1- \delta$,
\begin{enumerate}[label=(\arabic*)]
    \item $\|\sqrt v\hat\beta_{lasso}-\SV\|_\infty \leq \varepsilon$ holds when the utility is bounded in $[0,1]$ and \quantity{} is $k$-sparse; 
    \item $\|\sqrt v\hat\beta_{lasso}-\SV\|_2 \leq \varepsilon$ holds when the utility is monotone and SV is $k$-sparse,
\end{enumerate}
where \quantity{} is defined by $\cP=\mathrm{Beta}(1+\varepsilon', 1+\varepsilon')$ and $v=\mathbb E_{p\sim \cP}[\frac{1}{p(1-p)}]$. 

\end{corollary}
\begin{proof}
Observing that $\|\sqrt v\hat\beta_{lasso}-\SV\|_\infty \leq \|\sqrt v\hat\beta_{lasso}-\AME\|_2 + \|\sqrt v\AME-\SV\|_\infty $ and 
$\|\sqrt v\hat\beta_{lasso}-\SV\|_2 \leq \|\sqrt v\hat\beta_{lasso}-\AME\|_2 + \|\sqrt v\AME-\SV\|_2 $, we proceed by bounding both $\|\sqrt v\hat\beta_{lasso}-\AME\|_2$ and $\|\sqrt v\AME-\SV\|_\infty$ (or $\|\sqrt v\AME-\SV\|_2)$ by $\varepsilon/2$.

Prop. \ref{prop:lasso-rate-p-fea} directly yields the bound on the first term. In both cases (1) and (2), assumptions are verified as follows. First, a $k$-sparse \quantity{} is assumed in Case (1) and implied in Case (2) by a monotone utility plus a $k$-sparse \SV (details in Corollary \ref{appendix-corollary:sparse-AME-SV}).
Second, the $X_n$s are bounded, since $\forall n\in[N]: X_n\in [-\sqrt v,\sqrt v]$ and $v$ is finite:
\begin{equation}
\label{eq:beta-v-def}
    v=\int_0^1\frac{1}{p(1-p)}\frac{p^\varepsilon(1-p)^\varepsilon}{\mathrm B(1+\varepsilon,1+\varepsilon)}dp=\frac{\mathrm B(\varepsilon,\varepsilon)}{\mathrm B(1+\varepsilon,1+\varepsilon)}=4+\frac{2}{\varepsilon},
\end{equation}
where the last equality comes from the fact that $\Gamma(z+1)=z\Gamma(z)$.

To bound the second terms in both cases, notice that each version is given a bound in Lemma \ref{lemma:beta-ame-sv-bound}, and that both approach 0 when $\varepsilon' \to 0$ and are positive when $\varepsilon'>0$. In consequence, there exists $\varepsilon'>0$ dependent only on $\varepsilon$ such that both are $\leq \varepsilon / 2$, concluding the proof.
\end{proof}

\section{Efficient Sparse Beta-Shapley Estimator}
\label{appendix:beta-shapley}
Recall that Beta$(\alpha,\beta)$-Shapley~\cite{kwon2021beta} is our \AME defined on the distribution $\mathcal P=\operatorname{Beta}(\alpha,\beta)$. 
We show that our regression-based \AME estimator with $p$-featurization (\S\ref{appendix:sec-p-fea}) can efficiently estimate Beta$(\alpha,\beta)$-Shapley values when they are $k$-sparse, for all $\alpha>1$ and $\beta>1$, \ie achieving low $L_2$ error with high probability using $O(k\log N)$ samples. 
Prop.~\ref{prop:lasso-rate-p-fea} directly yields the result. 
Its assumptions are verified given that:
\begin{equation}
v=\int_0^1\frac{1}{p(1-p)}\frac{p^{\alpha-1}(1-p)^{\beta-1}}{\operatorname{B}(\alpha, \beta)}dp=\frac{\operatorname{B}(\alpha-1, \beta-1)}{\operatorname{B}(\alpha,\beta)}=\frac{(\alpha+\beta-2)(\alpha+\beta-1)}{(\alpha-1)(\beta-1)}
\end{equation}
is finite (the last equality is due to $\Gamma(z+1)=z\Gamma(z)$) and $\forall n\in[N]: X_n\in[-\sqrt{v}, \sqrt{v}]$ is also bounded.

\section{Extending to Approximate Sparsity}
\label{appendix:approximate-sparsity}

The above discussion assumes exactly $k$-sparse of the SVs, which in practice likely will not hold. 
In this section we extend our result to the case when it is only approximately sparse~\cite{rauhut2010compressive}. 
In such a setting, small non-zeros are allowed in the remaining $N-k$ entries of the SVs. Formally, it requires that the best $k$-sparse approximation $\sigma_k(\SV)=\inf_s\{\|\SV - \mathbf{s}\|_1: \mathbf{s} \text{ is } k \text{-sparse}\}$ is small.

\subsection{The LASSO Error Bound under Approximate Sparsity}
\newcommand{\Classoappri}{C_4}
\newcommand{\Classoapprii}{C_5}
\newcommand{\Classoappriii}{C_6}
\newcommand{\Classoappriv}{C_7}
First we extend the LASSO error bound in Prop.~\ref{prop:simple-lasso-error-bound}. This is relatively easy since Theorem 1.4 from \cite{lecue_regularization_2017} that Prop.~\ref{prop:simple-lasso-error-bound} has simplified supports approximate sparsity. 
We incorporate it by making two changes: a) $\beta^*$ is now allowed to be an approximately sparse vector verifying $\sigma_k(\beta^*)\leq \Classoappri(\delta)\|\xi\|_{L_q}k\sqrt{\frac{\log(N)}{M}}$, where $\Classoappri(\delta)$ is a constant, $\xi=\langle \beta^*, X\rangle-Y$ and $q$ is some constant $>2$; b) the result is accordingly rewritten to the following: there exists a regularization parameter $\lambda$ and a value $\Classoapprii(B-A,\delta)$ such that with probability at least $1-\delta$,
\begin{equation}
\label{eq:lasso:error:original}
    \|\hat{\beta}_{lasso}-\beta^*\|_2 \leq \Classoapprii(B-A,\delta)\|\xi\|_{L_q}\sqrt{\frac{k\log(N)}{M}}.
\end{equation}

The proof is identical to that of Prop.~\ref{prop:simple-lasso-error-bound} thus omitted.
Intuitively, the difference is that a $k$-sparse approximation with small enough error needs to exist to arrive to a similarly small enough error for the LASSO estimate.
Another way to state this is that the sample size $M$ in Eq. \ref{eq:lasso:error:original} is upper bounded by the approximate sparsity requirement (the smaller error the best $k$-sparse approximation is, the larger $M$ can be).
Denote the upper bound by $M_{max}=\max \{M: \sigma_k(\beta^*)\leq \Classoappri(\delta)\|\xi\|_{L_q}k\sqrt{\frac{\log(N)}{M}}\}=\Classoappri(\delta)^2\|\xi\|_{L_q}^2/ \sigma_k(\beta^*)^2k^2\log(N)$.
Approximate sparsity has two implications. First, given that the error bound $\varepsilon=\mathcal E(M)=\Classoapprii(B-A,\delta)\|\xi\|_{L_q}\sqrt{\frac{k\log(N)}{M}}$ decreases monotonically as $M$ increases, the minimal error possibly achievable is lower bounded by a function of the ``sparsity level'' $\sigma_k(\beta^*)$:
\begin{equation}
     \varepsilon \geq \mathcal E(\lfloor M_{max}\rfloor) \approx \frac{\sigma_k(\beta^*)}{\sqrt{k}}\cdot\frac{\Classoappri(\delta)}{\Classoapprii(B-A,\delta)}.
\end{equation}
We note that \cite{jia_towards_2020} shares a similar lower bound on $\varepsilon$.
Second, recall that $M\geq k(1+\log(N/k))=M_{min}$ is required, implying that the theorem is only applicable when $\lfloor M_{max} \rfloor \geq M_{min}$. Because $M_{max}$ increases with lower error sparse approximations, this is equivalent to require that:
\begin{equation}
\label{eq:apprsparse}
    \sigma_k(\beta^*) \leq \Classoappri(\delta)\|\xi\|_{L_q}k\sqrt{\frac{\log(N)}{\lceil M_{min} \rceil}}\approx \Classoappri(\delta)\|\xi\|_{L_q}\sqrt{\frac{k\log(N)}{1+\log(N/k)}}.
\end{equation}

Though this appears to be an extra requirement compared to \cite{jia_towards_2020}, our empirical results suggest that it is not limiting in the cases we studied. 
Indeed, this only rules out sample sizes $\leq k(1+log(N/k))$, smaller than the $M$ needed for good performance across our evaluations.

We now restate the result as a form of $(\varepsilon, \delta)$-approximation to make the result more approachable.
\begin{corollary}
For every sufficiently sparse $\beta^*$ s.t. Equation \ref{eq:apprsparse}, and every $\delta>0, \varepsilon \geq \mathcal E(\lfloor M_{max} \rfloor), N\geq 3$, there exists some constant $\Classoappriv(B-A,\varepsilon,\delta)$ such that when the sample size $M=\max(\lceil M_{min} \rceil, \min(\lfloor M_{max} \rfloor,\lceil \Classoappriv(B-A, \varepsilon, \delta)k\log N \rceil))$, with a probability at least $1-\delta$,
$\|\hat{\beta}_{lasso}-\beta^*\|_2 \leq \varepsilon$.
\end{corollary}
\begin{proof}
Let $\mathcal E(M) \leq \varepsilon$, we have $M \geq C_5(B-A,\delta)^2\|\xi\|_{L_q}^2k\log N/\varepsilon^2$. 
Because of $\|\xi\|_{L_q}^q \leq e$ (a fact proved in the proof of Prop.~\ref{prop:simple-lasso-error-bound}), we can further simplify it to $M \geq C_5(B-A,\delta)^2ek\log N/\varepsilon^2$. 
Let $C_7(B-A,\varepsilon, \delta)=C_5(B-A,\delta)^2e/\varepsilon^2$.
Next, by clipping it with $\lfloor M_{max} \rfloor$, $\sigma_k(\beta^*)\leq \Classoappri(\delta)\|\xi\|_{L_q}k\sqrt{\frac{\log(N)}{M}}$ is verified and $\varepsilon\geq \mathcal E(M)$ still holds due to $\varepsilon \geq \mathcal E(\lfloor M_{max} \rfloor)$; 
Equation \ref{eq:apprsparse} further ensures that there exists at least a choice of $M$ between $\lceil M_{min} \rceil$ and $\lfloor M_{max} \rfloor$.
Finally, by further clipping with $\lceil M_{min} \rceil$, all preconditions are then satisfied and applying Equation~\ref{eq:lasso:error:original} concludes the proof.
\end{proof}

\subsection{Extending the SV estimators}


We first derive an $L_2$ error bound assuming monotone utility, and later discuss an $L_{\infty}$ error bound when utility is not monotone.

As a reminder, we chose a distribution $\mathcal{P}$ such that the \AME under $\mathcal{P}$ is close enough to the \SV, and then apply LASSO to estimate the \AME with a low error. 
The LASSO error bound requires the sparsity of $\beta^*$, which utilizes the sparsity of the \AME, which is derived from the sparsity of the \SV.
When \SV is instead approximately sparse, we can still derive an approximate sparsity guarantee for the \AME and consequently for $\beta^*$.

\begin{lemma}
When the utility is bounded in $[0,1]$ and monotone, $\sigma_k(\beta^*)\leq \sqrt 2\sigma_k(\AME) \leq 3\sqrt 2\sigma_k(\SV)$ holds for both truncated uniform $Uni(\varepsilon', 1-\varepsilon')$ and $Beta(1+\varepsilon',1+\varepsilon')$ with $\varepsilon' \in (0,0.5)$.
\end{lemma}

\begin{proof}
Recall that $\beta^*=\AME/\sqrt v$, we have $\sigma_k(\beta^*)=\frac{\sigma_k(\AME)}{\sqrt{v}}\leq \sqrt 2 \sigma_k(\AME)$, where the inequality is due to $v\geq \frac{1}{2}$ for $Uni(\varepsilon', 1-\varepsilon')$ (see Corollary~\ref{corollary:efficient-sv-linfty}) and $v\geq 4$ for $Beta(1+\varepsilon',1+\varepsilon')$ (see (\ref{eq:beta-v-def})).

Next we connect $\sigma_k(\AME)$ and $\sigma_k(\SV)$. Monotonicity ensures that $\AME_n>0$ and $\SV_n>0, \forall n$. By (\ref{eq:ame-sv-multi-bound}), either $\AME_n\leq \SV_n$ or $\AME_n\leq \SV_n(1+\Delta)$ holds. Thus $\sigma_k(\AME) \leq (1+\Delta)\sigma_k(\SV)$.
Further applying (\ref{eq:mul-bound-dist-connection}) for $Uni(\varepsilon', 1-\varepsilon')$ and (\ref{eq:beta-delta-bound}) for $Beta(1+\varepsilon',1+\varepsilon')$ concludes the proof.
\end{proof}

With a similar application of the LASSO error bound as previously done in \eg Corollary~\ref{corollary:efficient-sv}, we arrive at a similar $(\varepsilon, \delta)$-approximation:
\begin{corollary}
For every constant $\varepsilon>0, \delta>0, N\geq3$, there exists constants $q>2, \varepsilon'$, a LASSO regularization parameter $\lambda$, and an $M=O(k\log N)$, such that with probability at least $1-\delta$:
\begin{enumerate}[label=(\arabic*)]
    \item $\|\sqrt v\hat\beta_{lasso}-\SV\|_\infty \leq \varepsilon$ holds when the utility is bounded in $[0,1]$ and $\sigma_k(\AME)\leq \frac{1}{\sqrt 2}\Classoappri(\delta)\|\xi\|_{L_q}k\sqrt{\frac{\log(N)}{\lceil M_{min} \rceil}}\approx \frac{1}{\sqrt 2}\Classoappri(\delta)\|\xi\|_{L_q}\sqrt{\frac{k\log(N)}{1+\log(N/k)}}$ and $\varepsilon \geq \mathcal E(\lfloor M_{max} \rfloor)\approx O(\sigma_k(\AME) / \sqrt{k})$; 
    \item $\|\sqrt v\hat\beta_{lasso}-\SV\|_2 \leq \varepsilon$ holds when the utility is monotone and $\sigma_k(\SV)\leq \frac{1}{3\sqrt 2}\Classoappri(\delta)\|\xi\|_{L_q}k\sqrt{\frac{\log(N)}{\lceil M_{min} \rceil}}\approx \frac{1}{3\sqrt 2}\Classoappri(\delta)\|\xi\|_{L_q}\sqrt{\frac{k\log(N)}{1+\log(N/k)}}$ and $\varepsilon \geq \mathcal E(\lfloor M_{max} \rfloor)\approx O(\sigma_k(\SV) / \sqrt{k})$,
\end{enumerate}
where the noise $\xi$ is defined as $\langle\beta^*, X\rangle-Y$, the \quantity{} is defined under $\cP=Uni(\varepsilon',1-\varepsilon')$ and $v=\mathbb E_{p\sim \cP}[\frac{1}{p(1-p)}]$.
\end{corollary}

The result of Beta distribution is similar.

\input{appendix-eval}

%% file: appendix-eval.tex
\section{Evaluation Details}
\label{appendix:eval-details}

\subsection{Warm-starting Optimization and Hyperparameter Tuning}
\label{appendix:eval-details-ws}
Given the high cost of training deep learning models, we support an optimization that uses warm-starting as a proxy for full model training. Specifically, instead of training each submodel from scratch, we fine-tune the main model on each subset $S$ for a fixed number of iterations, usually the number of iterations in one main model training epoch. Although this results in a more noisy estimate of $U(S)$, our estimator is able to handle the noise, yielding an overall speedup in wall clock time.

With warm-starting, instead of learning a model from the subset $S$, we ``unlearn'' the signal from the points not in $S$. This has two implications on our choice of $\mathcal{P}$.
First, changing the outcome of a given query usually requires removing all its contributors, as even a small number of them is sufficient to maintain the signal learned in the main model. Hence, we only consider lower inclusion probabilities ($p\leq0.5$).
Second, warm-starting does not collapse model even on very small data subsets, as opposed to learning the model from scratch. We can thus consider smaller values of $p$, and settle on the range $\mathcal{P}=Uni\{0.01, 0.1, 0.2, 0.3, 0.4, 0.5\}$.

\mypara{Hyperparameters of model training for warm-starting}
Warm-start training requires specifying the hyper-parameters (\eg batch size, learning rate, training time, etc.). 
If the original learning rate changed adaptively during the course of training, \eg when using a training approach such as Adam~\cite{kingma2017adam}, we use the learning rate and batch size for the final epoch and run fine-tuning for 
one epoch on the data subset. 

Otherwise we fine-tune every subset model for the same fixed number of iterations, and vary batch size proportionally to the number of training examples included, such that every datapoint is iterated through roughly only once (\ie one epoch on the data subset). Moreover we observe that when batch size is below a certain number the models soon all collapse. Therefore, we lower bound the batch size by that number, which is 100 for both CIFAR10 datasets and 20 for the ImageNet dataset. The reason of such a co-design is the following. To make the numbers comparable, the subset models should be all fine-tuned with the same number of steps. If we still use a constant batch size, every datapoint will be visited different number of times across different data subset sizes, making the impact of one source to be less comparable. Thus we decide to vary the batch size accordingly such that each datapoint is visited roughly once. 
In addition, we choose the learning rate such that the validation accuracy of most subset models drops by roughly 20\%, a not-too-large but still significant number. The reason is that the learning rate should be large enough that when no or few poison is included, the subset model should have the poisoning mostly erased, and vice versa. 

\ignore{
We now add some experiments to support the arguments above.
\mathias{I am removing that so far, which I think is basically Panda's suggestion lower as well. Mostly, it's because it touches on issues we never mention in the paper! We probably want to keep that in mind for later, but for now: (1) it might confuse people and open cans of worms for them; and (2) it might be considered a bad use of appendix to discuss the contribution outside of the space requirement.}

\begin{figure}[h]
    \centering
    \begin{subfigure}{0.48\textwidth}
    \includegraphics[width=1\textwidth]{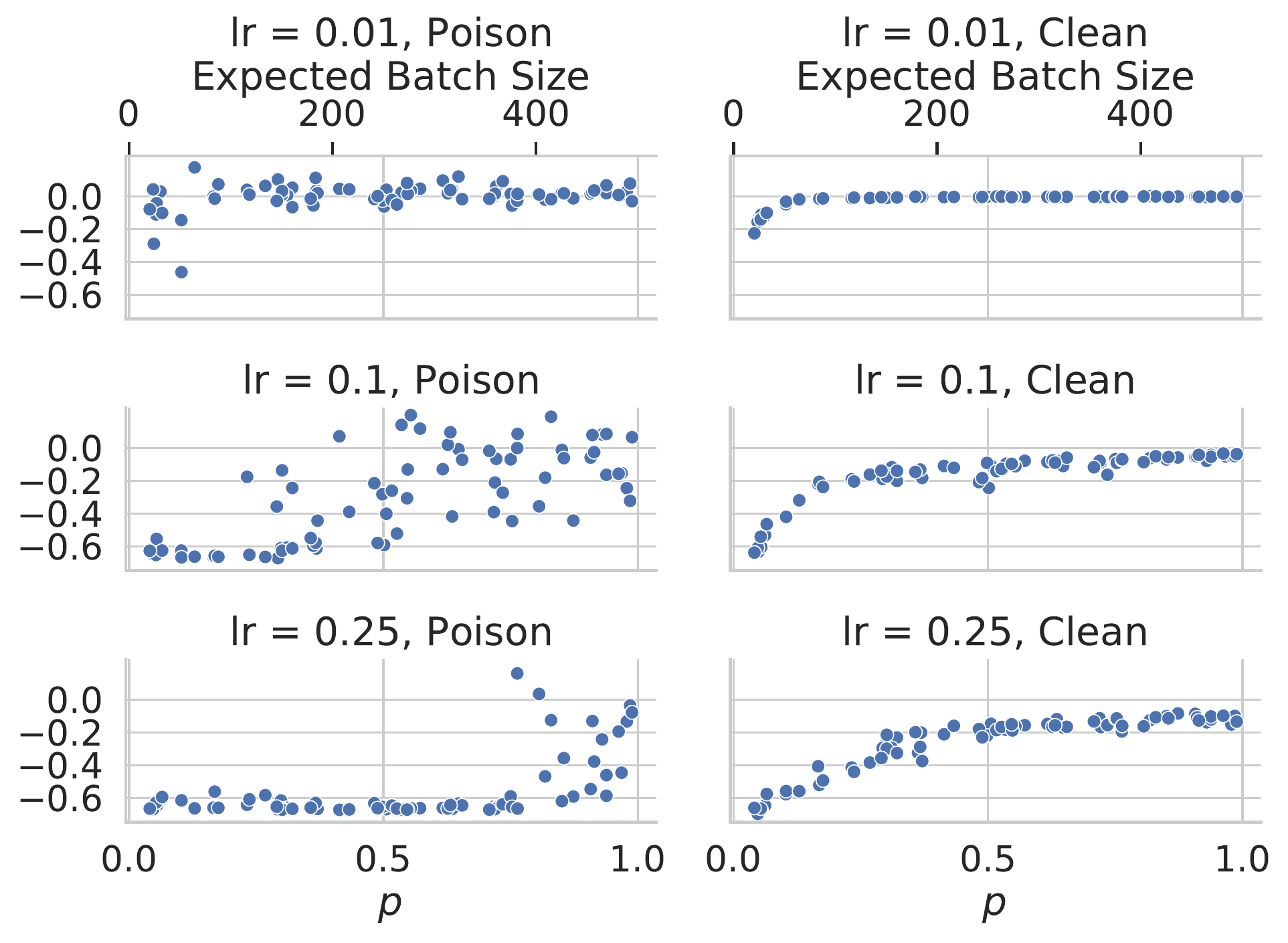}
    \caption{No lower bounding batch size.}
    \label{fig:bsvslr-cifar10-20-ft}
    \end{subfigure}
    \hfill
    
    \begin{subfigure}{0.48\textwidth}
    \centering
    \includegraphics[width=1\textwidth]{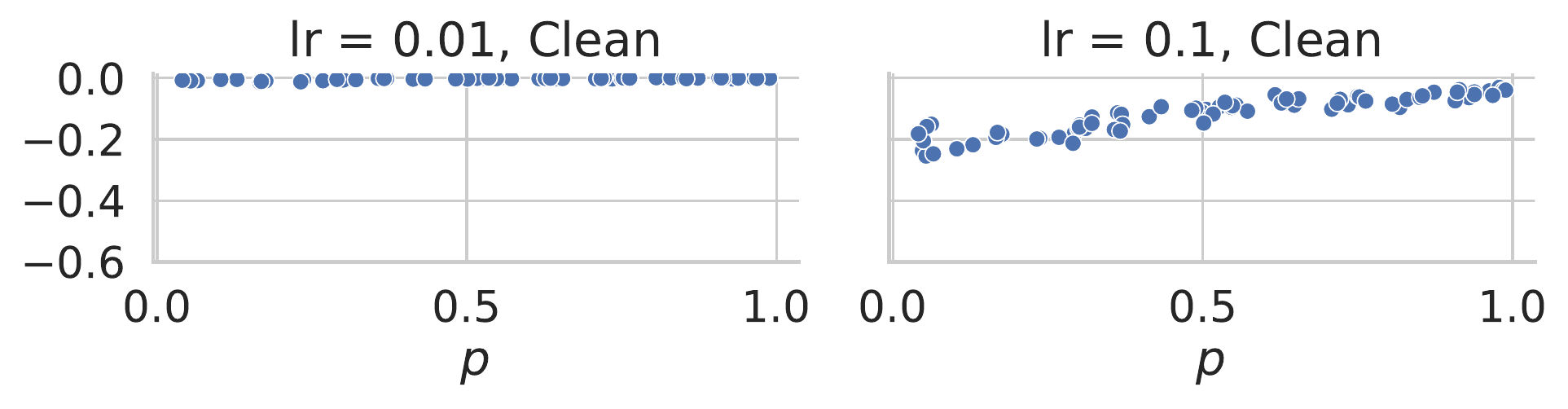}
    \caption{Lower bounding batch size by 100.}
    \label{fig:bsvslr-cifar10-20-ft-lb}
    \end{subfigure}
    \caption{Effect of batch size and learning rate in warm-starting on CIFAR10-20. The y-axis is the change on the attack success rate (left) on the poison validation set and the normal accuracy (right) on the clean validation set. Negative values mean drops on the rate/accuracy compared to the initial model warm-starting is loading from. Each point denotes a subset model fine-tuned on data subset sampled using $p$, thus the expected batch size is $p$ fraction of the original batch size.}
    \label{fig:bsvslr}
\end{figure}




\notepanda{For the next two: These are not hyperparameters that show up at the beginning of the section. We should either add them there, move them to a separate subsection (Microbenchmarks for fine-tuning) or to the appendix. I strongly vote for appendix.}

\emph{The effect of lower bounding batch size:}
We conduct a microbenchmark on CIFAR10-20 dataset, where we fine-tune models using different learning rates on data subsets sampled with $p$ varying from 0.01 to 0.99, without lower bounding the batch size. We split the validation set into two halves and poison one of the set to compute the attack success rate and keep the other one untouched to compute the normal validation accuracy. Both numbers are shown as the change compared to the initial model that warm-starting starts with. The validation accuracy is shown on the right column of Fig. \ref{fig:bsvslr-cifar10-20-ft}, which soon drops when the batch size is lower than a certain number (\ie the low $p$ range) in lr=0.01 and 0.1. This effect disappears under high lr such as 0.5, since now every model has a poor accuracy. We then lower bound the batch size and the result is shown in Fig. \ref{fig:bsvslr-cifar10-20-ft-lb}. The validation accuracy now all remain at a reasonable level for lr=0.1 and 0.01.

\emph{Why pick a learning rate that decreases the accuracy by 20\%?} 
Recall that the goal is to erase poisoning while not break the model. To assess how much the poisoning is erased, we show the attack success rate is shown in the left column in Fig. \ref{fig:bsvslr-cifar10-20-ft}. When $p$ is low and samples nearly no poisons, learning rate of 0.01 is too small to erase the poisoning while 0.5 is too large and erase the poisoning by completely breaking the models; 0.1 strikes the right balance. While in practice the learning rate should be decided during offline phase when we have no information of the attack/queries to compute the attack success rate. Therefore we also compute the normal accuracy on the other clean validation set, and use it to empirically decide learning rate. The ideal learning rate will have some non-negligible influence on the accuracy but not to the extend that totally breaks the model, such as the right middle graph of Fig. \ref{fig:bsvslr-cifar10-20-ft} which decreases it by 20\% for most models. In practice we use this number as the principle to choose learning rate.
}


\subsection{Datasets and Attacks}
\label{appedix:eva-setting-details}
We provide more context and details on the datasets, models, and attacks from Table \ref{tab:dataset-summary}.

\mypara{Datasets} We evaluate our approach using the following four data sets and inference tasks:
\begin{compactenum}[a]
\item \textit{CIFAR10}, an image classification task with ten classes~\cite{krizhevsky2009learning}. We consider each individual training sample as a single source, and use the \texttt{ResNet9} model and training procedure described in~\cite{baek_wbaektorchskeleton_2020}. For one of the attacks (Poison Frogs~\cite{Shafahi2018PoisonFT}), we use \texttt{VGG-11} instead of \texttt{ResNet9}, and use transfer learning to specialize a model trained on the full CIFAR10 data using the training procedure described in~\cite{noauthor_kuangliupytorch-cifar_nodate}. Transfer learning is used to specialize the model for 10\% of the CIFAR10 training data. We only re-train the last layer and freeze every other layers.
\item \textit{EMNIST}, a hand-written digit classification task with examples from thousands of users each with their own writing style \cite{cohen2017emnist}. Each user is a data source with multiple examples. We use models and training procedures from~\cite{noauthor_pytorchexamples_nodate}.
\item \textit{ImageNet}, a one-thousand class image classification task~\cite{ILSVRC15}. The training data includes more images than CIFAR10 (over 1 million vs $50000$), and each image has a higher resolution (average of $482$x$418$ pixels vs $32$x$32$ pixels), thus increasing the training overhead. For ImageNet, we group training data into sources using the URL where the image was collected. Specifically, we treat each URL path (excluding the item name) as a source, and then combine all paths contributing fewer than 10 images into one source. This results in a dataset with $5025$ sources. We use a \texttt{ResNet50}~\cite{he2016deep} model trained using the procedure from~\cite{noauthor_pytorchResnet_nodate}.
\item  \textit{NLP}, a sentiment analysis task on 1 million book reviews written by 307k Amazon users~\cite{nlpdset}. Most users contribute fewer than $1000$ reviews: we select and combine multiple such users at random when producing sources. This results in $1000$ sources, $7$ of which contain reviews from a single user, and the rest group random users into a single source. We use the model and training procedure described in~\cite{bentrevett_rnn_2019} to produce a binary classifier that uses the review text to predict whether or not the review has a positive score (\ie greater than 3).
\label{sec:eval-1-d}
\end{compactenum}

\mypara{Attacks} We evaluate our approach using three types of poisoning attacks. (1) \textit{Trigger attacks}~\cite{chen2017targeted} which rely on a human-visible trigger to poison models. On image detection tasks, we use a 5x5 red square added to the top of each image or a watermark as our trigger. Column $k$ in Table~\ref{tab:dataset-summary} lists the number of poisoned sources. For the NLP task we use a neutral sentence as a trigger. (2) A \textit{label-flipping attack} on EMNIST, where a poison source copies all the data from a benign user (in our case user 171) and associates a single label (in our case \texttt{6}) for all of this copied data.
(3) The \textit{Poison Frogs attack}~\cite{Shafahi2018PoisonFT} on CIFAR10, which is a clean-label attack that introduces imperceptible changes to training images that will poison a target model in a transfer learning setup, where the last few layers of an existing pre-trained model are refined using additional training data to improve inference performance.
Figure~\ref{fig:poison-example} shows examples of trigger and poison frog attacks.

\mypara{Hyperparameters}
We evaluate the impact of hyperparameters using micro-benchmarks in \S\ref{sec:micro}. Unless otherwise stated, we use $q=0$ for knockoffs. As we explain in \S\ref{sec:micro}, this is a conservative choice that seeks to minimize the false discovery rate. We also use $\mathcal{P}=Uni\{0.2,0.4,0.6,0.8\}$ for training-from-scratch and $\mathcal{P}=Uni\{0.01, 0.1, 0.2, 0.3,0.4,0.5\}$ for warm-start training.

We use Glmnet~\cite{noauthor_bbalasub1glmnet_python_nodate,glmnet} for the LASSO implementation. For the LASSO regularization parameter $\lambda$, a correct choice is required by our error bound (Prop. \ref{prop:lasso-rate}), which unfortunately we have no information on. Bypassing this obstacle remains as an interesting future work. Indeed this has been studied in \cite{lecue2017regularizationfollowup} and other slightly weaker LASSO error bounds (\eg Theorem 3.5.1. from \cite{pauwels2020lecture}) but with known $\lambda$ exist.
In practice we choose the $\lambda$ using a common empirical procedure~\cite{glmnet} that runs 20-fold cross-validation (CV) and chooses the largest $\lambda$ (sparsest model) with errors within one standard deviation of the best CV error, which we denote as $\lambda_{1se}$. For the SV estimation, we use $\lambda_{min}$ that gives the best CV error, as we are not seeking a sparsest model here.

 \subsection{Ablation Study for Different Parts of the Methodology and
Parameters}
\label{sec:micro}
Next we use microbenchmarks to evaluate the effect that different hyperparameters have on our methods. Then, we show a discussion of hyperparameters for warm starting, the benefits of using Knockoffs, and a comparison between LASSO and diff-in-means, a straight-forward \quantity{} estimator that uses empirical means to estimate expectations. Finally, we discuss the tradeoff between using training-from-scratch and warm-starting. 

\subsubsection{Effect of Hyperparameters}
\label{subsec:hyperparameters-effect}
\mypara{Effect of the Target FDR Level $q$}
\label{subsec:q-effect}
All results presented thus far were with $q=0$. Here, we vary $q$, which allows us to trade-off precision to achieve higher recall.
For this microbenchmark we use $\lambda_{min}$, which we define as the $\lambda$ with the best cross-validation error, instead of $\lambda_{1se}$ which we use in the rest of our evaluation. This is because the precision provided by LASSO imposes a lower-bound on the precision achieved using Knockoffs', and in our experiments we found that LASSO alone can achieve high-precision when $\lambda_{1se}$ is used, making it harder to observe effects at lower values of $q$.
Figure~\ref{fig:varyq-cifar10-50} shows the results of varying $q$ in the CIFAR10-50 setting. We can see that while increasing $q$ does lead to a small increase in recall, it comes at significant cost to precision.

\begin{figure}[t]
    \centering
    \includegraphics[width=0.6\columnwidth]{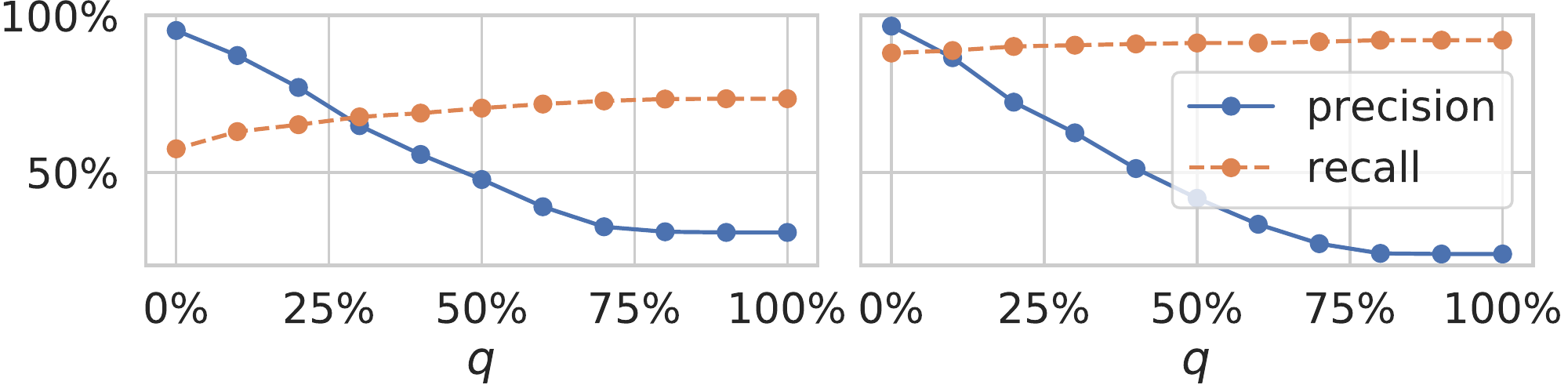}
    \caption{Effect of $q$ in Knockoffs on CIFAR10-50. Left: training-from-scratch; Right: warm-start training.}
    \label{fig:varyq-cifar10-50}
\end{figure}


\begin{figure}[t]
    \centering
    \begin{subfigure}{.49\textwidth}
        \includegraphics[width=\columnwidth]{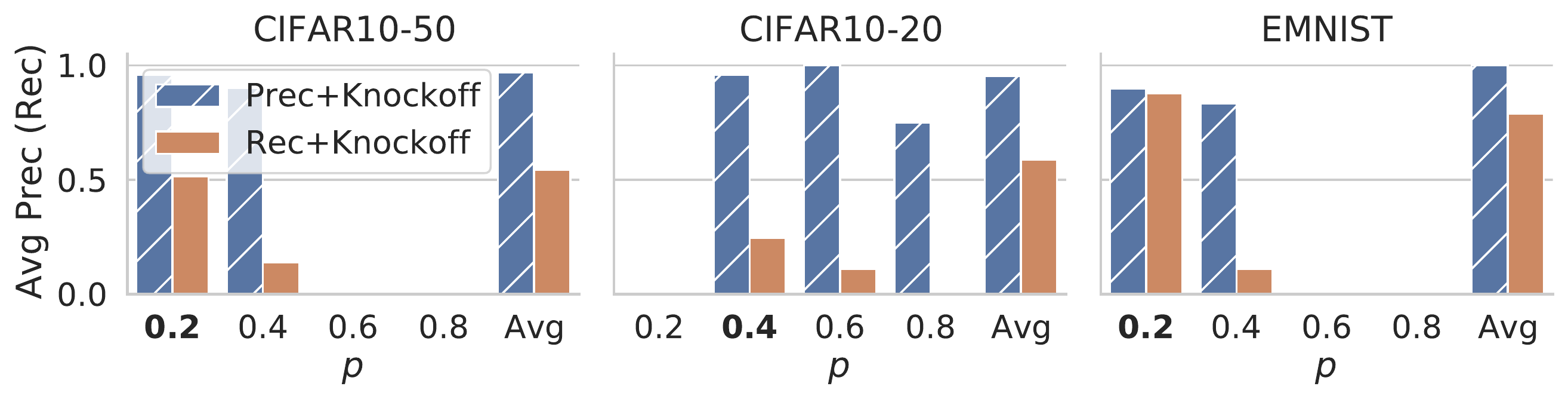}
        \caption{Training-from-scratch}
    \end{subfigure}
    \hfill
    \begin{subfigure}{.49\textwidth}
        \includegraphics[width=\columnwidth]{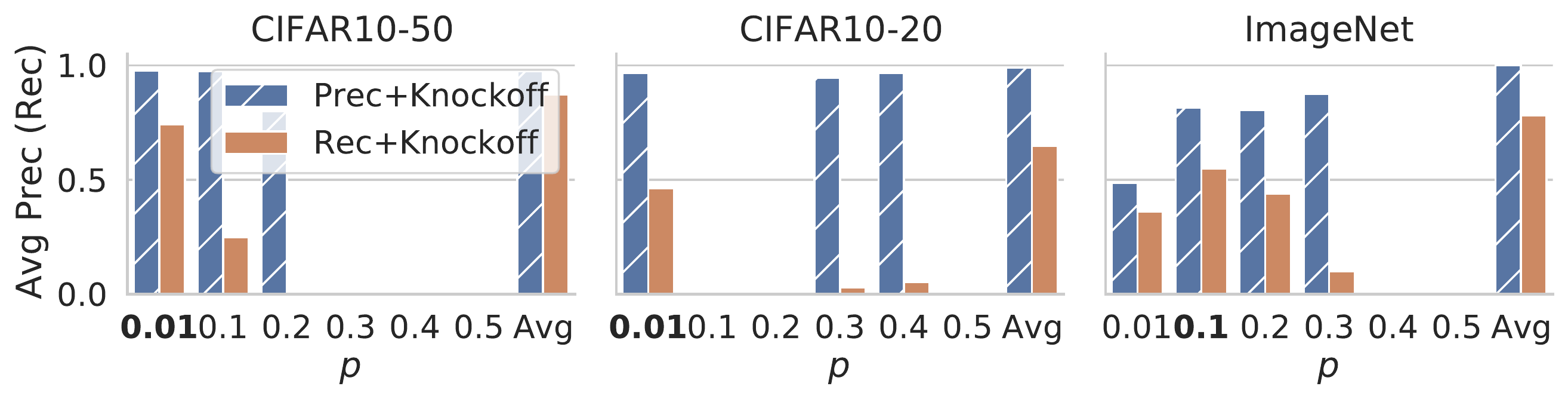}
        \caption{Warm-starting}
    \end{subfigure}
    \caption{LASSO results run only on observations from a single $p$ value, we highlight the $p$ value providing the best result. \texttt{Avg} is the result from using the entire $p$ value grid.}
    \label{fig:lasso-varyp}
\end{figure}


\mypara{Effect of $p$}
Next we address the question of how values of $p$ affect our method. 
Figure~\ref{fig:lasso-varyp} shows our metrics on subset models drawn with one single value of $p$.
We can see that no single value of $p$ suffices across datasets and training algorithms. For instance, training-from-scratch $p=0.2$ works well for CIFAR10-50 and EMNIST, but not for CIFAR10-20.
This difference between CIFAR10-50 and CIFAR10-20 is likely due to the power of the attack: both use the same attack, but CIFAR10-50 uses a larger number of poisoned sources.
This means that smaller values of $p$ are more likely to select poisoned sources in CIFAR10-50, explaining our observations.
For warm starting, we confirm that small values of $p$ ($p=0.01$) perform better than larger ones. These results thus show that (a) no single value of $p$ suffices for all models, thus motivating our use of a grid, and (b) the grid $\mathcal{P}$ can be tuned when a prior is available.

\mypara{Effect of adding less informative $p$s}
One might wonder how including sub-optimal $p$ values impacts our results? To answer this question, in Figure~\ref{fig:lasso-varyp}, we compare the single $p$ results to that of using all $p$s in the grid (shown as \texttt{avg}). We find that while very uninformative $p$s can hurt our results (\eg with EMNIST), this is rare and the effect is small. On the other hand, in many cases the grid result is better than the result from just using the single best $p$.
\ignore{
\subsection{Distribution of $AME$ defined on single-$p$ subset population.}
Previous experiments for evaluating the effect of different $p$s are only assessing through the final result, \ie the precision and recall. To further illustrate its effect, we study the $AME$ measured on each $p$. We bin the observations according their $p$ and calculate the diff-in-means estimates on each bin group. Result in \ref{fig:diff-in-means} shows that some $p$ is good at separating the poisoned and clean sources, while others do not. The best $p$ observed here is generally consistent with that in Fig. \ref{fig:lasso-varyp}, and further justifies the need for varying $p$.
}

\mypara{Effect of different $p$s on different queries}
\begin{table}[t]
\centering
\begin{tabular}{lllll}
Query   & Prec\_0.4 & Rec\_0.4 & Prec\_0.6 & Rec\_0.6 \\
\#3  & \textbf{90.9}      & \textbf{50.0}     & nan       & 0.0      \\
\#7  & \textbf{91.7}      & \textbf{55.0}     & nan       & 0.0      \\
\#9  & 100.0     & 10.0     & \textbf{100.0}     & \textbf{20.0}     \\
\#10 & 100.0     & 15.0     & \textbf{100.0}     & \textbf{25.0}     \\
\#14 & \textbf{100.0}     & \textbf{40.0}     & nan       & 0.0      \\
\#18 & 75.0      & 15.0     & \textbf{100.0}     & \textbf{20.0}    
\end{tabular}
\caption{Single-$p$ LASSO results for $p=0.4$ and $p=0.6$ on selected queries, CIFAR10-20 training-from-scratch. $p=0.4$ has better average results (see Fig. \ref{fig:lasso-varyp}), but $p=0.6$ outperforms for some queries.}
\label{tab:individual-query}
\end{table}
As discussed above, no single value of $p$ suffices across all datasets. Surprisingly, as shown in Table~\ref{tab:individual-query}, we found that even within the same dataset, the best $p$ can differ \emph{across queries}.

\mypara{Comparison against fixing $p=0.5$}
We also compare with a simple distribution where each subset is sampled with equal probability. This corresponds to using fixed $p=0.5$. The result is much worse than our default sampling over a grid as shown in Table \ref{tab:p0.5}.
\begin{table}[h]
\centering
\begin{tabular}{llllll}
    & precision & recall & $c$  \\
\hline
\begin{tabular}[c]{@{}l@{}}CIFAR10-50-tfs \end{tabular} &        100        &      3.2     &8    \\ 
\begin{tabular}[c]{@{}l@{}}CIFAR10-50-ws\end{tabular}        &  100       &         1.8   &    12 \\
EMNIST                                                                      & 100         & 7.8      & 16
\end{tabular}
\caption{Experiment results of using fixed $p=0.5$. ``tfs'' denotes training-from-scratch and ``ws'' denotes warm-starting.}
\label{tab:p0.5}
\end{table}


\subsubsection{Benefit of LASSO and Knockoffs}
\label{appendix:pure-lasso}
We evaluate the benefit of using LASSO over the more straight-forward diff-in-means that replaces the expectation with empirical mean in Eq. \ref{eq:amep}, and the power of Knockoffs in controlling the FDR. First, to have a clean comparison against diff-in-means, we only run LASSO purely without adding the Knockoffs component. The selection is replaced with a procedure that selects all positive coefficients; For diff-in-means we select a threshold such that the recall matches that of pure LASSO for easy comparison. The result in Table \ref{tab:comparision} shows that LASSO performs stably with valid precision while diff-in-means is fragile especially under high noise as in warm-starting. Next we compare this pure LASSO result against LASSO+Knockoff with varying number of observations in Figure~\ref{fig:grow-c-full}. We can see that LASSO+Knockoffs ensure that precision remains high even with a small number of samples. This is in contrast to LASSO which reduces precision when the number of samples decrease. Thus LASSO+Knockoffs allow our technique to be safely used even when an insufficient number of observations are available.

\begin{table}[t]
\centering
\begin{tabular}{llllllll}
        & \begin{tabular}[c]{@{}l@{}}Prec+\\Pure\end{tabular} & \begin{tabular}[c]{@{}l@{}}Rec+\\Pure\end{tabular} & \begin{tabular}[c]{@{}l@{}} Prec+\\Diff\end{tabular} & \begin{tabular}[c]{@{}l@{}}Rec+\\Diff\end{tabular} & $c$  \\
        \hline
Poison Frogs & 84.0 & 100.0 & \textbf{100.0} & 100.0 & 8  \\
CIFAR10-50-tfs  & \textbf{88.3}      &  60.9                & 28.5&60.9               & 16 \\
CIFAR10-20-tfs  & \textbf{79.9} & 64.0  &76.7                 & 64.0   & 8  \\
EMNIST      & \textbf{98.9} & 84.4  & 62.9               & 84.4 & 16  \\\hline
NLP &\textbf{97.9}&98.2& 5.6                  & 98.2  &24\\
CIFAR10-50-ws &\textbf{76.4}&89.6& 41.6                 & 89.6   &24\\
CIFAR10-20-ws &\textbf{98.2}&66.0&9.2                  & 66.0 &48\\
ImageNet &\textbf{100.0}&77.0&94.2                 & 77.0  &12
\end{tabular}
\caption{Average precision and recall of \textit{pure LASSO} and \textit{diff-in-mean}. \emph{tfs} and \emph{ws} denote training-from-scratch and warm-starting respectively.}
\label{tab:comparision}
\end{table}

\begin{figure*}[t]
    \centering
    \begin{subfigure}{0.85\textwidth}
        \includegraphics[width=\textwidth]{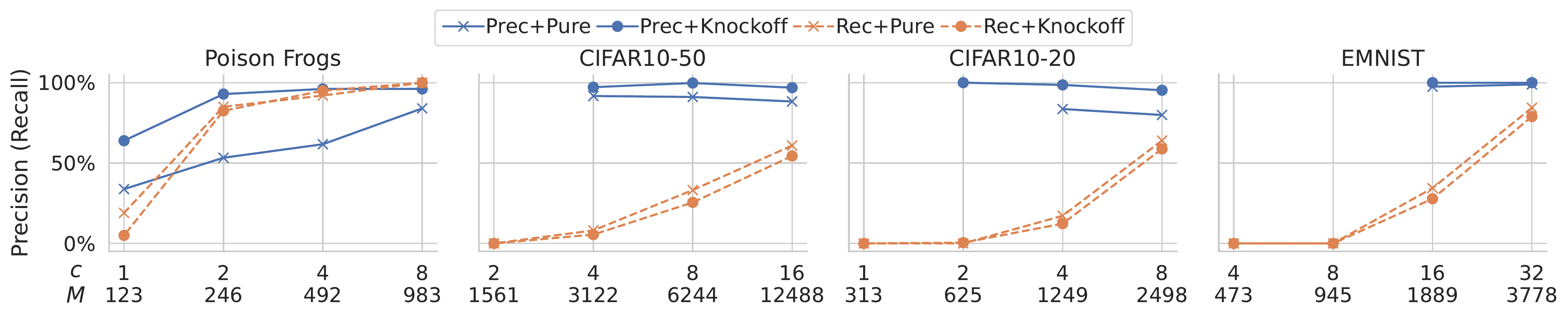}
        \caption{Training-from-scratch}
    \end{subfigure}
    \hfill
    \begin{subfigure}{0.85\textwidth}
        \includegraphics[width=\textwidth]{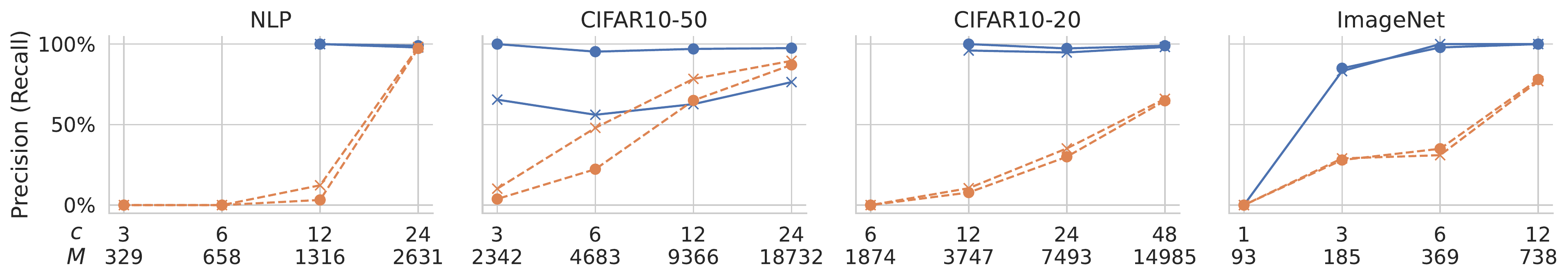}
        \caption{Warm-starting}
    \end{subfigure}
    \caption{Effect of growing $c$. ``Prec+Pure'' is the precision of LASSO without knockoffs; ``Prec+Knockoff'' is the precision of LASSO with knockoffs. ``Rec+Pure'' and ``Rec+Knockoff'' show the recall for both setups. Poison Frogs uses transfer learning and is not amenable to fine-tuning.}
    \label{fig:grow-c-full}
\end{figure*}

\subsubsection{Training-from-scratch vs Warm-starting}
Recall that warm-starting can greatly speed up the training procedure, but with a potential drawback of introducing higher noise for LASSO. In our evaluation, training CIFAR10-50 from scratch took 108.5 seconds, compared to 7.9 seconds when using warm-start training, representing a speedup of 13.8X. However, more observations might be required when using warm-start training. For example, in Table~\ref{tab:tfs-vs-ft} we find that for CIFAR10-20 the warm-start model achieves very low recall ($0.8\%$), while training from scratch achieves reasonable recall ($58.8\%$). We believe this is because of noise due to the number of poison sources required to trigger the attack. On the other hand, we see that CIFAR10-50 does not suffer from this problem, and in fact warm-start provides better results than training from scratch. This is because lower values of $p$ work well to trigger the attack for CIFAR10-50, and warm starting over-samples this region. Adding more observations when using warm-start training with CIFAR10-20 improves this situation, and we find that when using $c=48$ we can achieve a recall of $64.8\%$.

\begin{table}[t]
\centering
\resizebox{0.5\columnwidth}{!}{
\begin{tabular}{llllll}
          & \multicolumn{2}{l}{Training-from-scratch} & \multicolumn{2}{l}{Warm-starting} & \multirow{2}{*}{$c$} \\
          & Prec                 & Rec                & Prec             & Rec            &                    \\\hline
CIFAR10-50 & 96.9             & 54.4             & \textbf{97.0}         & \textbf{77.6}         & 16                 \\
CIFAR10-20 & \textbf{95.3}             & \textbf{58.8}             & 100.0         & 0.8         & 8                 
\end{tabular}
}
\caption{Precision and recall comparison for training-from-scratch vs warm-starting under the same number of observations.}
\label{tab:tfs-vs-ft}
\end{table}

\subsection{Runtime Evaluation}
\label{appendix:runtime}

\begin{figure*}[t]
    \centering
    \begin{subfigure}{0.49\textwidth}
        \includegraphics[width=\textwidth]{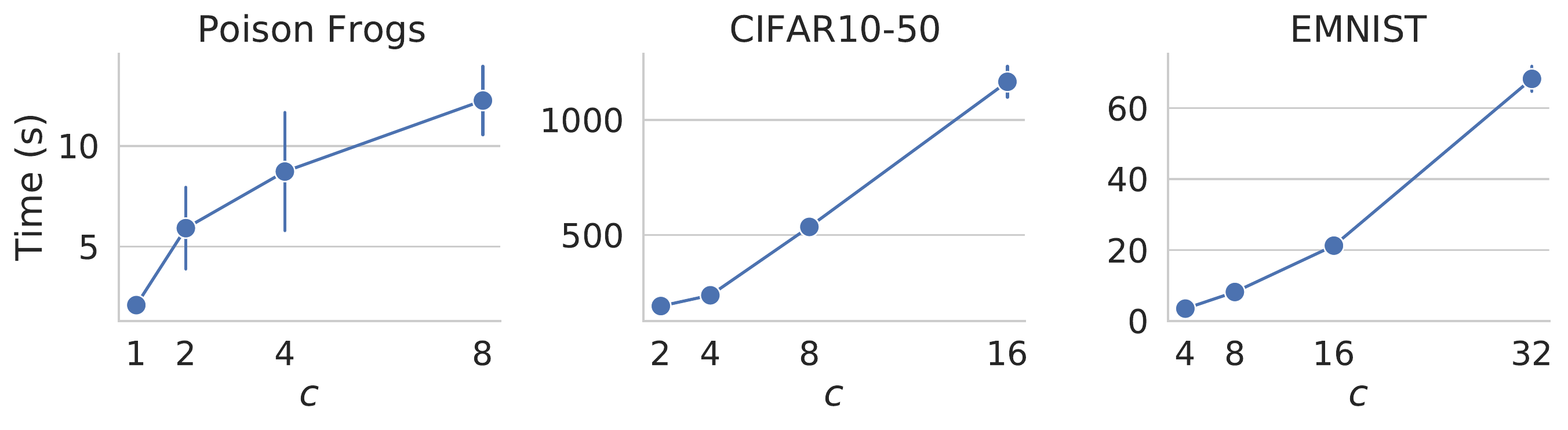}
        \caption{Training-from-scratch}
    \end{subfigure}
    \hfill
    \begin{subfigure}{0.49\textwidth}
        \includegraphics[width=\textwidth]{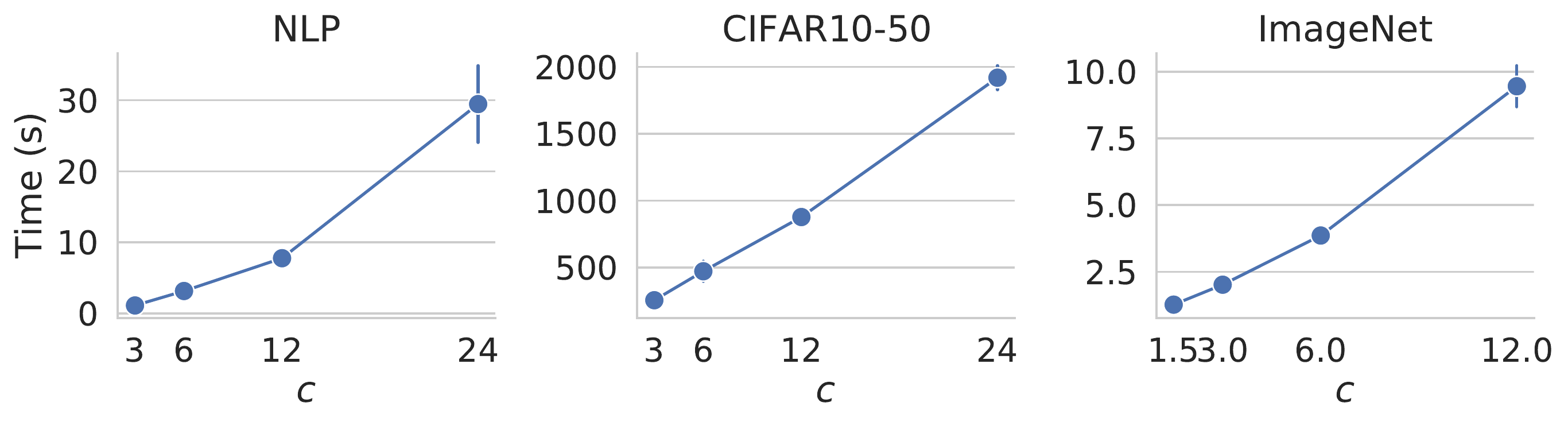}
        \caption{Warm-starting}
    \end{subfigure}
    \caption{Average run times of LASSO on the queries as we grow $c$ on training-from-scratch and warm-starting benchmarks. The error bar draws the standard deviation. CIFAR10-20 results are almost identical and thus omitted.}
    \label{fig:time-grow-c}
\end{figure*}

\mypara{Run times of LASSO}
The time to run LASSO grows linearly with the number of observations, as shown in Figure~\ref{fig:time-grow-c}. 
Together with Figure~\ref{fig:grow-c-full} we can see that our technique can achieve good precision using a small number of observations, which translates to a small query time. Precision and recall improve with more observations, but this increases the time taken to execute a query.

    

    

\mypara{Run times of inference}
\label{appendix:eval:inference}
We report the total time of running ImageNet inference on a single RTX8000 GPU for a batch of 40 queries on 739 models, which corresponds to $c=12$. We implement the inference of a model as three steps: allocating the model memory, loading the model from disk to GPU, and the actual GPU computation. They take 293.6, 447.5, and 584.9 seconds respectively for 739 models in total. The current throughput is 1.81 queries per minute, and the latency is 22.1 minutes. We note that both numbers could be further optimized (\eg the model memory allocation can be done only once and reuse for every model; the model loading and actual computation can be pipelined, etc.), and one can also batch more queries together to further improve throughput. In this paper we focus on improving precision and recall. We plan to focus on further performance optimizations in the future.

\subsection{Additional Evaluation of Comparison with Existing Works}
\label{appendix:comparisons}

\subsubsection{Additional Evaluation of SCAn}
\label{apsec:scan}

\begin{table}[t]
\centering
\begin{tabular}{lllll}
& \multicolumn{2}{c}{Ours} & \multicolumn{2}{c}{Repr.} \\
        & Prec & Rec         & Prec & Rec   \\\hline
ImageNet & 100 & 78.0        & 85.9 & 78.0  \\ 
CIFAR10-50  & 96.9 & 54.4    & 99.9 & 54.4 \\
CIFAR10-20  & 95.3 & 58.8    & 68.9 & 58.8  \\
EMNIST      & 100 & 78.9     & 100 & 78.9
\end{tabular}
\caption{Comparison with Representer Points at the same recall level. Representer Points use best tuned $\lambda$ assuming knowledge of ground truth.}
\label{tab:baseline-best}
\end{table}

\begin{table}[t]
\centering
\begin{tabular}{llllllll}
name            & npoi & prec\_dis  & rec\_dis  & prec & rec & AI        \\ \hline
ImageNet        & 5    & 100 & 80.0 & 100    & 100   & 5.500744  \\
CIFAR10-50  & 50   & 96.4  & 54.0 & 2.1      & 100   & 0.788898  \\
CIFAR10-50* & 100  & 100 & 54.0 & 6.2      & 100   & 1.682893  \\
CIFAR10-50* & 101  & 100 & 54.5 & 7.9      & 100   & 1.744635  \\
CIFAR10-50* & 102  & 100 & 53.9 & 10.0     & 100   & 1.832493  \\
CIFAR10-50* & 103  & 100 & 54.4 & 10.4     & 100   & 1.826089  \\
CIFAR10-50* & 104  & 100 & 54.8 & 100    & 99.0    & 8.742186  \\
CIFAR10-50* & 105  & 100 & 54.3 & 100    & 99.0    & 8.916950  \\
CIFAR10-50* & 106  & 100 & 54.7 & 100    & 98.1    & 8.926177  \\
CIFAR10-50* & 107  & 100 & 54.2 & 100    & 98.1    & 9.054615  \\
CIFAR10-50* & 108  & 100 & 54.6 & 4.1      & 100   & 0.716676  \\
CIFAR10-50* & 109  & 100 & 54.1 & 100    & 98.2    & 9.374960  \\
CIFAR10-50* & 110  & 100 & 54.5 & 100    & 98.2    & 9.583294  \\
CIFAR10-50* & 120  & 100 & 54.2 & 100    & 97.5    & 11.441413 \\
CIFAR10-50* & 130  & 100 & 54.6 & 100    & 97.7    & 12.932239 \\
CIFAR10-50* & 140  & 100 & 54.3 & 100    & 98.6    & 13.600518 \\
CIFAR10-50* & 150  & 100 & 54.7 & 100    & 98.7    & 15.521987 \\
CIFAR10-20  & 20   & 11.1  & 60.0 & 0.9      & 100   & 0.700744  \\
CIFAR10-20* & 40   & 27.9  & 60.0 & 1.7      & 100   & 0.793756  \\
CIFAR10-20* & 60   & 46.1  & 58.3 & 2.5      & 98.3    & 0.894113  \\
CIFAR10-20* & 80   & 52.8  & 58.8 & 3.4      & 98.8    & 0.970892  \\
CIFAR10-20* & 100  & 74.7  & 59.0 & 4.3      & 99.0    & 1.093217  \\
CIFAR10-20* & 120  & 80.5  & 58.3 & 5.2      & 99.2    & 1.253165  \\
CIFAR10-20* & 140  & 87.2  & 58.6 & 6.0      & 98.6    & 1.419778  \\
CIFAR10-20* & 160  & 89.5  & 58.8 & 7.0      & 98.8    & 1.527789  \\
CIFAR10-20* & 180  & 93.0  & 58.9 & 7.9      & 98.3    & 1.746322  \\
CIFAR10-20* & 200  & 97.5  & 59.0 & 9.2      & 98.5    & 1.884684  \\
CIFAR10-20* & 300  & 100 & 58.7 & 99.6     & 92.7    & 4.334019  \\
CIFAR10-20* & 400  & 100 & 58.8 & 100    & 93.0    & 7.996566  \\
CIFAR10-20* & 500  & 100 & 58.8 & 100    & 92.2    & 10.964163 \\
CIFAR10-20* & 600  & 100 & 58.7 & 100    & 92.3    & 14.168781 \\
CIFAR10-20* & 700  & 100 & 58.7 & 100    & 92.0    & 17.779293 \\
CIFAR10-20* & 800  & 100 & 58.8 & 100    & 91.6    & 20.268145 \\
CIFAR10-20* & 900  & 100 & 58.8 & 100    & 91.7    & 24.644904 \\
CIFAR10-20* & 1000 & 100 & 58.8 & 100    & 91.5    & 26.829417 \\
EMNIST          & 10   & 100 & 80.0 & 100    & 100   & 84.367889
\end{tabular}
\caption{Full result of SCAn, where ``*'' indicates the dataset is supplemented, ``\_dis'' indicates the distance-based score.}
\label{tab:scan-full}
\end{table}

SCAn~\cite{tang2021demon} is a state-of-the-art poison defense technique designed to identify whether or not a given model is poisoned, and find the poisoned class when it is. It requires access to some clean data (10\% by default). SCAn computes an AI (Anomaly Index) score for each class, and reports that a class is poisoned if its AI$>7.3891$. 
In their ``online setting'', it can also detect if a query input is poisoned by clustering on features extracted from the last hidden layer's activation. An input is marked as poisoned if it belongs to a poisoned class and is clustered in the group with fewer clean data. We modified SCAn to run on a training set to detect poisoned training data. 

\mypara{Modification of SCAn} Originally in the online setting, SCAn is designed to detect poison in every incoming test datapoint. It trains an untangling model on a clean set of data, and on receiving a test datapoint, it first fine-tunes the model and then use it to cluster with all the data it has to decide if the given test datapoint is poisoned. To modify it to detect poison data in the training set, instead of on the test data we train the untangling model on the training data. For efficiency, we assume the whole training set is available in the beginning, so that we only need to train the model once on the whole training set. 
This gives advantage to SCAn because originally in the beginning the model is poorly fit and gives many false positives because of insufficient data.
SCAn also only gives a coarse-grained (poisoned v.s. clean) outcome to each datapoint, preventing us from exploring its precision-recall tradeoff in evaluation. Thus we modify SCAn to provide a score for each datapoint, using the distance to the hyper-plane ($v^Tr$ in Eq 6 of the paper).

\mypara{The result} Outlier detection approaches like SCAn target a somewhat different problem than data attribution.
Outliers can represent not just attacks but also benign but rare data, and thus outlier-detection cannot differentiate between attacks and benign data, nor between different attacks on the same data. We evaluate this effect using a mixed attack setup on CIFAR10, where we simultaneously use 3 different attacks, each of which poisons 20 images. The attacks used are trigger attacks (similar to CIFAR10-20), except each attack places the trigger in a different location. We found that SCAn clusters nearly all poisoned images (along with many clean images), regardless of the attack used, into the same group, showing that it cannot differentiate between attacks. In contrast, when using our approach with the same hyperparameters as CIFAR10-20 from warm-starting, ours can select the correct attack for each query, and achieves an average precision (recall) of 96.3\% (65.5\%), 97.4 (89.5\%) and 97.1\% (71.3\%), respectively for 20 random queries from each attack.

Furthermore, even when a single type of attack is used, SCAn cannot detect all of the poisoned data. We evaluate this by using SCAn on CIFAR10-50, CIFAR10-20, EMNIST and ImageNet. 
To make SCAn work for EMNIST and ImageNet, which are source-based datasets, we adapt it to assign a score to a source equal to the number of selected datapoints within the source.
The results in Fig.~\ref{fig:pr-curve} shows that SCAn falsely selects many clean data for both CIFAR10 benchmarks, yielding very low precision. This is consistent with the results reported by SCAn authors, who state that at least 50 poison datapoints are needed for SCAn to work robustly. On the other hand, SCAn achieves perfect precision and recall on the other 2 benchmarks, since they are  source-based datasets that contain hundreds of poison datapoints. Further, we found that SCAn's AI fails to identify the poisoned class for all datasets other than EMNIST. Thus SCAn cannot be used unless we are guaranteed to have a sufficiently large set of datapoints. 

\mypara{Additional evaluation for other variants of SCAn} 

We also study if it helps SCAn to add more poison datapoints. We observe that SCAn starts to work well when we run SCAn on 104 and 300 poison datapoints in CIFAR10-50 and CIFAR10-20, respectively.\footnote{Note that the main model is still trained on the original dataset without additional poison datapoints added.} The full result is shown in Table \ref{tab:scan-full}, where the we show both the precision and recall of the original SCAn and that of the variant after our distance modification. In the distance-based variant, we compute the precision and recall with a threshold that makes the recall match that of ours as closely as possible. This leads to an ad-hoc solution to this insufficient data problem: because query inputs are likely to be poisoned data, the operator will wait for an enough number of them to supplement the dataset before running SCAn. However, it is unclear how many queries should s/he wait for since (a) the number of poison data needed can vary across attacks and we do not know it in advance, and (b) even if we know it, we cannot know if this query is an attack. The problem is even worse when multiple attacks coexist, in which case the queries may spread across different attacks/classes, increasing the time needed for each attack/classes to have accumulated sufficient poisons.

On the mixed attack, with more poison datapoints SCAn still clusters the three attacks together, though it stops falsely clustering clean data into the same group as poisoned data.

\begin{table}[t]
\centering
\begin{tabular}{llll}
name           & lambda & precision & recall \\\hline
CIFAR10-20 & 3000.0 & 4.8       & 58.8   \\
CIFAR10-20 & 30.0   & 18.0      & 58.8   \\
CIFAR10-20 & 12.0   & 68.5      & 58.8   \\
CIFAR10-20 & 11.9   & 68.9      & 58.8   \\
CIFAR10-20 & 11.0   & 64.9      & 58.8   \\
CIFAR10-20 & 3.0    & 1.0       & 58.8   \\
CIFAR10-20 & 0.003  & 0.6       & 58.8   \\
\hline
CIFAR10-50 & 3000.0 & 95.2      & 54.4   \\
CIFAR10-50 & 30.0   & 99.9      & 54.4   \\
CIFAR10-50 & 11.9   & 97.8      & 54.4   \\
CIFAR10-50 & 3.0    & 0.7       & 54.4   \\
CIFAR10-50 & 0.003  & 0.7       & 54.4   \\
\hline
EMNIST         & 3000.0 & 100     & 78.9   \\
EMNIST         & 3.0    & 100     & 78.9   \\
EMNIST         & 0.003  & 100     & 78.9   \\
\hline
ImageNet       & 3000.0 & 80.8      & 78.0   \\
ImageNet       & 3.0    & 80.8      & 78.0   \\
ImageNet       & 0.003  & 80.8      & 78.0   \\
ImageNet       & 0.0001 & 84.9      & 78.0   \\
ImageNet       & 4e-05  & 85.9      & 78.0   \\
ImageNet       & 3e-05  & 85.9      & 78.0   \\
ImageNet       & 2e-05  & 84.9      & 78.0   \\
\end{tabular}
\caption{Full result of Representer Point}
\label{tab:repr-full}
\end{table}

\begin{figure}[ht]
    \centering
    \includegraphics[width=0.7\textwidth]{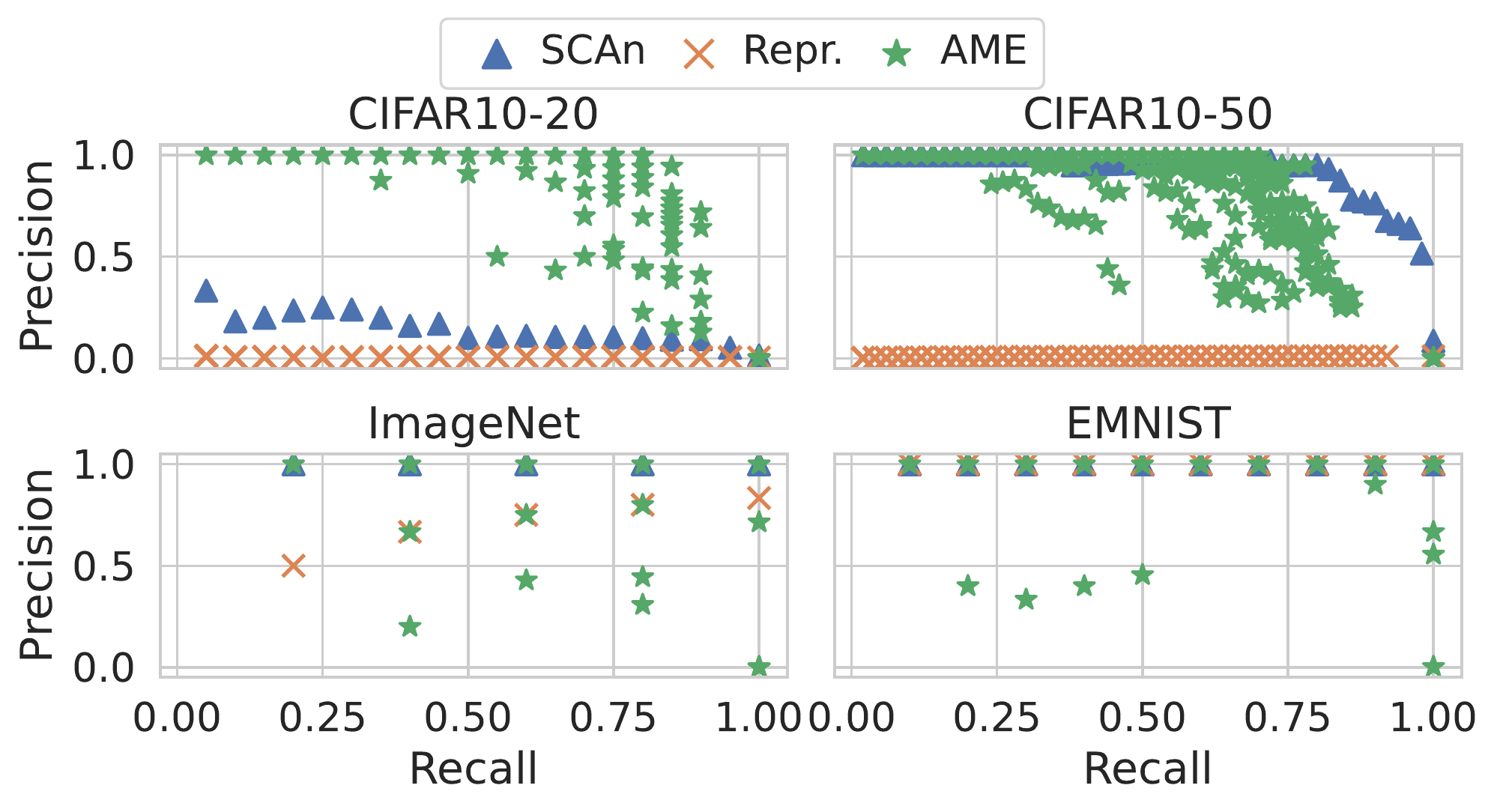}
    \caption{The counterpart of Fig.~\ref{fig:pr-curve} with AME using the regularization parameter choice of $\lambda_{min}$.}
    \label{fig:pr-curve-lambdamin}
\end{figure}

\subsubsection{Additional Evaluation of Rrepresenter Point}
\label{apsec:repr}

Representer Points~\cite{repr} provide an approach to quantify the contribution each training datapoint has on a given query prediction. Given a query $t$, this approach assigns a score $\alpha_{i,Q_L}f_i^Tf_t$ to each training datapoint $i$, where $f_i$ is the feature extracted from the activation of the model's last hidden layer, $f_t$ is the query input's feature, and  $\alpha_{i,Q_L}$ is a scalar computed from the model's gradients.  The query effect $f_i^Tf_t$ is interpreted as the similarity between $i$ and $t$, while the model effect $\alpha_{i,Q_L}$ represents $i$'s importance to the model. There is a hyperparameter $\lambda$ to trade off between the computation time and the result: smaller $\lambda$s provide better results but are slower.

We compute a source-based score by summing up scores for all datapoints within a source, and pick the threshold such that the recall matches AME's recall as closely as possible. In our evaluation, we used $\lambda=3e-3$; Figure \ref{fig:pr-curve} shows the results. While this approach performs better than ours on EMNIST, it is worse on ImageNet and does not work on either of the  CIFAR10 datasets. We note that $\lambda$ is crucial to the result, but the paper's suggestion of using a small $\lambda$ like $3e-3$ does not always lead to good results because it tends to weight model effects over query effects. The paper does not document this problem, so it is unclear how $\lambda$ can be tuned to avoid the problem. Indeed, finding a good $\lambda$ requires prior knowledge of what sources are poisoned, as we show next.

We discover that $\lambda$ has a potential effect crucial to result: it balances between model effect v.s. query effect, which is undocumented in the paper. The default $\lambda$ gives bad result, so we manually tune it per attack with a grid search assuming knowing the ground truth and report the best result, as shown in Table \ref{tab:baseline-best}. It works better than ours on EMNIST and CIFAR10-50 but worse on ImageNet and CIFAR10-20. We note that the best $\lambda$ can vary across datasets: as shown in Table \ref{tab:repr-full} the $\lambda$ that works the best on CIFAR10-50 can be 30, which gives poor result on CIFAR10-20. Therefore, it is unclear how to find the right $\lambda$ that works against unknown attacks even on the same dataset.

To investigate why the default $\lambda$ falls apart, we look at what are selected and find that regardless of the query, it tends to select many the same images: the selection results between two random clean queries of the same predicted label often share $>70\%$ of selections. This is likely because small $\lambda$ tends to overweight the model effect. 

\subsubsection{Additional Evaluation of Influence Functions}
\label{appendix:eval:inf}
\begin{figure}[t]
    \centering
     \begin{subfigure}{0.45\textwidth}
         \centering
         \includegraphics[width=\textwidth]{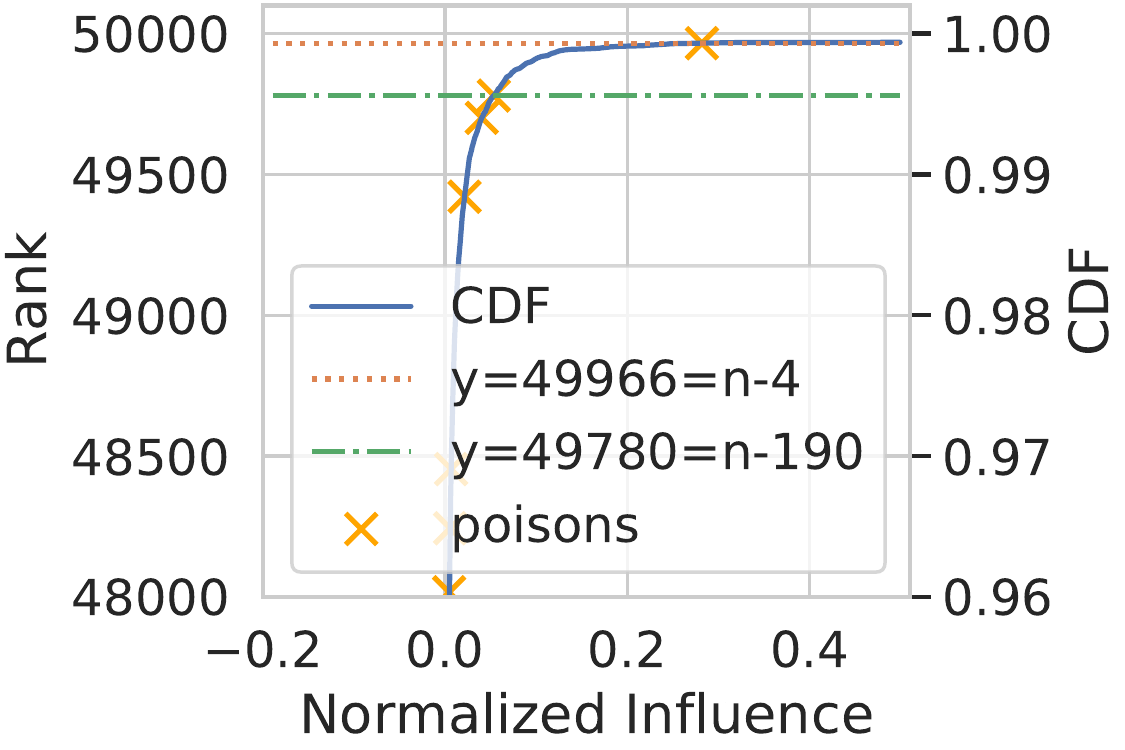}
         \caption{Part of CDF of the influence scores.}
         \label{fig:inf-1}
     \end{subfigure}
     \begin{subfigure}{0.25\textwidth}
         \centering
         \includegraphics[width=\textwidth]{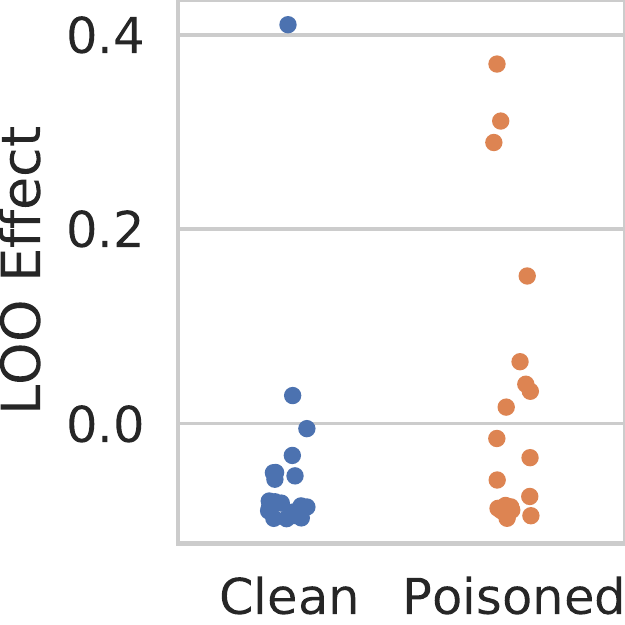}
         \caption{True LOO effect in $y$-axis of selected clean and poisoned datapoints.}
         \label{fig:inf-2}
     \end{subfigure}
    \caption{Experiment of the influence function on CIFAR10-20. The loss function is defined as the average loss on 20 query inputs. We use the parameter configuration of $r=10, t=1000$, damping term $\lambda=0.01$, and scale term=25.}
    \label{fig:inf}
\end{figure}

Influence functions provide an approach to estimating the leave-one-out (LOO) effect, defined as $L_{excluded} - L_{original}$ where $L$ denote the model's loss. We evaluated their efficacy for identifying poison datapoint by computing 
the influence function for each training point in the CIFAR10-20 setup. We expect proponent points will have higher influence scores, and thus poisoned data will have a higher influence score. In Figure~\ref{fig:inf-1} we show a CDF of influence scores for each training datapoint, as well as their rank by influence value. We observe that only 2 (out of 20) of the poisoned data points show up in the top 190 elements. We can thus see that influence functions do not suffice for detecting poisoned datapoints in this case.

We next examine the efficacy of directly using leave-out-one training to detect poisoning. We do so by training 39 models where we leave out the 10 datapoints with the highest influence scores, the 10 datapoints with the lowest influence scores, and any of the 20 poison datapoints not included in these sets. In Figure~\ref{fig:inf-2} we show the LOO effect for all of these models. We can observe from this that many poisoned datapoints have LOO effects that are nearly as low as those observed for correct datapoints. This shows that a more precise technique for estimating LOO would still be insufficient for detecting poisoned sources.

\subsubsection{Additional Evaluation of Shapley Value}
Due to prohibitive costs of computation for running existing Shapley value estimators on our large experiments, we instead provide an evaluation on a subset MNIST with data poisoning (see details of the dataset in Appendix \ref{appendix:subsub:sv-mnist}). 
Unlike previous experiments in which each query explains a single prediction on a test example, we now look at one query to explain the attack success rate on an entire poisoned test set. The utility measurement is hence much less noisy than in our main experiments (we expect the \AME to be even better comparatively in a noisy setting, except maybe for Compressive Sensing).
The baselines are KernelSHAP~\cite{koh_understanding_2017}, Monte Carlo~\cite{ghorbani_data_nodate}, Truncated Monte Carlo~\cite{ghorbani_data_nodate}, and two sparsity-aware approaches ``KernelSHAP (L1)''~\cite{koh_understanding_2017} and Compressive Sensing~\cite{jia_towards_2020} (implementation details in Appendix \ref{appendix:sv}). 
\AME uses the training-from-scratch setting with $\cP=\{0.2, 0.4, 0.6, 0.8\}$.
The sample sizes (\ie number of utility evaluations) are all fixed to 1024, a number that is small for consistency with the large experiment setup, and still exceeds $N=1000$ so that permutation (Monte Carlo) and non-regularized regression (KernelSHAP) based approaches are applicable. 

We again compare the precision of each method at different recall levels, by varying internal decision thresholds. 
The result in Figure~\ref{fig:sv-pr} shows that sparsity-aware approaches achieve better performance than others. 
In particular, \AME and ``KernelSHAP (L1)'' achieve the same performance and outperform other approaches. 
We also compare them for \SV estimation in \S\ref{eval:sv}, and show that KernelSHAP (L1) is less suited in that case.

\begin{figure}[h]
    \centering
    \includegraphics[width=0.6\textwidth]{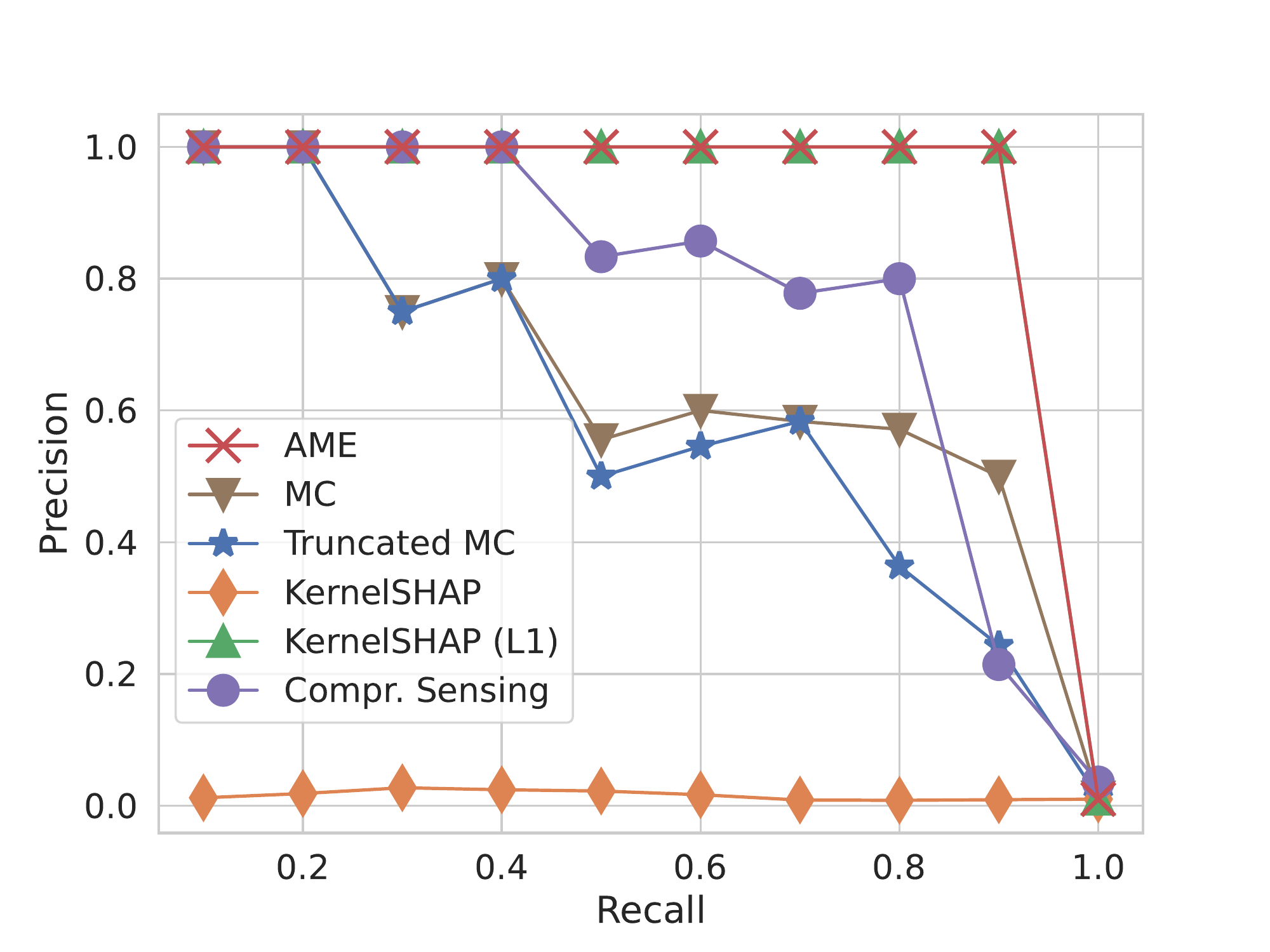}
    \caption{Precision vs Recall curve comparison.}
    \label{fig:sv-pr}
\end{figure}

\subsection{Details of the Hierarchical Design}
\label{appendix:hierarchical}

Our hierarchical design is as follows: we partition users into groups based on when they first posted a review, which we treat as a proxy for when the user account was created. We group users into partitions of  $1000$ users each (so groups span different time periods: we found that fixed group sizes performed better than fixed time periods).
Unlike the previous evaluation with the NLP data set (\S\ref{sec:eval-1-d}), in this experiment we do not combine tail users, i.e., users with a small number of review, into a single user. We use this process to partition the NLP dataset into $n_1=308$ top-level sources, each of which contains (up to) $1000$ second-level sources. This partition covers all 307,159 users in the NLP dataset (\ie $n_2=307,159$). We poison $15$ second-level sources belonging to $2$ randomly chosen top-level sources: putting $10$ poisoned sources in one top-level source, and $5$ in the other. We compare the results of applying our hierarchical approach in this setting, to the results of using the general approach on $307,159$ sources. 

We use different $p$s when selecting top-level and second-level sources in the hierarchical approach, because top-level sources might contain multiple proponents and thus have a larger influence on inference results. In our experiment, we used $\mathcal{P}=Uni\{0.2, 0.4, 0.6, 0.8\}$) for the first-level ($p_1$) and trained models from scratch, and we used $\mathcal{P}=Uni\{0.1, 0.2, 0.3, 0.4, 0.5\}$ for the second level ($p_2$) and used warm-start training. 

Figure~\ref{fig:hierarchy} shows the precision and recall achieved by the hierarchical approach on 20 different queries, as we vary the number of subset models. We find that (a) the hierarchical approach achieves good precision and recall in identifying top-level sources that contribute to data poisoning, even when fewer subset models are used than would allow second-level sources to be detected; and (b) recall when identifying second-level sources degrades gracefully as fewer observations are used. The hierarchical approach is also more efficient: it can achieve near-perfect precision and $60\%$ recall with 1,588 observations. In comparison, the general method could not identify any of the poisoned sources when using the same number of observations.

Finally, the hierarchical approach also reduces the time spent running LASSO+Knockoffs, \ie it reduces query runtime. To demonstrate this, we compare the running times for executing a query with 1,588 observations, using $5$ threads on a server with 256GB of memory. The non-hierarchical approach takes $52.5$s to construct the design matrix and $688.58$s to run LASSO+Knockoffs. The hierarchical approach takes $37.343$s to construct two design matrices ($0.123$s for the top level, and $37.220$s for the second level) and $13.829$s to run LASSO+Knockoffs ($2.369$s for the top level, $11.460$s for the second level). This corresponds to an overall query time of $48.68$s for the hierarchical approach versus $741.08$s for the general approach, a reduction of $15\times$.

\subsection{Details and Additional Evaluation of Data Attribution for Non-poisoned Predictions}
\label{appendix:eval:general}

{
  \newcommand{\gw}{0.4}
  \newcommand{\gwh}{1}
  
  \begin{figure}[h]
  \captionsetup[subfloat]{labelformat=empty,skip=0pt,belowskip=1pt}
  \centering
  \begin{tabular}{c@{\hskip 0.1in}c@{\hskip 0.02in}c@{\hskip 0.02in}c@{\hskip 0.02in}c@{\hskip 0.02in}c@{\hskip 0.02in}c@{\hskip 0.02in}c@{\hskip 0.02in}c@{\hskip 0.02in}c}
    Query Input&\multicolumn{5}{c}{Examples Selected} \\ \hline
  \subfloat[birds]{\includegraphics[width = \gw in]{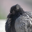}} &
  \subfloat[birds]{\includegraphics[width = \gw in]{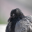}} &
  \subfloat[birds]{\includegraphics[width = \gw in]{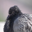}} &
  
  \\
  \subfloat[deer]{\includegraphics[width = \gw in]{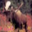}} &
  \\
  \subfloat[frogs]{\includegraphics[width = \gw in]{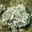}} &
  \subfloat[frogs]{\includegraphics[width = \gw in]{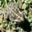}} &
  \subfloat[frogs]{\includegraphics[width = \gw in]{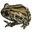}} &
  \subfloat[birds]{\includegraphics[width = \gw in]{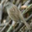}} &
  \\
  \subfloat[ships]{\includegraphics[width = \gw in]{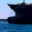}} &
  \subfloat[ships]{\includegraphics[width = \gw in]{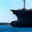}} &
  \subfloat[ships]{\includegraphics[width = \gw in]{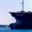}} &
  \subfloat[ships]{\includegraphics[width = \gw in]{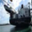}} &
  \subfloat[ships]{\includegraphics[width = \gw in]{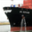}} &
  \subfloat[ships]{\includegraphics[width = \gw in]{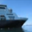}} &
  \end{tabular}
  \caption{Queries for correct predictions. The last row only shows a subset since they are too long to fit in.\label{fig:app-grid-correct}}
    \end{figure}
  
  \begin{figure}[h]
    \captionsetup[subfloat]{labelformat=empty,skip=0pt,belowskip=1pt}
    \centering
    \begin{tabular}{c@{\hskip 0.1in}c@{\hskip 0.02in}c@{\hskip 0.02in}c@{\hskip 0.02in}c@{\hskip 0.02in}c@{\hskip 0.02in}c@{\hskip 0.02in}c@{\hskip 0.02in}c@{\hskip 0.02in}c}
      Query Input&\multicolumn{5}{c}{Examples Selected} \\ \hline

    \subfloat[pre: airpl.; gt: cats]{\parbox[t]{\gwh in}{\centering\includegraphics[width = \gw in]{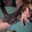}}} &
    \subfloat[airpl.]{\includegraphics[width = \gw in]{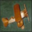}} &
    \subfloat[airpl.]{\includegraphics[width = \gw in]{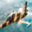}} &
    \subfloat[airpl.]{\includegraphics[width = \gw in]{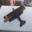}} &
    \subfloat[airpl.]{\includegraphics[width = \gw in]{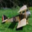}} 
    \\
    \subfloat[pre: cats; gt: dogs]{\parbox[t]{\gwh in}{\centering\includegraphics[width = \gw in]{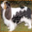}}} &
    \subfloat[cats]{\includegraphics[width = \gw in]{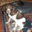}} &
    \subfloat[cats]{\includegraphics[width = \gw in]{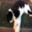}} &
    \subfloat[cats]{\includegraphics[width = \gw in]{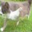}} &
    \subfloat[dogs]{\includegraphics[width = \gw in]{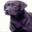}} &

    \\
    \subfloat[pre: deer gt: horses]{\parbox[t]{\gwh in}{\centering\includegraphics[width = \gw in]{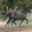}}} &
    \subfloat[deer]{\includegraphics[width = \gw in]{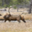}} &
    \subfloat[birds]{\includegraphics[width = \gw in]{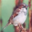}} &
    \subfloat[deer]{\includegraphics[width = \gw in]{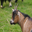}} &
    \\
    \subfloat[pre: airpl. gt: ships]{\parbox[t]{\gwh in}{\centering\includegraphics[width = \gw in]{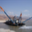}}} &
    \subfloat[airpl.]{\includegraphics[width = \gw in]{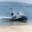}} &
    \subfloat[airpl.]{\includegraphics[width = \gw in]{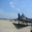}} &
    \subfloat[airpl.]{\includegraphics[width = \gw in]{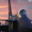}} &
    \subfloat[airpl.]{\includegraphics[width = \gw in]{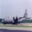}} &
    \subfloat[airpl.]{\includegraphics[width = \gw in]{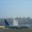}} &
    \end{tabular}
    \caption{Queries for wrong predictions. \textit{Pre} is model's prediction and \textit{gt} is the ground truth. The last row only shows a subset since they are too long to fit in. \label{fig:app-grid-wrong}}
    \end{figure}
}

\begin{figure}[ht]
    \centering
    \includegraphics[width=0.45\textwidth]{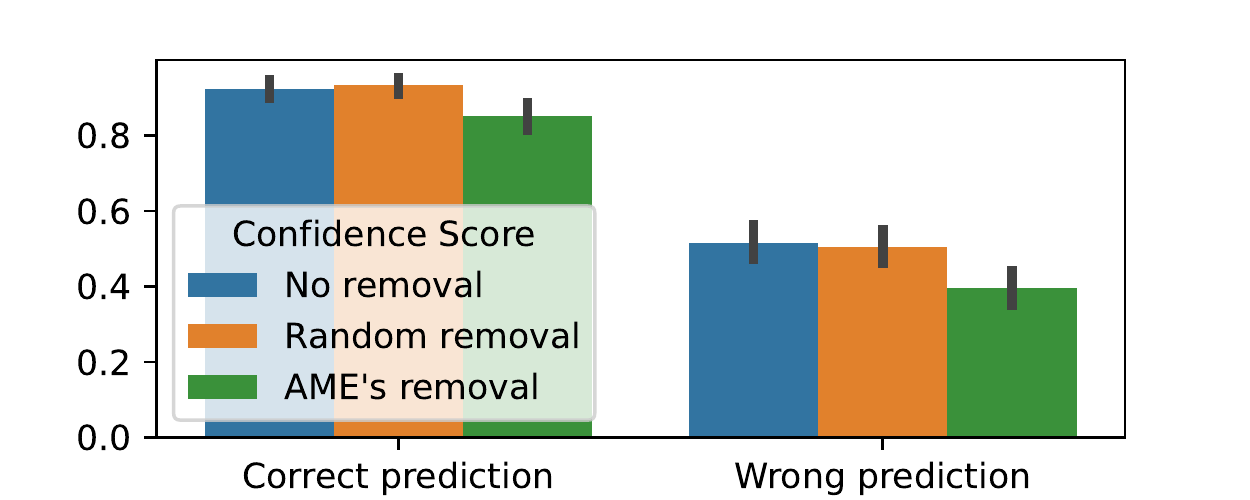}
    \caption{Quantitative comparison of model explanation through the drop in the confidence scores. Y-axis shows the average confidence scores across all correct/wrong queries. The 95\% confidence interval is drawn as the vertical bars.}
    \label{fig:conf-drop}
\end{figure}

We randomly select 10 classes and their images from the ImageNet dataset and use the same ResNet-9 model and training procedure. The main model achieves a top-1 accuracy of 89.2\% on the validation set, from which we randomly select 40 with correct predictions and 40 with incorrect predictions, as the queries. We train roughly 2700 subset models (calculated from $c=10,k=20,N\approx 10000$) from scratch. We switch to using $\lambda_{min}$ (with $q=0$) since empirically $\lambda_{1se}$ is too conservative and selects few images, if any. 

The results in Figure~\ref{fig:grid-wrong} and \ref{fig:grid-correct} (due to space constraints we show at most 3 selected images for each query) show that many selected images share similar visual characteristics with the query input, either in color, blur effect or texture. Some even are photos taken on the same place from different time/angle. These results suggest that our approach can be used to understand model behavior beyond the use case of data poisoning. Additional results for different queries, and a comparison to a naive baseline, can be found at \url{https://enola2022.github.io/}. 

We also evaluate model explanation on CIFAR10-50 dataset. We reuse subset models from the CIFAR10-50 training-from-scratch experiments\footnote{We did this to save time for model training. Though the training set includes poison, they should have little effect on the non-poisoned query inputs.}, and randomly select 39 images from the vanilla, non-poisoned test set as query inputs. Some queries with correct main model predictions (4 out of 39 in our evaluation) still fail because LASSO does not converge in 3600 seconds, possibly because no sparse solution exists. As a result we return an empty result set. Selected results are shown in Figure~\ref{fig:app-grid-correct} and \ref{fig:app-grid-wrong}, and we also draw all 39 queries we have ran and their result in grids, as shown in Fig. \ref{fig:all-queries-correct} and Fig. \ref{fig:all-queries-wrong}, where each row consists of two queries, each starting with the query input followed by images selected. 

We also report a quantitative result on CIFAR10-50 by removing the images selected by ours and retraining the model 6 times and measure the change in the predictions. We find that the predicted labels are often not changed compared to the main model. This is expected because our approach is designed to achieve good precision and as a result only a small number of proponents from the strongest are removed, and other proponents in the class generically support the prediction. However, we still see the average confidence score is lower compare to the main model's when only removing the strongest proponents, as shown in Fig. \ref{fig:conf-drop}, while a baseline that randomly selects datapoints in the class to remove fails to do so.

\subsection{Shapley Value Estimation Setting and Additional Evaluations}
\label{appendix:sv}
\subsubsection{Simulated Dataset}
\label{appendix:subsub:sv-sim}
\mypara{Experimental Setup} We craft a two-valued threshold function as the utility, which evaluates to 1 when there are at least $2$ in the first $k$ sources are present. Formally, $U(S)=\mathbf{I}(|S\cap [k]| \geq 2)$. We pick $k=3$ and $N=1000$ for this experiment. This design allows us to compute the estimation error $\|\sqrt v\hat\beta_{lasso}-\SV\|_2^2$ because the true value of Shapley values for the $k$ sources can be easily known: by symmetry they are $U([N])/k=1/k$ for the $k$ sources and 0 otherwise. 

We compared our \AME based \SV estimator to Monte Carlo~\cite{jia_towards_2020,ghorbani_data_nodate}, Truncated Monte Carlo~\cite{ghorbani_data_nodate}, Group Testing~\cite{jia_towards_2020}, Compressive Sensing~\cite{jia_towards_2020}, KernelSHAP~\cite{lundberg2017unified}, and Paired Sampling~\cite{covert2021improving}. 
We run each approach 6 times with different random seeds. 
We plot the 95\% confidence interval in the shaded area. The baseline is as follows:
\begin{compactenum}[a]
    \item \textit{Truncated Monte Carlo.} We adopt the implementation published by the author~\cite{ghorbani_data_nodate} and use the default hyperparameters (\eg truncation tolerance is set to 0.01).
    \item \textit{Group Testing.} We adopt the implementation provided by the author~\cite{jia_gt}. We use hyperparameter $\epsilon=\frac{2}{\sqrt{N}}$ as in the script. We also tried $\epsilon=0.01$ and $0.1$ and the results remain the same, thus omitted.
    \item \textit{Compressive Sensing.} 
We implement the algorithm using CVXPY~\cite{agrawal2018rewriting, diamond2016cvxpy}. The algorithm has a hyperparameter ``$M$'' that has no clear documentation for its choice.
We therefore run one set of the 6-trial experiment for every $M\in\{2^7, 2^{8}, \dots, 2^{16}\}$ and report for each sample size the best mean estimation error (the mean is taken on the 6 trials). 
This assesses its upper limit. We use $\epsilon=0.01$ for the other hyperparameter ``$\epsilon$'' when not specified. Unlike the original algorithm that shifts the SVs by the average utility $U([N])/N$ (see $\bar s$ in Algorithm 2 in their paper), we do not perform shifting because in a poisoning (thus sparse) setup most SVs are more likely to be zero rather than the average.
    \item \textit{KernelSHAP.} We use their official implementation with the default hyperparameter setting. Note that by default KernelSHAP also uses $L_1$ regularization in a heuristic way. We also evaluate it with the regularization turned off to understand the effect. 
    \item \textit{Paired Sampling.} We use their official implementation with batch size equal to 128, no convergence detection and non-stochastic cooperative game.
\end{compactenum}

\mypara{Additional evaluations and ablation study}
In \S\ref{eval:sv} in the main body we have shown a condensed result comparing against part of the baselines. Now we present the comparison against other baselines in Fig.~\ref{fig:est-sim-extra}. High-level findings remain the same. 
We can also see that smaller $\varepsilon$ leads to higher approximation precision at the cost of convergence speed.
We also provide additional results of other choices of $\epsilon$ for Compressing Sensing in Figure~\ref{fig:est-sim-cs}.
It shows that $\epsilon=0.1$ though has a marginal drop compared to $\epsilon=0.01$ in estimation error when the sample size $\leq2^{12}$ but with a (relatively big) sacrifice on the estimation error on large sample sizes, echoing a tradeoff role $\varepsilon$ plays in AME (see \S\ref{eval:sv}), while $\epsilon=0.001$ has almost identical result, suggesting a saturation on one end of the tradeoff space.
This may also explain another interesting finding: the results are almost the same as Monte Carlo for the two smaller $\epsilon$'s. 
Recall that compressive sensing utilizes $\epsilon$ to control the tradeoff between the sparsity and the accuracy (w.r.t. the measurements) of the recovered solution. 
Jia et al.'s design implicitly makes use of Monte Carlo for measurement, implying that the most accurate solution is Monte Carlo. 
The smaller the $\epsilon$, the more accurate the recovery and hence the closer the result is to Monte Carlo.

\begin{figure}[ht]
    \centering
\hfill
\begin{subfigure}{0.4\textwidth}
    \centering
    \includegraphics[width=1\textwidth]{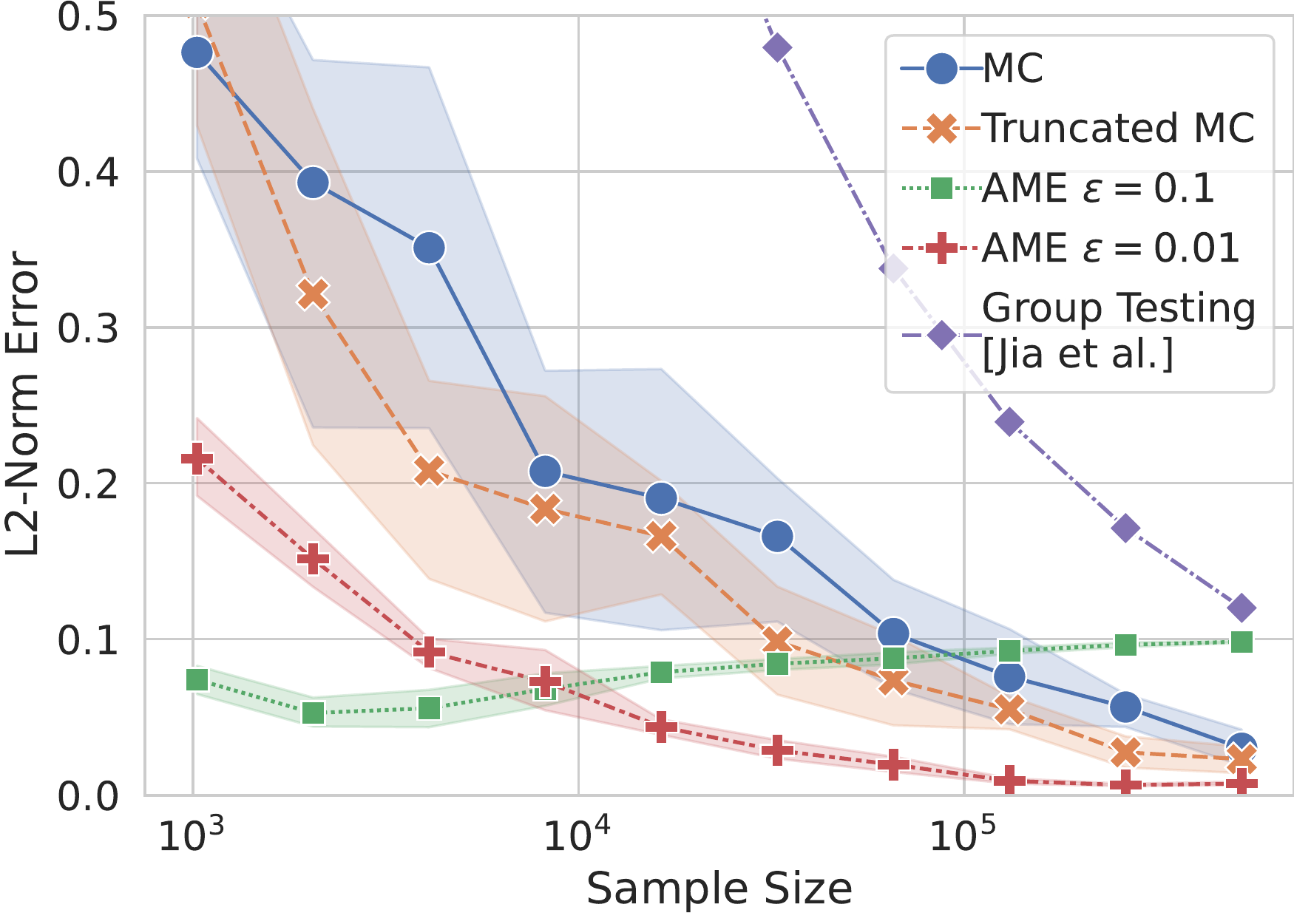}
    \caption{Simulated dataset.\vspace{\baselineskip}}
    \label{fig:est-sim-extra}
\end{subfigure}
\hfill
\begin{subfigure}{0.4\textwidth}
    \centering
    \includegraphics[width=1\textwidth]{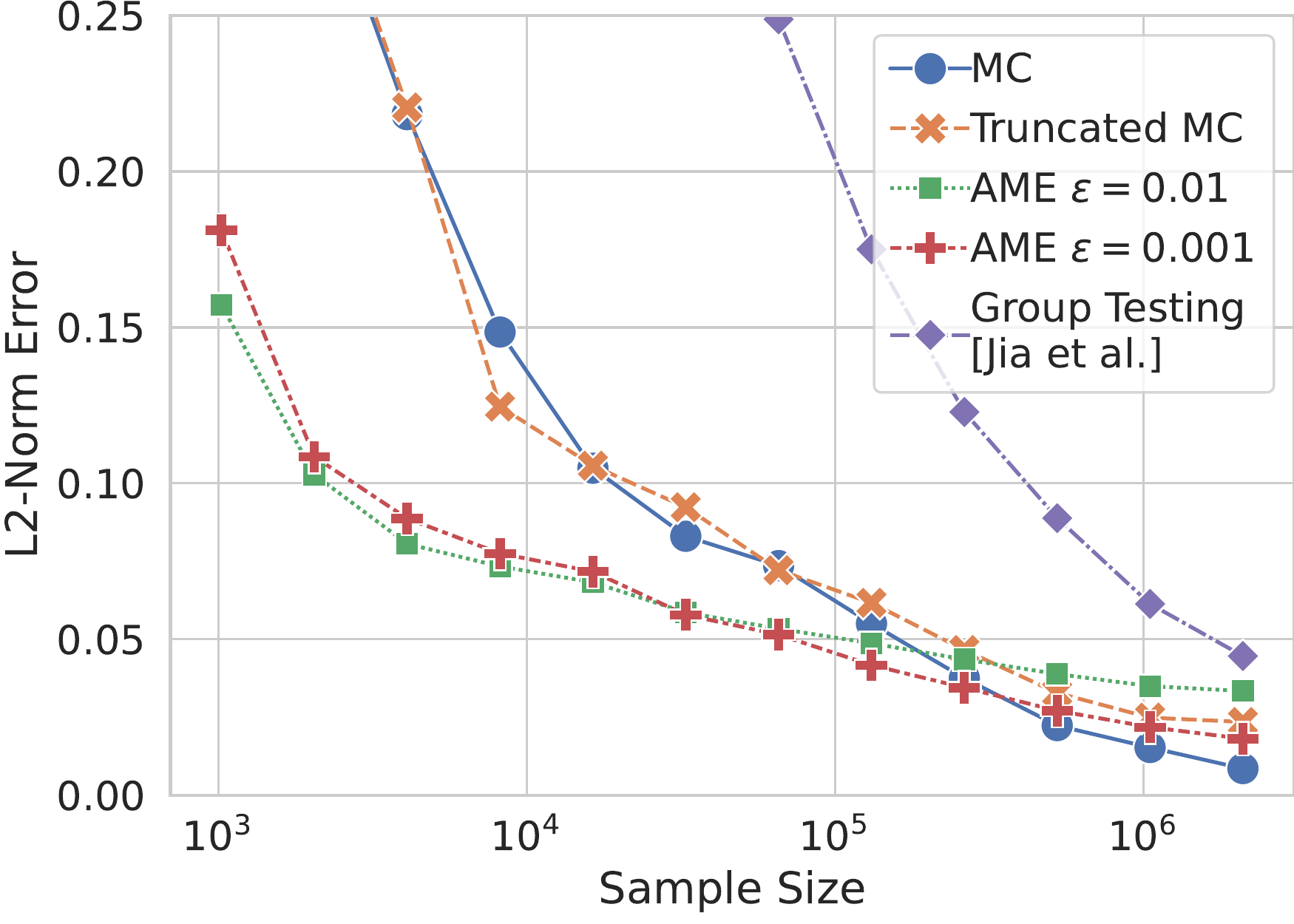}
    \caption{MNIST dataset. We use as ground truth the Monte Carlo (MC)'s result on a sample size of $2^{22}$.}
    \label{fig:est-mnist-extra}
\end{subfigure}
\hfill
\caption{Additional comparison against other baselines on SV estimation.}
\end{figure}

\begin{figure}[h]
    \centering

\hfill
\begin{subfigure}{0.4\textwidth}
    \centering
    \includegraphics[width=1\textwidth]{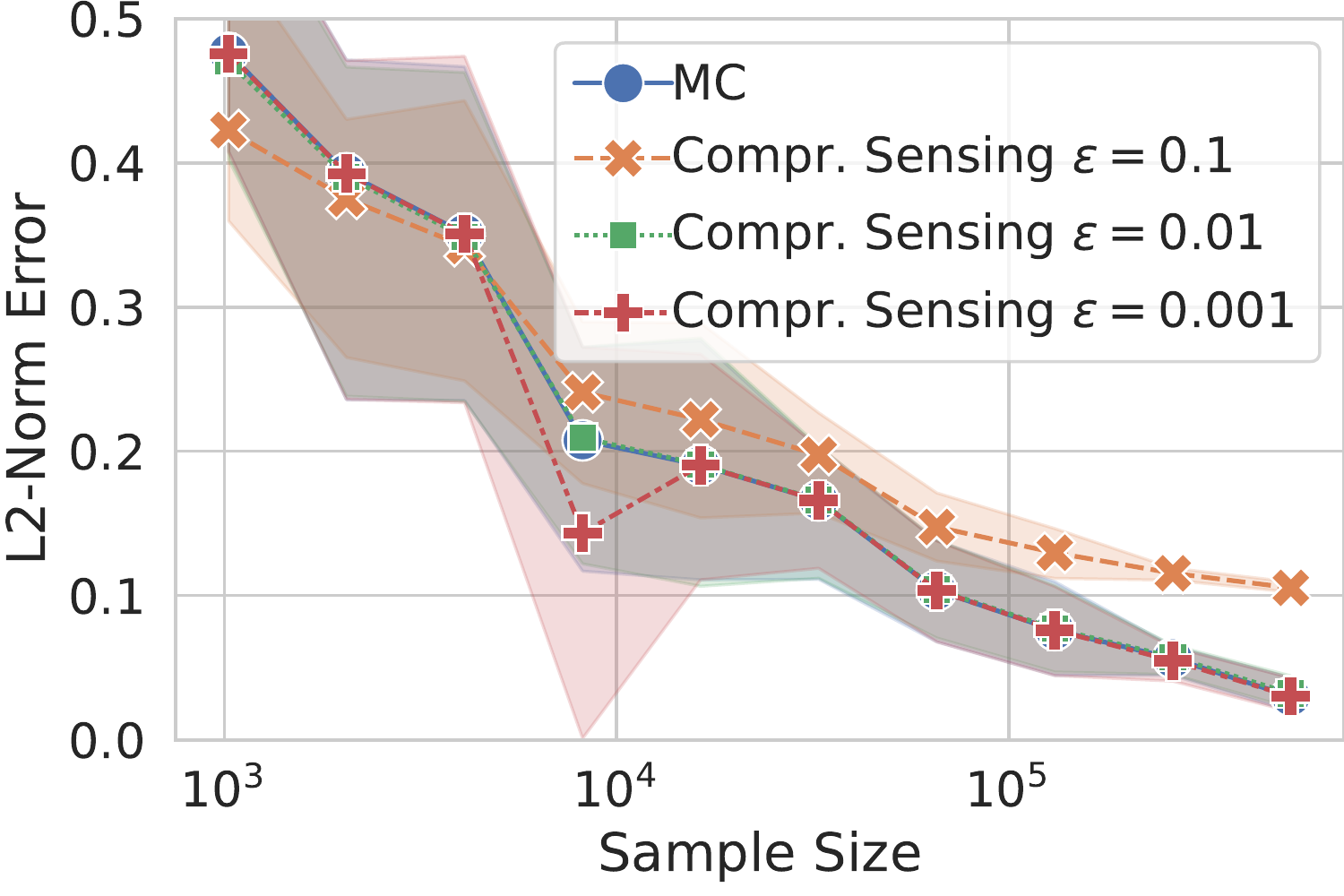}
    \caption{Simulated dataset.}
    \label{fig:est-sim-cs}
\end{subfigure}
\hfill
\begin{subfigure}{0.4\textwidth}
    \centering
    \includegraphics[width=1\textwidth]{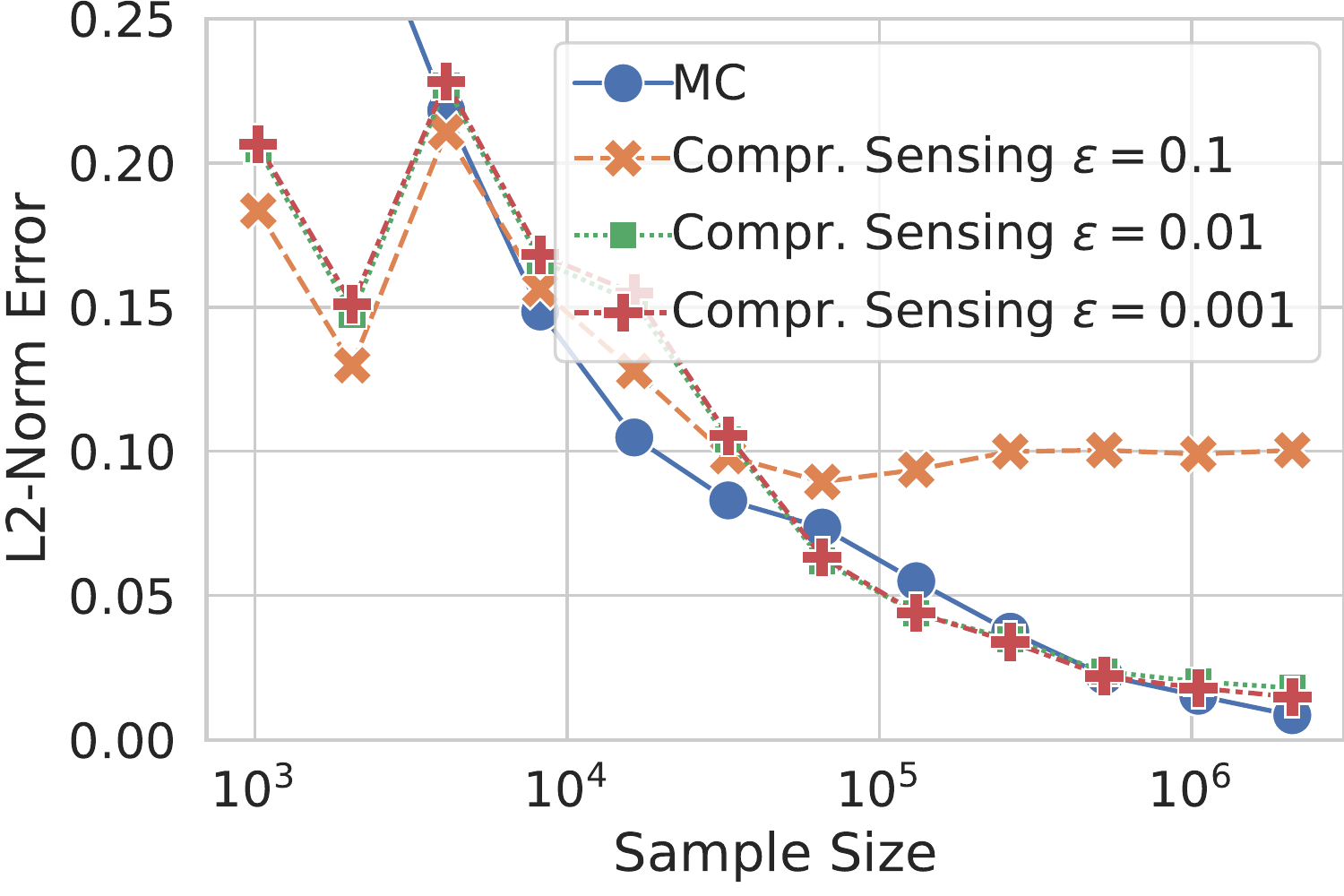}
    \caption{MNIST dataset.}
    \label{fig:est-mnist-cs}
\end{subfigure}
    \caption{Compressive Sensing with different choices of $\epsilon$.}
    \label{fig:est-cs}
\hfill
\end{figure}

\begin{figure}[h]
    \centering
\hfill
\begin{minipage}{0.4\textwidth}
    \centering
    \includegraphics[width=1\textwidth]{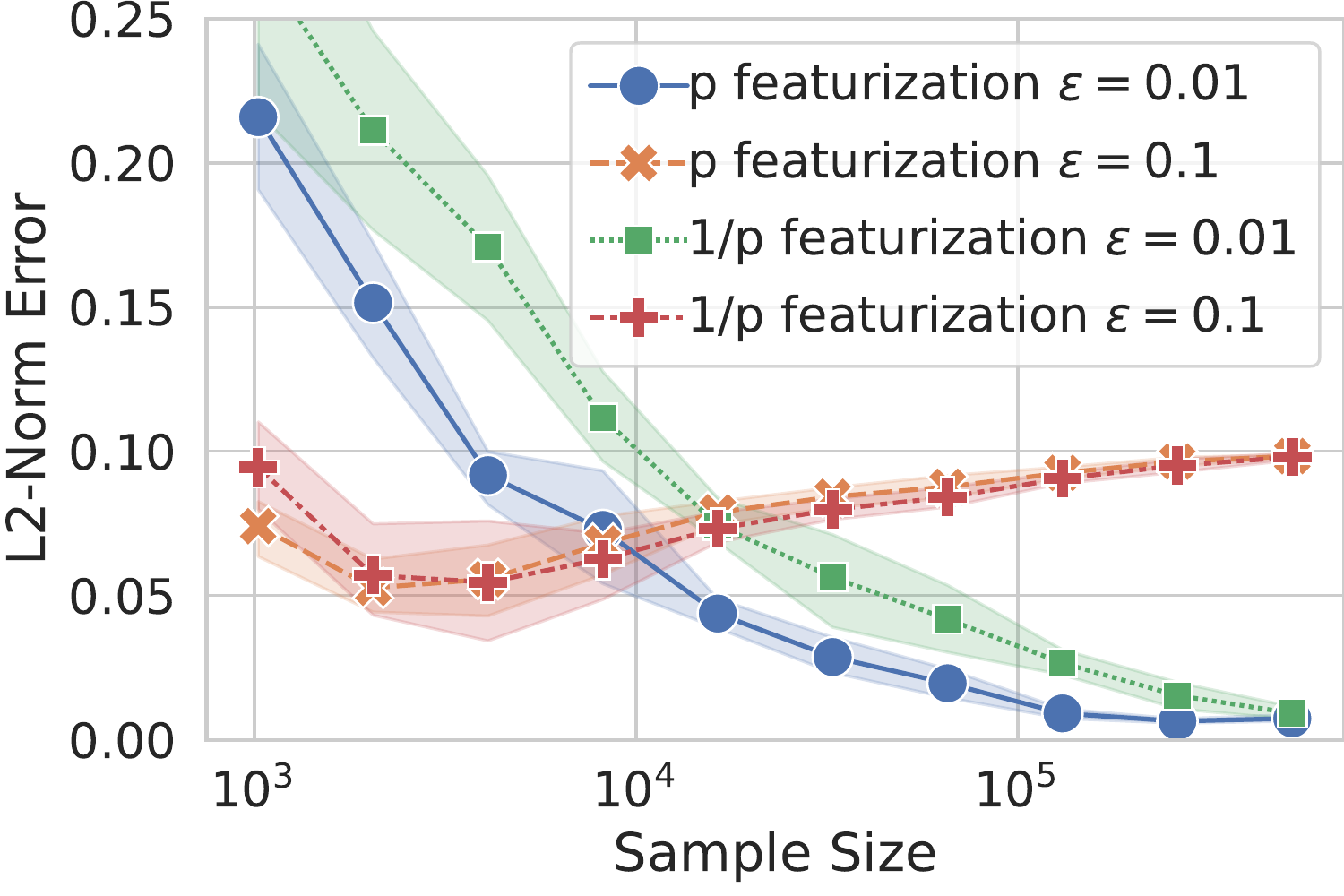}
    \caption{$p$-featurization vs $1/p$-featurization (\S\ref{appendix:sec-p-fea}). Both use truncated uniform distribution $Uni(\varepsilon, 1-\varepsilon)$.}
    \label{fig:est-fea-cmp}
\end{minipage}
\hfill
\begin{minipage}{0.4\textwidth}
    \centering
    \includegraphics[width=1\textwidth]{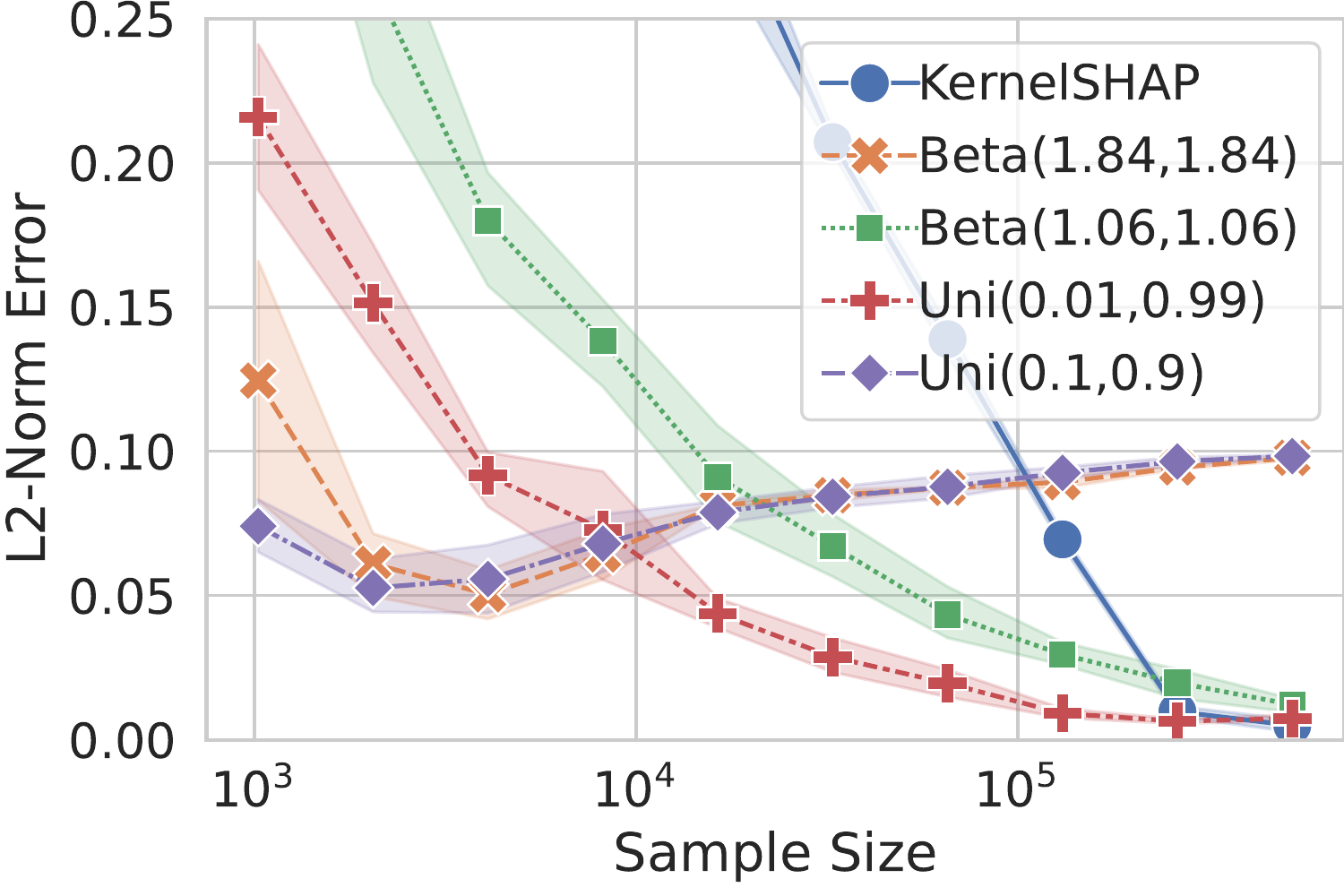}
    \caption{Beta vs truncated uniform distribution. Our approaches (Beta* and Uni*) all use $p$-featurization (\S\ref{appendix:sec-p-fea}). }
    \label{fig:est-pfea-beta}
\end{minipage}
\hfill
\end{figure}

We provide additional evaluations regarding the choice of featurization and distribution of our SV estimator from AME.
In Fig. \ref{fig:est-fea-cmp} we compare the impact of featurization and fix the distribution to truncated uniform. It shows that they both work well and $p$-featurization performs slightly better (\S\ref{appendix:sec-p-fea}); In Fig. \ref{fig:est-pfea-beta} we compare beta distribution with truncated uniform $Uni(\varepsilon, 1-\varepsilon)$. The parameters for beta distribution are chosen to match that of truncated uniform such that their AMEs evaluate to the same value. In other words, they both converge to the same estimation error on an infinite sample size. It shows that they both perform well, but truncated uniform performs slightly better.

\subsubsection{Poisoned MNIST dataset}
\label{appendix:subsub:sv-mnist}
\mypara{Experimental Setup} For non-simulated case, given the prohibitive cost of computing the true SV, we craft a tiny MNIST dataset by randomly subsampling 1000 datapoints from both the original training set and test set. To create sparsity, we further poison 10 training datapoints randomly by imposing a white square on the top-left corner and overwriting their labels to 0. We also craft a poison test set by imposing the same trigger on the clean test set and use the attack success rate as the utility function. We use logistic regression (with regularizaiton $0.02$) as the model for fast training. The resulting attack success rate of the full model is 65.7\%, and the test accuracy is 89.4\%. Different from the simulation case, it is computationally infeasible to compute the true SV. We instead use the estimation from the classic Monte Carlo method~\cite{ghorbani_data_nodate} as a proxy. We keep doubling the sample size until the estimation converges with a consecutive $L_2$ difference falling below $0.01$, and use the final estimation as the proxy.

We compare the same set of approaches as in the simulated experiment (except Paired Sampling given the computational cost and its almost identical performance to KernelSHAP on the simulated dataset). The configurations for these approaches mostly remain the same, except we only run one trial to save computational cost.

\mypara{Additional evaluations}
Fig.~\ref{fig:est-mnist-extra} shows the comparison against other baselines not shown in \S\ref{eval:sv} in the main body, with Monte Carlo, Truncated Monte Carlo and Group Testing added. 
The qualitative findings hold. 

\begin{figure*}[t]
    \centering
\begin{subfigure}{\textwidth}
    \includegraphics[width=\textwidth]{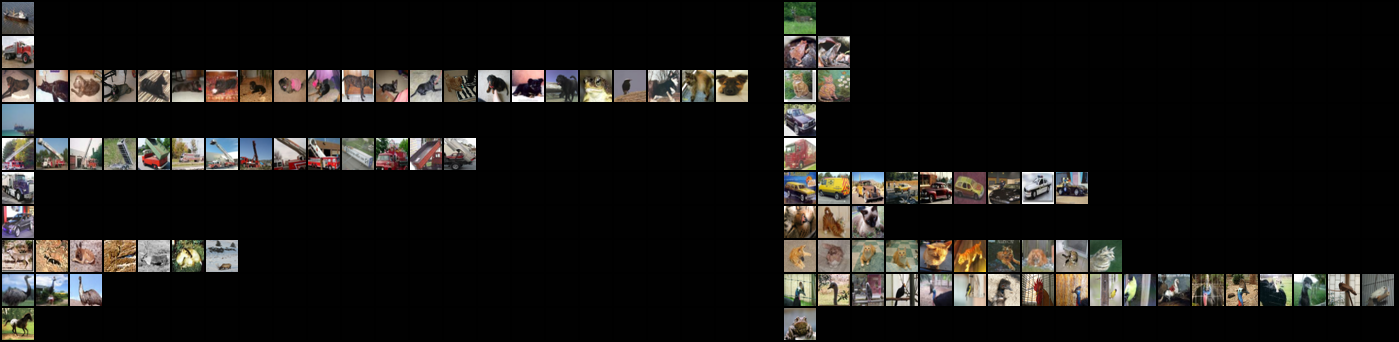}
\end{subfigure}
\hfill
\begin{subfigure}{\textwidth}
    \includegraphics[width=\textwidth]{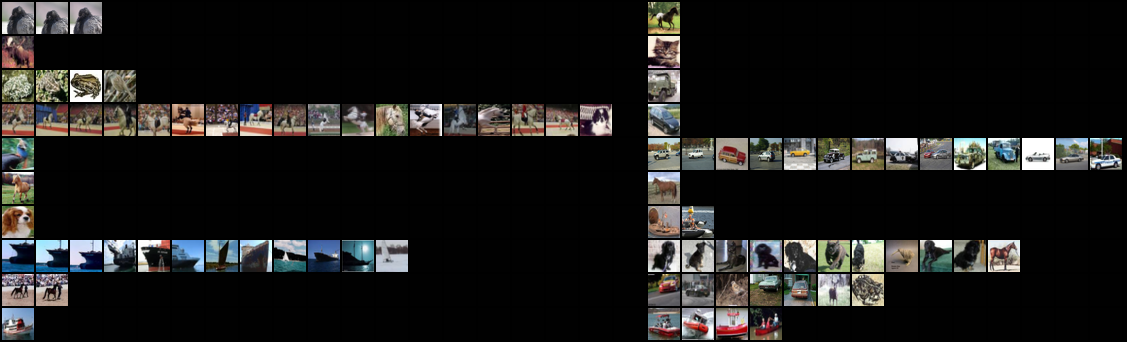}
\end{subfigure}
    \caption{Correct-prediction queries and their results, where each row shows two queries and each query starts with the query input followed by images selected.}
    \label{fig:all-queries-correct}
\end{figure*}
\begin{figure*}[t]
    \centering
\begin{subfigure}{\textwidth}
    \includegraphics[width=\textwidth]{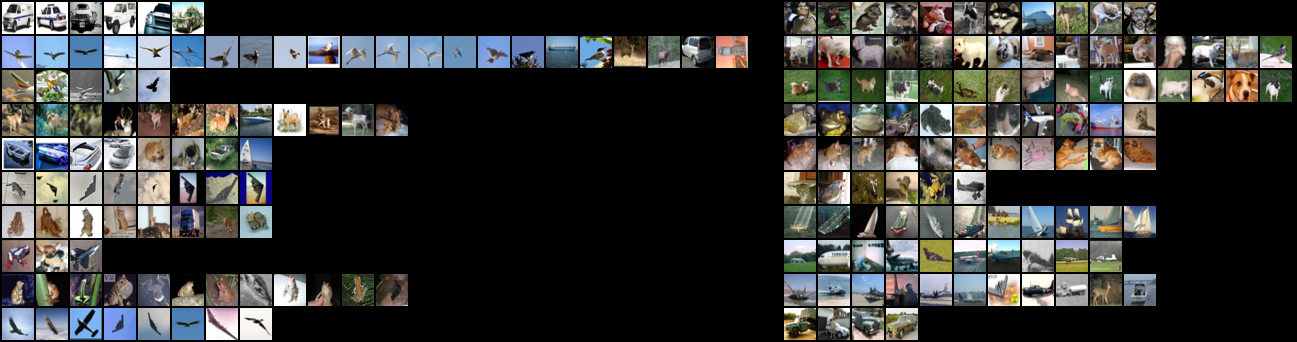}
\end{subfigure}
\hfill
\begin{subfigure}{\textwidth}
    \includegraphics[width=\textwidth]{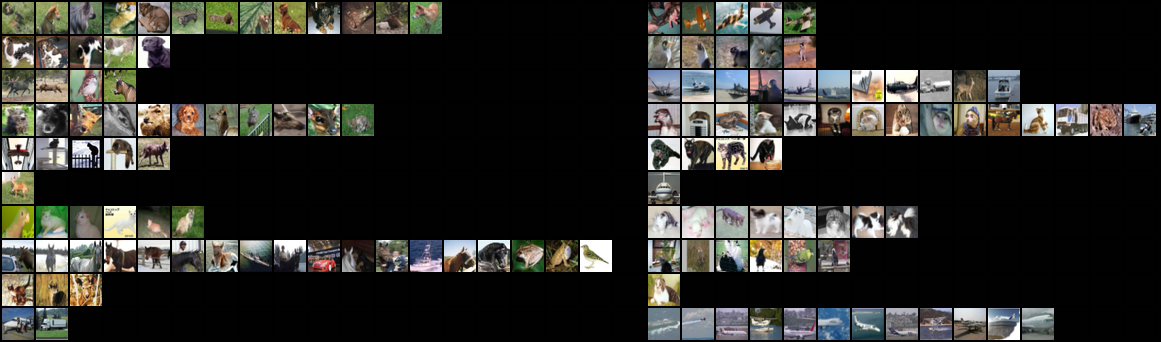}
\end{subfigure}
    \caption{Wrong-prediction queries and their results, where each row shows two queries and each query starts with the query input followed by images selected.}
    \label{fig:all-queries-wrong}
\end{figure*}

\ignore{
\section{Proof for $O(n\log n)$ finite-sample error bound of diff-in-means}
\label{appendix:diff-in-means}

\begin{proof}
Hoeffding's inequality: with probability $(1-\delta)$, for one index $i$, $|\widehat{AME}_i - AME_i| \leq \sqrt{\frac{1}{2m} \log(\frac{2}{\delta})}$.

We first need a union bound for it to hold over the $n$ dimensions, so we use the ``change of variable'' $\delta \leftarrow \frac{\delta}{n}$ to account for the union bound, and get a $\sqrt{\log(n)}$ term. That gives the following bound: with probability at least $1-\delta$, $\forall i, | \widehat{AME_i} - AME_i |  \leq \sqrt{\frac{1}{2m} \log(\frac{2n}{\delta})}$.

Over all coefficients, $\| \widehat{AME_i} - AME_i\|_2 \leq n^\frac{1}{2} \max_i |\widehat{AME_i} - AME_i| = O(n^\frac{1}{2} \sqrt{\frac{log(n)}{m}})$. By setting $m=O(n\log n)$, we can bound the L2-norm error with high probability.
\end{proof}
}
\section{Extended Related Work}
\label{appendix:relwork}

The closest related works, {\bf model explanation}, also aim to provide principled measures of training data impact on the performance of an ML model. These techniques include Shapley value, influence functions~\cite{koh_understanding_2017,koh_accuracy_2019} and Representer Point~\cite{repr}.
TracIn~\cite{pruthi_estimating_2020} and CosIn~\cite{hammoudehsimple} are two other recent proposals for measuring the influence of a training sample based on how it impacts the loss function during training. Existing model explanation approaches fall short. They either focus on marginal influence on the whole dataset~\cite{koh_understanding_2017,koh_accuracy_2019}; make strong assumptions (\eg convex loss functions) that disallow their use with DNNs~\cite{koh_understanding_2017,koh_accuracy_2019,basu_influence_2020}; cannot reason about data sources or sets of training samples~\cite{pruthi_estimating_2020,hammoudehsimple, chakarov2016debugging}; or subsample training data but focus on a single inclusion probability and thus cannot explain results in all scenarios\cite{feldman2020neural,ilyas2022datamodels,zhang2021counterfactual}. 

Of those, the works focused on efficiently estimating the \textbf{Shapley value (SV)} are closest. SV which was proposed in game theory~\cite{Shapley1953}, has been widely studied in recent years, with applications to data valuation in ML~\cite{jia_towards_2020,ghorbani_data_nodate} and feature selection/understanding~\cite{lipovetsky2001analysis,jethani2021fastshap, lundberg2017unified, covert2021improving}, as well as extensions such as D-Shapley~\cite{ghorbani2020distributional} which generalizes SV by modeling the dataset as a random variable.
The SV is notoriously costly to measure, and efficient estimators are the focus of much recent work since the early permutation algorithm with $L_2$ error bounds in $O(N^2 \log(N))$ samples~\cite{maleki2014bounding} under bounded utility.
Recently proposed estimators also reduce SV estimation to regression problems, although different than ours~\cite{lundberg2017unified,covert2021improving,jethani2021fastshap,kwon2021beta}. Beta-Shapley~\cite{kwon2021beta} is closest, as it generalizes data Shapley to consider different weightings based on subset size, which coincides with our AME under $p$-induced distribution (\eg Beta-Shapley is AME when $p\sim$ beta distribution). We also study this approach, and alternative distributions (including a truncated uniform which we found a bit better in practice).
None of these works study the sparse setting, or provide efficient $L_2$ bounds for estimators in this setting. This may stem from their focus on explaining {\em features}, in smaller settings than we consider for training data and in which sparsity may be less natural.
The most comparable work is that of Jia et.al~\cite{jia_towards_2020}, which provides multiple algorithms, including for the sparse, monotonous utility setting. Their approach uses compressive sensing, which is closely related to our LASSO based approach, and yields an $O(N\log\log N)$ rate. We significantly improve on this rate with an $O(k\log N)$ estimator, much more efficient in the sparse ($k \ll N$) regime.
Other estimation strategies, use-cases, and relaxations have been recently proposed for the \SV.
\cite{mitchell2022sampling} study efficient sampling of permutations. However, each ``permutation sample'' requires $N$ utility evaluations, and is thus incompatible with our setting.
\cite{frye2020asymmetric} focus on feature \SV's, and leverage the causal structure over features (on the data distribution) to decide value assignments. This is done by reweighing entire permutations of the features (\eg giving more weight to permutations where a given feature appears early). In contrast, the \AME reweighs the utility of different subsets of the data points, based on their size, which is orthogonal, and seems more amenable to sparse estimators.
\cite{harris2022joint} extend the \SV axioms to consider the joint effect of multiple players, prove the uniqueness of the solution, and derive its formulation. It measures the contribution of all subsets up to a specified subset size, which is however incompatible to our setup due to large $N$ (\eg there will be $O(N^2)\approx 2.5$ billion contribution terms just for subset size of two on our CIFAR-10 datasets). In contrast, the \AME studies the contribution for individual players, as does the regular \SV, and focuses on efficient estimation in the sparse regime. Our hierarchical setting considers fixed sources, and not all possible subsets.

\ignore{
Probability of Sufficiency (PS) also builds on causal inference theory, and measures the causal impact of switching data points' labels \mbox{\cite{chakarov2016debugging}}.
The proposed estimation techniques do not apply to complex ML models such as DNNs.
Probability of Sufficiency (PS)~\cite{chakarov2016debugging} measures the causal impact by switching data points' labels, but their approach does not apply to complex ML models such as DNNs.
In contrast, we propose a principled quantity that measures the marginal impact of a data point when added to different subsets of the data, and show that it captures individual contributions to group effects. 
}

Another related body of work is the work focused on defending against data poisoning attacks.
Reject On Negative Impact (RONI)~\cite{barreno_security_2010,baracaldo_mitigating_2017} describes an algorithm that measures the LOO effect of data points on subsets of the data. However, RONI is sample-inefficient and the paper does not prescribe a subset distribution to be used. As we explained previously, the choice of subset distribution can impact precision and recall.
Another line of work \cite{tran2018spectral,hayase2021spectre,chen2018detecting,tang2021demon,shen2016auror} uses outlier detection to identify poisoned data. These are not query aware and thus can select benign outliers. Other approaches~\cite{doan2020februus,tang2021demon,veldanda2020nnoculation} assume the availability of clean data, or make strong assumptions about the model, e.g., assuming that linear models~\cite{Jagielski2018ManipulatingML}, or the attack, e.g., assuming that the attack is a source-agnostic trigger attacks~\cite{gao2019strip}, a trigger attack with small norm or size~\cite{chou2018sentinet,Wang2019NeuralCI,udeshi2019model}, or a clean-label attack~\cite{peri2020deep}. None of these approaches can generalize across techniques.

Our work is also related to the existing literature on \textbf{data cleaning and management}. However, data cleaning approaches are not query driven, and must rely on other assumptions. As a result, many approaches depend on user provided integrity constraints \cite{chu2013holistic,chu2013discovering} or outlier detection~\cite{maletic2000data, hellerstein2008quantitative}. As a result these approaches cannot always identify poisoned data~\cite{koh2018stronger} and might also identify benign outliers. Recent approaches~\cite{krishnan2017boostclean,dolatshah2018cleaning}, have also used another downstream DNN for data cleaning. However, these approaches assume that corrupt data has an influence on test set performance, an assumption that may not hold in scenarios such as data poisoning. Finally, Rain~\cite{wu2020complaint} is a recent query-driven proposal that proposes using influence function to explain SQL query results. While Rain shares similar goals, we focus on DNNs, a different use case.


Finally, our work leverages, and builds on, a large body of existing work from different fields.
\ignore{
First is the {\bf causal inference} literature.
If we regard the inclusion of a data source as a treatment, our methodology is related to factorial experiments, from which our \quantity{} is inspired. 
For instance, ~\cite{kang2007demystifying, imbens_causal_2015} focuses on single treatment \quantity{} (i.e. only one source) and observational study rather than randomized experiment. 
Multiple treatments are introduced in \cite{egami_causal_2019,dasgupta_causal_2015,hainmueller_causal_2014}, though their quantity uses different population distribution than ours, and different estimation techniques.
In the computer science filed, Sunlight \cite{10.1145/2810103.2813614} uses a similar approach to study ad targeting. However, this paper uses only one sampling probability, and a holdout set for statistical confidence instead of our knockoff procedure.
}
First, our work is related to the {\bf causal inference} literature if we regard the inclusion of a data source as a treatment, from which \quantity{} is inspired.
For instance, \cite{kang2007demystifying, imbens_causal_2015} focuses on single treatment \quantity{} (i.e. only one source). 
Multiple treatments are introduced in \cite{egami_causal_2019,dasgupta_causal_2015,hainmueller_causal_2014}, though their quantity uses different population distribution than ours, and different estimation techniques.
In the computer science filed, Sunlight \cite{10.1145/2810103.2813614} uses a similar approach but with only one sampling probability and without knockoff procedure.
Second, {\bf sparse recovery} studies efficient algorithms to recover a sparse signal from high dimensional observations. We leverage LASSO \cite{lecue_regularization_2017} --with properties related to those of compressive sensing \cite{candes2006compressive}-- and knockoffs \cite{candes_panning_2017}.
Other approaches to the important factor selection problem exist, such as the analysis of variance (ANOVA) \cite{bondell_simultaneous_2009, egami_causal_2019, post_factor_2013} used in \cite{egami_causal_2019}, but we think LASSO is better suited to our use-case due to its scalability guarantees.
Third, our goal is related to {\bf group testing} \cite{noauthor_group_2021,group-survey} as discussed in \S\ref{sec:intro},
and studying if and how group testing ideas could improve our technique is an interesting avenue for future work.